\documentclass{article}

\usepackage[round]{natbib}

\usepackage[main, final]{neurips_2026}

\usepackage[utf8]{inputenc} 
\usepackage[T1]{fontenc}    
\usepackage{hyperref}       
\usepackage{url}            
\usepackage{booktabs}       
\usepackage{amsfonts}       
\usepackage{nicefrac}       
\usepackage{microtype}      
\usepackage{xcolor}         

\usepackage{bm}
\usepackage{amsmath}
\usepackage{amssymb}
\usepackage{mathtools}
\usepackage{amsthm}

\usepackage{multirow, multicol}

\theoremstyle{plain}
\newtheorem{theorem}{Theorem}[section]
\newtheorem{proposition}[theorem]{Proposition}

\theoremstyle{definition}

\theoremstyle{remark}



\usepackage{algorithm}
\usepackage{algpseudocode}
\algrenewcommand\alglinenumber[1]{\scriptsize #1:}

\title{Simulation-free Unbalanced Dynamic Optimal Transport with General Growth Penalty}

\author{%
  Junda Ying$^{1}$, Yuxuan Wang$^2$, Bowen Yang$^3$, Peijie Zhou$^{2,4,5,6,\dagger}$ and Lei Zhang$^{1,2,3,4,7,\dagger}$ \\
  $^{1}$Beijing International Center for Mathematical Research, Peking University.\\$^{2}$Center for Quantitative Biology, Peking University.\\$^{3}$School of Mathematical Sciences, Peking University.\\$^{4}$Center for Machine Learning Research, Peking University.\\$^{5}$National Engineering Laboratory for Big Data Analysis and Applications, Beijing.\\$^{6}$AI for Science Institute, Beijing.$^{7}$Institute for Artificial Intelligence, Peking University.\\
\texttt{yingjunda@stu.pku.edu.cn,pjzhou@pku.edu.cn,zhangl@math.pku.edu.cn} \\
}

\begin{document}

\maketitle

\begin{abstract}
Inferring cellular dynamics from unpaired single-cell snapshots requires modeling both state transitions and population growth or death. Unbalanced dynamic optimal transport (UDOT) addresses this by penalizing growth along transport paths, making the choice of growth penalty a key way to encode biological priors on proliferation and apoptosis.  However, existing UDOT solvers either rely on computationally expensive NeuralODE simulations or depend on analytical solutions of conditional paths, restricting their efficiency solely to quadratic penalties, i.e. Wasserstein-Fisher-Rao (WFR) geodesics. To enable an efficient UDOT solver for general growth penalties, we first show that concave growth penalties lead to degenerate solutions where growth and transport are separated.  We then introduce \textbf{S}imulation-free \textbf{U}nbalanced \textbf{D}ynamic \textbf{O}ptimal transport (SUDO), a simulation-free framework for UDOT with general non-quadratic convex growth penalties. SUDO learns the conditional paths and transport costs, solves the induced semi-coupling problem, and subsequently leverages unbalanced flow matching to achieve a simulation-free solution. On WFR benchmarks, SUDO matches the accuracy of efficient, analytical solution-driven algorithms while outperforming simulation-based methods in computational speed. Beyond WFR, SUDO supports asymmetric penalties that encode proliferation-dominant priors and produce more plausible trajectories and growth estimates on synthetic and single-cell datasets.
\end{abstract}

\section{Introduction}
Inferring the underlying dynamics of complex biological processes is a central challenge in single-cell trajectory inference. However, single-cell RNA sequencing (scRNA-seq) inherently destroys the cell, restricting our observations to static, unpaired temporal snapshots rather than continuous time-series profiles of individual cells. To overcome this fundamental limitation, a standard paradigm is to treat the observed snapshots at different time points as probability distributions and connect them using Dynamic Optimal Transport (DOT) \citep{waddingot,bunne2023learning,trajectorynet,cfm_tong}. By minimizing the transport cost between distributions, DOT provides a mathematically rigorous framework for reconstructing the continuous cellular trajectories from snapshots.

While DOT assumes strict mass conservation during transitions, biological systems are characterized not only by state transitions but also by cell proliferation and apoptosis. Ignoring growth of the cell population can lead to incorrect inferences of the underlying trajectories. To address this issue, recent studies have introduced unbalanced extensions of OT (UOT) \citep{UOT,klein2024genot,moscot,trajectorynet,neural_uot,TIGON,mioflow2}. Among all extensions, we focus on the unbalanced dynamic OT (UDOT) formulation, where a growth penalty denoted $\Psi$ is added to the DOT cost, thus explicitly modeling the growth of the cell population.

However, efficiently solving UDOT remains a challenge. Current UDOT solvers generally fall into two categories. The first category comprises simulation-based methods \citep{TIGON,DeepRUOT,sun2026variational}. While flexible, these methods rely on computationally expensive NeuralODE simulations \citep{NeuralODE}, resulting in scalability issues and making it challenging for large-scale single-cell datasets. The second category includes simulation-free methods \citep{wfr_fm}, which are based on the flow matching framework \citep{cfm_lipman} and exhibit improved scalability. However, since these methods strictly depend on the availability of analytical solutions for the conditional path and transport cost, they are restricted to the Wasserstein-Fisher-Rao (WFR) problem \citep{chizat2018interpolating,chizat2018unbalanced}, where the penalty $\Psi$ must be a specific quadratic function.

The ability to model general penalty functions $\Psi$ is essential for capturing complex biological realities. For example, the standard quadratic penalty symmetrically penalizes proliferation and apoptosis; however, this can be an unrealistic assumption for many biological scenarios. If prior knowledge tells us that a specific developmental process involves minimal apoptosis but massive proliferation, an asymmetric $\Psi$ can be a more biologically sound option. As discussed in \citep{BSB,sun2026variational}, different biological priors or reference dynamics naturally correspond to different choices of $\Psi$. Currently, a central trade-off exists: methods capable of handling general $\Psi$ suffer from scalability issues, while efficient algorithms are restricted to quadratic \(\Psi\). There is a need for a UDOT solver that balances the flexibility of $\Psi$ selection with computational efficiency.

To address these challenges, we propose \textbf{S}imulation-free \textbf{U}nbalanced \textbf{D}ynamic \textbf{O}ptimal transport (SUDO), a simulation-free framework for efficiently solving UDOT problems with general growth penalties. By learning conditional paths and transport costs, SUDO extends the framework proposed by \citep{wfr_fm}, enabling it to solve UDOT problems in a simulation-free manner without relying on the analytical solution of conditional paths. Our main contributions are summarized below:

\begin{itemize}

\item We introduce SUDO, an efficient simulation-free framework for UDOT problems with a broad class of growth penalties, balancing the flexibility of penalty selection with computational efficiency.

\item We theoretically identify the non-degenerate regime of growth penalties for UDOT. In particular, we show that concave growth penalties lead to degenerate UDOT solutions where non-trivial growth and transport are separated, thereby motivating convex growth penalties as the relevant class for biologically meaningful trajectory inference.

\item We demonstrate that SUDO recovers the accuracy of analytical solution-based solvers on WFR problems while significantly outperforming NeuralODE-based baselines in scalability.

\item We demonstrate the necessity of selecting different $\Psi$ as biological priors, demonstrating SUDO's flexibility in integrating diverse prior knowledge through tailored penalty selection.

\end{itemize}

\section{Related works}
\subsection{Unbalanced Optimal Transport}
While traditional static Optimal Transport (OT) \cite{kantorovich1942translocation} seeks the optimal coupling between two probability distributions, \citep{benamou2000computational} showed that using a squared Euclidean transport cost is equivalent to finding a continuous probability flow that minimizes total kinetic energy, known as dynamic OT. Unbalanced OT generalizes it to measures with different total masses, typically by adding penalties to the marginals in the static form \citep{UOT,figalli2010optimal} or to the mass growth in the dynamic form \citep{liero2018optimal,chizat2018unbalanced,Kondratyev}. The latter is called the unbalanced dynamic OT (UDOT). A well-known example is the quadratic penalty case, referred to as Wasserstein-Fisher-Rao (WFR) metric. \citep{chizat2018interpolating,chizat2018unbalanced} provided both the analytical solution between Diracs and the static form for the WFR problem, alongside an efficient solver for the latter. Nevertheless, whether analytical solutions exist for non-quadratic penalties remains unclear.

\subsection{Flow Matching and Extensions}
Flow matching \citep{cfm_lipman,albergo2022building} is a simulation-free framework for learning probability path interpolating observed marginals. Previous works combined it with OT-couplings for learning dynamics \citep{cfm_tong,klein2024genot,rohbeck2025modeling}. Stochastic extensions \citep{tong2023simulationfree,lee2025mmsfm} and unbalanced extensions \citep{UOT,cao2025taming,corso2025composing,wang2025joint,wfr_fm,USB} are also introduced to model complex dynamics. At the core of the framework, one needs to design the conditional path corresponding to the underlying dynamics. When closed form is not available, it is also possible to learn it \citep{kapusniak2024metric, petrovic2026curly}.

\subsection{Neural Network solvers for Unbalanced Dynamic Optimal Transport}
Motivated by single-cell dynamics inference, researchers tried to solve UDOT with neural networks. \citep{TIGON,DeepRUOT,sun2026variational,peng2026stvcr} addressed the problem via NeuralODE \citep{NeuralODE}, which can deal with arbitrary \(\Psi\) but is computationally expensive. \citep{UOT,cao2025taming,corso2025composing,wang2025joint} introduced flow matching into the field. Mathematically, however, the problem they solved is not exactly the UDOT problem. Recently, \citep{wfr_fm} proposed a simulation-free framework for WFR, but it can only solve UDOT with quadratic \(\Psi\). A simulation-free framework for UDOT with general \(\Psi\) is therefore needed.

\section{Preliminaries}
\label{sec:pre}
\textbf{Optimal Transport (OT) and Extensions.} Let $\mathcal{M}_+(\mathcal{X})$ represent the set of all absolutely continuous finite measures supported on $\mathcal{X} \subseteq \mathbb{R}^d$. $\mu_0(\bm{x})$, $\mu_1(\bm{x})$ are two measures in $\mathcal{M}_+(\mathcal{X})$. Assuming \(\mu_0,\mu_1\) are both probability measures, the static OT between \(\mu_0\) and \(\mu_1\) \citep{kantorovich1942translocation} is defined as
\begin{equation}
\label{eq:static OT}
\begin{aligned}
&\text{OT}(\mu_0,\mu_1) =\inf_{\gamma} \int_{\mathcal{X}^2}  \, C(\bm{x},\bm{y})\gamma(\bm{x},\bm{y})\mathrm{d} \bm{x}\mathrm{d} \bm{y}\\
&\text{s.t.}\int_{\mathcal{X}}\gamma(\bm{x},\bm{y})\mathrm{d} \bm{y}=\mu_0(\bm{x}),\ \int_{\mathcal{X}}\gamma(\bm{x},\bm{y})\mathrm{d} \bm{x}=\mu_1(\bm{y})
\end{aligned}
\end{equation}
where \(\gamma\) is called the coupling, and \(C(\bm{x},\bm{y})\) is the transport cost of transporting unit mass from \(\bm{x}\) to \(\bm{y}\). When choosing \(C(\bm{x},\bm{y})=\frac{1}{2}\Vert\bm{x}-\bm{y}\Vert^2\), it also has a dynamic OT (DOT) formulation, also known as the BB-form \citep{benamou2000computational}.
\begin{equation}
\label{eq:dynamic OT}
\begin{aligned}
&\text{OT}(\mu_0,\mu_1) =\inf_{\rho,\bm{u}} \int_0^1\int_{\mathcal{X}}  \, \frac{1}{2}\Vert \bm{u}(\bm{x},t)\Vert_2^2\rho_t(\bm{x})\mathrm{d} \bm{x}\mathrm{d}t \\
&\text{s.t.} \ \ \ \ \ \ \ \partial_t\rho+\nabla_{\bm{x}}\cdot(\rho \bm{u})=0,\ \rho_0=\mu_0,\ \rho_1=\mu_1
\end{aligned}
\end{equation}

The velocity field \(\bm{u}\) models how mass flows through the space, transforming \(\mu_0\) to \(\mu_1\) with least kinetic energy cost. It is also known as the square of the 2-Wasserstein distance (scaled by \(\frac{1}{2}\)). With clear interpretability, DOT has been widely used in single-cell dynamics inference \citep{cfm_tong,moscot,klein2024genot}. Recently, researchers also utilized unbalanced DOT (UDOT) to simultaneously model single-cell trajectories and cell population growth. When the total mass of \(\mu_0,\mu_1\) may be different, the UDOT between \(\mu_0,\mu_1\) is defined as  
\begin{equation}
\label{eq:UDOT}
\begin{aligned}
&\text{UDOT}(\mu_0,\mu_1) =
\inf_{\rho,g,\bm{u}} \int_0^1\int_{\mathcal{X}}  \, \frac{1}{2}\Big(\Vert \bm{u}(\bm{x},t)\Vert_2^2+\Psi(g(\bm{x},t))\Big)\rho_t(\bm{x})\mathrm{d} \bm{x}\mathrm{d}t \\
&\text{s.t.}\ \ \ \ \ \ \ \ \ \ \ \ \ \ \partial_t\rho+\nabla_{\bm{x}}\cdot(\rho \bm{u})=\rho g,\ \rho_0=\mu_0,\ \rho_1=\mu_1
\end{aligned}
\end{equation}
where \(g\) is a growth rate for modeling cell proliferation and apoptosis \citep{TIGON}. \(\Psi\) is the growth penalty for enforcing relative preferred growth rate. When \(\Psi(g)=\delta^2g^2\), (\ref{eq:UDOT}) is known as the square of the Wasserstein-Fisher-Rao (WFR) metric \citep{chizat2018interpolating,liero2018optimal}. One can also consider the stochastic trajectory by replacing the continuity equation constraint with Fokker-Planck equation. The resulting optimization problem is called the regularized unbalanced OT (RUOT) \citep{DeepRUOT, BSB}.
\begin{equation}
\label{eq:dynamic RUOT}
\begin{aligned}
&\text{RUOT}(\mu_0,\mu_1) =
\inf_{\rho,g,\bm{u}} \int_0^1\int_{\mathcal{X}}  \, \frac{1}{2}\Big(\Vert \bm{u}(\bm{x},t)\Vert_2^2+\Psi(g(\bm{x},t))\Big)\rho_t(\bm{x})\mathrm{d} \bm{x}\mathrm{d}t \\
&\text{s.t.}\ \ \ \ \ \ \ \ \ \ \ \ \ \ \partial_t\rho+\nabla_{\bm{x}}\cdot(\rho \bm{u})=\rho g+\frac{\sigma^2}{2}\Delta_{\bm{x}}\rho,\ \rho_0=\mu_0,\ \rho_1=\mu_1
\end{aligned}
\end{equation}
\textbf{Continuous Measure Flow.} Consider a time dependent velocity field $\bm{u} :  \mathbb{R}^d \times[0,1] \to \mathbb{R}^d$, and a time dependent growth rate $g : \mathbb{R}^d \times[0,1] \to \mathbb{R}$. A population of weighted particles \((\bm{x}_i,m_i)\) evolve under the ODE: \(\mathrm{d}\bm{x}_t = \bm{u}(\bm{x}_t,t)\mathrm{d}t,\  \mathrm{d}\operatorname{ln}m_t=g(\bm{x}_t,t)\mathrm{d}t\). The time-dependent measure \(\rho_t\) induced by the particles \(\{\bm{x}_i\}\), weighted by their mass \(\{m_i\}\), follows the continuity equation with source term \(\partial_t\rho+\nabla_{\bm{x}}\cdot(\rho \bm{u})=\rho g\). One can also replace the ODE with SDE to introduce stochasticity. A diffusion term will be added to the continuity equation with source term.



\textbf{Unbalanced Flow Matching (UFM).} \citep{wfr_fm} developed a simulation-free UFM framework for learning the continuous measure flow satisfying \(\partial_t\rho+\nabla_{\bm{x}}\cdot(\rho \bm{u})=\rho g\), which interpolates \(\mu_0\) and \(\mu_1\). They parameterized two neural networks \(\bm{u}_{\bm{\theta}}\) and \(g_{\bm{\theta}}\) to approximate the true \(\bm{u}\) and \(g\). The neural networks are trained to minimize the regression loss (UFM loss).
\begin{equation}
\begin{aligned}
&\mathcal{L}_{\text{UFM}}(\bm{\theta})=\int_{0}^{1}\int_\mathcal{X} (\left\| \bm{\bm{u}_{\theta}}(\bm{x},t) - \bm{u}(\bm{x},t) \right\|_2^2+\left\| g_{\bm{\theta}}(\bm{x},t) - g(\bm{x},t) \right\|_2^2)\rho_t(\bm{x})\mathrm{d} \bm{x}\mathrm{d}t
\label{eq:UFM}
\end{aligned}
\end{equation}
Although the true \(\rho, \bm{u},g\) are intractable, they proved that minimizing the loss above is equivalent to minimize the conditional version below (CUFM loss).
\begin{equation}
\begin{aligned}
\label{eq:CUFM}
&\mathcal{L}_{\text{CUFM}}(\bm{\theta})=
\mathbb{E}_{t, \bm{z}, \bm{x}\sim \tilde{\rho}_t(\bm{x}\vert \bm{z})}m_t(\bm{z})\big(\left\| \bm{\bm{u}_{\theta}}(\bm{x},t) - \bm{u}_t(\bm{x}\vert\bm{z}) \right\|_2^2
+\left\| g_{\bm{\theta}}(\bm{x},t)- g_t(\bm{x}\vert\bm{z}) \right\|_2^2\big)
\end{aligned}
\end{equation}
where \(t\sim \mathcal{U}[0,1]\), \(\bm{z}\) is a conditional variable with density \(q(\bm{z})\), \(\rho_t(\bm{x}\vert\bm{z})\) is called the conditional measure path such that the marginal measure path \(\rho_t(\bm{x})\) satisfies that \(\rho_t(\bm{x})=\int\rho_t(\bm{x}\vert\bm{z})q(\bm{z})\mathrm{d}\bm{z}\). The conditional measure is further decoupled into a mass term and a probability density \(\rho_t(\bm{x}\vert\bm{z}) = m_t(\bm{z})\tilde{\rho}_t(\bm{x}\vert\bm{z})\). Following \citep{chizat2018unbalanced,wfr_fm}, one can choose \(\bm{z}=(\bm{x}_0,\bm{x}_1)\) drawn from the UDOT semi-coupling \(\gamma_0(\bm{x}_0,\bm{x}_1)\) (introduced in section \ref{sec:theory}), and derive the conditional velocity \(\bm{u}_t(\bm{x}\vert\bm{z})\), growth rate \(g_t(\bm{x}\vert\bm{z})\) and mass \(m_t(\bm{z})\) from the travelling Dirac (also introduced in section \ref{sec:theory}). The resulting flow recovers the UDOT flow. 


\textbf{Growth Penalty as Biological Prior.}
While previous studies have explored the WFR problem i.e. the quadratic penalty $\Psi(g)=\delta^2g^2$, which symmetrically penalizes the absolute magnitude of the growth rate to favor moderate rates of both cell proliferation and apoptosis, more general \(\Psi\) can be introduced to incorporate other type of biological priors.

A family of convex \(\Psi\) was discussed by \citep{BSB}. Three specific \(\Psi\) were introduced as proliferation/apoptosis preference with nice biological interpretation where cellular dynamics are modeled via branching stochastic process. By utilizing these $\Psi$, we can explicitly incorporate preferences for proliferation and apoptosis into the growth modeling. For details, see \ref{appendix:BSB}.
\begin{equation}
\label{eq:BSB Psi}
\left \{
\begin{aligned}
&\Psi(g)=2(1-g+g\operatorname{log}g)\ (g>0)\ &\text{Only-growth}\\
&\Psi(g)=2(1+g-g\operatorname{log}(-g))\ (g<0)\ &\text{Only-death}\\
&\Psi(g)=2\big(1-\sqrt{1+g^2}+g\operatorname{log}(g+\sqrt{1+g^2})\big)\ &\text{No preference}\\
\end{aligned}
\right .
\end{equation}
In parallel, a family of concave \(\Psi(g)=|g|^p,0<p<1\) was also introduced as regularization terms by \citep{sun2026variational}. The convexity or concavity of $\Psi$ determines a certain monotonicity of the growth rate during the process. A more detailed discussion on their results can be found in \ref{appendix:varrout}.

In this paper, we consider a broader class of $\Psi$ satisfying the following assumptions: 1) $\Psi$ first decreases and then increases, reaching a minimum value of 0. This implies the existence of a unique, most-preferred growth rate $g_0$, and the penalty increases the further the rate deviates from this value; 2) $\Psi$ is either convex or concave on both sides of $g_0$, meaning that the marginal penalty for deviating from $g_0$ is either increasing or decreasing. Our first results can be stated as following.

\begin{theorem}
\label{thm:concave}
For \(\Psi\) which is concave on both sides of \(g_0\), the solution of UDOT problem (\ref{eq:UDOT}) separates the non-trivial growth and transport.
\end{theorem}
The proof is left to \ref{pf:concave}. As discussed in \citep{wang2025joint,wfr_fm}, in a biologically realistic cellular trajectory, a non-trivial population growth and cell state transitions tend to occur concurrently. Thus, in the following part of this paper, we focus on convex \(\Psi\).

\section{Travelling Dirac and Static Form}
\label{sec:theory}
As an important special case, we first study the UDOT problem between two Diracs. The corresponding optimal path $(\bm{x}(t), m(t))$ is usually called the travelling Dirac \citep{chizat2018interpolating}.
\begin{equation}
\label{eq:Dirac UDOT}
\begin{aligned}
&\text{UDOT-DD}(m_0\delta_{\bm{x}_0},m_1\delta_{\bm{x}_1}) = \inf_{m,\bm{x}} \int_0^1\frac{1}{2}\Big(\Vert \dot{\bm{x}}(t)\Vert_2^2+ \Psi(\frac{\dot{m}(t)}{m(t)})\Big)m(t)\mathrm{d}t \\
&\text{s.t.} \ \ \ \ \ \ \ \ \ \ \ \ \ \ \ m(0)=m_0,m(1)=m_1,\bm{x}(0)=\bm{x}_0,\bm{x}(1)=\bm{x}_1&
\end{aligned}
\end{equation}

\begin{theorem}
\label{thm:EL eqn}
There exists two 1-dimensional functions $k(t),l(t)$ which only depend on $d=\Vert \bm{x}_1-\bm{x}_0\Vert,r=\frac{m_1}{m_0}$ s.t. $\bm{x}(t)=\bm{x}_0+k(t)\frac{\bm{x}_1-\bm{x}_0}{d}$, $m(t)=m_0l(t)$, i.e. the travelling Dirac is straight.
\end{theorem}
The proof is left to \ref{pf:EL eqn}. The theorem reduces the high-dimensional optimization problem (\ref{eq:Dirac UDOT}) to a 1-dimensional problem. 
\begin{equation}
\label{eq:Path}
\begin{aligned}
&\text{UDOT-DD}(m_0\delta_{\bm{x}_0},m_1\delta_{\bm{x}_1})=C_d(m_0,m_1) = m_0\inf_{l,k} \int_0^1\frac{1}{2}\Big(\dot{k}(t)^2+ \Psi(\frac{\dot{l}(t)}{l(t)})\Big)l(t)\mathrm{d}t \\
&\text{s.t.} \ \ \ \ \ \ \ \ \ \ \ \ \ \ \ \ \ \ \  \ \ \ \ \ \ \ \ \ \ \ \ \ l(0)=1,l(1)=r,k(0)=0,k(1)=d&
\end{aligned}
\end{equation}

Some properties are straightforward.
\begin{proposition}
\label{prop:cost form}
The cost $C_d$ is a homogeneous function of degree 1 w.r.t $(m_0, m_1)$. Thus, it takes form of $C_d(m_0,m_1)=m_0f_d(\frac{m_1}{m_0})$ where $f_d(r)=C_d(1,r)$. For even $\Psi$, $C_d$ is symmetric. 
\end{proposition}

The proof is left to \ref{pf:cost form}. As an important family of \(\Psi\) which is interested in this paper, for uniformly convex functions, we have the following theorem.

\begin{theorem}
\label{thm:uniformly convex}
For uniformly convex $\Psi$, i.e. \(\exists \kappa>0 \ s.t. 
\Psi''\ge\kappa\), $C_d$ is jointly convex and sublinear w.r.t $(m_0, m_1)$ when \(d\le\pi\sqrt{\frac{\kappa}{2}}\), i.e. $C_d(m_0,m_1)+C_d(n_0,n_1)\ge C_d(m_0+n_0,m_1+n_1)$.
\end{theorem}

The proof is left to \ref{pf:uniformly convex}. As an example, when \(\Psi=\delta^2g^2\), i.e. WFR, the upper bound of \(d\) is \(\pi\delta\), which is consistent to \citep{chizat2018interpolating}. Having a sublinear $C_d$, one can reformulate (\ref{eq:UDOT}) into a static form, following the standard technique introduced by \citep{chizat2018unbalanced}. This can be viewed as a Kantorovich form of UDOT.
\begin{equation}
\label{eq:static UDOT}
\text{UDOT}(\mu_0,\mu_1) = \inf_{(\gamma_0,\gamma_1)\in \Gamma(\mu_0, \mu_1)} \int_{\mathcal{X}^2}  \, C_{\Vert \bm{x}-\bm{y}\Vert}(\gamma_0(\bm{x},\bm{y}),\gamma_1(\bm{x},\bm{y}))\mathrm{d} \bm{x}\mathrm{d} \bm{y},
\end{equation}
where $(\gamma_0,\gamma_1)$, satisfying the following constraints, is called semi-coupling.  
\begin{equation}
\label{eq:semi-coupling constraint}
    \Gamma(\mu_0, \mu_1) \stackrel{\mathrm{def.}}{=} \left\{(\gamma_0, \gamma_1) \in \left(\mathcal{M}_+(\mathcal{X}^2)\right)^2 : \int_\mathcal{X}\gamma_0(\bm{x},\bm{y})\mathrm{d} \bm{y} = \mu_0(\bm{x}), \int_\mathcal{X}\gamma_1(\bm{x},\bm{y})\mathrm{d} \bm{x} = \mu_1(\bm{y})\right\}.
\end{equation}
Intuitively, the semi-coupling means that mass $\gamma_0(\bm{x},\bm{y})$ was sent from $\bm{x}$, and $\gamma_1(\bm{x},\bm{y})$ was received at $\bm{y}$ after transport and growth. 

\begin{theorem}
\label{thm:convexity}
For uniformly convex $\Psi$ with $\Psi''\ge \kappa>0$, and \(\text{diam}(\mathcal{X})\le\pi\sqrt{\frac{\kappa}{2}}\), (\ref{eq:static UDOT}) is a convex optimization w.r.t $(\gamma_0,\gamma_1)$.  
\end{theorem}
The proof is left to \ref{pf:convexity}. As mentioned in section \ref{sec:pre}, the UDOT flow (\ref{eq:UDOT}) can be obtained by averaging the travelling Dirac paths w.r.t the semi-coupling. Intuitively, The semi-coupling plans \textit{how much} mass goes from $x$ to $y$, while the travelling Dirac determines \textit{how} it moves and grows through space. 

We point out that the two main assumptions used in this section: 1) \(\Psi\) is uniformly convex; 2) \(\text{diam}(\mathcal{X})\le\pi\sqrt{\frac{\kappa}{2}}\), can easily hold in practice. Since a real-world dataset is always supported on a compact domain \(\mathcal{X}\), one can scale the growth penalty \(\Psi\) with sufficiently large \(\delta^2\) to satisfy \(\text{diam}(\mathcal{X})\le\delta\pi\sqrt{\frac{\kappa}{2}}\). And for general convex penalties such as (\ref{eq:BSB Psi}), a uniformly convex wall, for example \(\frac{\kappa}{2}g^2\), can be added for sufficiently large \(|g|\) to satisfy the uniform convexity. 
\section{Simulation-free Training of General UDOT problem}
The difficulty is that only when \(\Psi\) is a quadratic function, the travelling Dirac has a closed form, and the semi-coupling can be easily obtained by solving an entropy regularized OT \citep{chizat2018unbalanced}. To efficiently solve the UDOT under UFM framework, we divide the training pipeline into four stages: learning the travelling Dirac, learning the transport cost, solving the semi-coupling, and unbalanced flow matching. The inference workflow are left to \ref{appendix:inference}.

\subsection{Learning the Travelling Dirac Path}
Inspired by previous work \citep{kapusniak2024metric,petrovic2026curly}, we use a neural network to learn the travelling Dirac path which doesn't have a closed form. Thanks to theorem \ref{thm:EL eqn}, we only need to learn $k(t)$ and $l(t)$. Hence we never encounter the curse of dimensionality. We parameterize two neural network $\phi_{\eta}(t,d,r), \psi_{\eta}(t,d,r)$ to learn the path.
\begin{equation}
\label{eq:path model}
\begin{aligned}
&\bm{x}_{\eta}(t,d,r) = \bm{x}_0+(\bm{x}_1-\bm{x}_0)(t+t(1-t)\phi_{\eta}(t,d,r))\\
&m_{\eta}(t,d,r)=m_0r^t\operatorname{exp}(t(1-t)\psi_{\eta}(t,d,r))
\end{aligned}
\end{equation}
With the form above, the boundary value of the travelling Dirac is automatically satisfied. More clearly, $k_{\eta}(t,d,r)=d(t+t(1-t)\phi_{\eta}(t,d,r)),l_{\eta}(t,d,r)=r^t\operatorname{exp}(t(1-t)\psi_{\eta}(t,d,r))$. The goal here is to train the neural networks to minimize the path energy (\ref{eq:Path}). Viewing it as a variational problem, we solve it by minimizing a Monte Carlo style loss function \citep{e2018deepritz}. The derivatives are w.r.t $t$. The expectation w.r.t \((d,r)\) is taken over a 2-dim grid \(D\times R\), which is introduced below.
\begin{equation}
\label{eq:path loss}
\begin{aligned}
\mathcal{L}_{path}(\eta)=\mathbb{E}_{d,r}\mathbb{E}_{t\sim \mathcal{U}[0,1]}\Big(\dot{k}_{\eta}(t,d,r)^2+ \Psi(\frac{\dot{l}_{\eta}(t,d,r)}{l_{\eta}(t,d,r)})\Big)l_{\eta}(t,d,r)
\end{aligned}
\end{equation}
\subsection{Learning the Transport Cost Between Diracs}
Thanks to theorem \ref{thm:uniformly convex}, it is sufficient to approximate $C_d(1,r)$, which is a function of only $d,r$, rather than approximating $C_{\Vert\bm{x}_1-\bm{x}_0\Vert}(m_0,m_1)$. Further, since the total mass is bounded, \(C_d\) is bounded. Considering that $C_d(m_0,m_1)=m_0C_d(1,\frac{m_1}{m_0})=m_1C_d(\frac{m_0}{m_1},1)$, it is sufficient to only approximate \(C_d(1,r)\) for $r$ in some moderate range \([r_{min},r_{max}]\), since extremely small or large $r$ corresponds to extremely small $m_0$ or $m_1$, hence extremely small \(C_d(m_0,m_1)\). Also, in the following stages, we only need the value of \(C_d(m_0,m_1)\) for \(d=\Vert\bm{x}_1-\bm{x}_0\Vert\), where \(\bm{x}_1,\bm{x}_0\) are sampled from the dataset. Hence, it is also sufficient to approximate only for $d$ in some moderate range \([d_{min},d_{max}]\). In practice, we uniformly discretize the range of $d$ into a grid $D$, and uniformly discretize the range of $r$ on a log scale into a grid $R$. Given the trained path model $k_{\eta}$ and $l_{\eta}$, we do this by parameterizing a neural network $\mathcal{E}_{\xi}(d,r)$ and minimizing the regression loss on the 2-dim grid $D\times R$.
\begin{equation}
\label{eq:cost loss}
\begin{aligned}
\mathcal{L}_{cost}(\xi)=\sum_{d\in D}\sum_{r\in R}\Big(\mathcal{E}_{\xi}(d,r)-\frac{1}{2}\int_0^1\big(\dot{k}_{\eta}(t,d,r)^2+ \Psi(\frac{\dot{l}_{\eta}(t,d,r)}{l_{\eta}(t,d,r)})\big)l_{\eta}(t,d,r)\mathrm{d}t\Big)^2
\end{aligned}
\end{equation}
The integral is calculated numerically via the trapezoid scheme.  Note that the integral is 1-dimensional, and the derivative is calculated w.r.t $t$ which is also a 1-dimensional variable. Hence, the numerical integral also encounters no curse of dimensionality. Also, we don't need to calculate the gradient of the trajectory ($k_{\eta}$ and $l_{\eta}$) in this stage, which differs from simulation-based methods such as NeuralODE \citep{NeuralODE}. Thus, our method is simulation-free.

\subsection{Solving the Semi-coupling via Projected Gradient Descent}
Thanks to theorem \ref{thm:convexity}, (\ref{eq:static UDOT}) is a convex optimization w.r.t the semi-coupling $(\gamma_0,\gamma_1)$ with a convex constraint (\ref{eq:semi-coupling constraint}), hence, can be efficiently solved by projected gradient descent. Given the trained cost model $\mathcal{E}_{\xi}$, we approximate (\ref{eq:static UDOT}) with a neural surrogate by replacing the true cost $C$ with $\mathcal{E}_{\xi}$. 
\begin{equation}
\label{eq:surrogate UDOT}
\text{UDOT-SUR}(\mu_0,\mu_1) = \inf_{(\gamma_0,\gamma_1)\in \Gamma(\mu_0, \mu_1)} \int_{\mathcal{X}^2}  \, \gamma_0(\bm{x},\bm{y})\mathcal{E}_{\xi}({\Vert \bm{x}-\bm{y}\Vert},\frac{\gamma_1(\bm{x},\bm{y})}{\gamma_0(\bm{x},\bm{y})})\mathrm{d} \bm{x}\mathrm{d} \bm{y},
\end{equation}
The optimal semi-coupling $(\gamma_0^\star,\gamma_1^\star)$ is optimized via projected gradient descent. The projection step is adapted from \citep{duchi2008efficient}.

\subsection{Unbalanced Flow Matching}
Given the optimized semi-coupling $(\gamma_0^\star,\gamma_1^\star)$ and the trained path model $k_{\eta},l_{\eta}$, or say $\bm{x}_{\eta},m_{\eta}$, we then follow \citep{wfr_fm} to train an unbalanced flow matching model for learning the velocity field $\bm{u}(\bm{x},t)$ and growth rate $g(\bm{x},t)$ corresponding to the UDOT problem (\ref{eq:UDOT}). We parameterize two neural networks $\bm{u}_{\theta}(\bm{x},t), g_{\theta}(\bm{x},t)$ and minimize the conditional unbalanced flow matching loss (\ref{eq:CUFM}). 
\begin{equation}
\begin{aligned}
\label{eq:SUDO CUFM}
&\mathcal{L}_{\text{CUFM}}(\bm{\theta})=
\mathbb{E}_{t \sim \mathcal{U}[0,1], (\bm{x}_0,\bm{x}_1)\sim\gamma_0^\star, }m_{\eta,t}\big(\left\| \bm{\bm{u}_{\theta}}(\bm{x}_{\eta,t},t) - \dot{\bm{x}_{\eta,t}} \right\|_2^2
+\left\| g_{\bm{\theta}}(\bm{x}_{\eta,t},t)- \frac{\dot{m_{\eta,t}}}{m_{\eta,t}} \right\|_2^2\big)
\end{aligned}
\end{equation}
where $d=\Vert\bm{x}_1-\bm{x}_0\Vert,r=\frac{\gamma_1^*(\bm{x}_1,\bm{x}_0)}{\gamma_0^*(\bm{x}_1,\bm{x}_0)}$, and $m_{\eta,t}=m_{\eta}(t,d,r), \bm{x}_{\eta,t}=\bm{x}_{\eta}(t,d,r)$ are given by the learned travelling Dirac (\ref{eq:path model}). The derivatives are w.r.t $t$. The whole training process requires no trajectory simulation, hence is simulation-free. A stochastic version for empirically solving the RUOT (\ref{eq:dynamic RUOT}) problem can be found in \ref{appendix:USM}. Pseudocode are left to \ref{appendix:algorithm}.

\section{Experiments}
\label{sec:experiment}
\textbf{SUDO Recovers the WFR.} To validate SUDO's capability in solving UDOT problems, we applied it to the WFR problem (\(\Psi(g)=\delta^2g^2\)) — currently the only UDOT formulation for which an analytical solution is known. Since WFR-FM explicitly utilizes the WFR analytical solution, we consider it a strong approximation of the ground truth. We demonstrate that SUDO can match the marginal measures with high accuracy while obtaining a sufficiently low UDOT cost. We first test whether SUDO can match the target measures. We trained SUDO on three synthetic datasets: Simulation (2D) \citep{DeepRUOT}, Dyngen (5D) \citep{cannoodt2021spearheading}, and Gaussian (1000D) \citep{wfr_fm}. We calculated the \(\mathcal{W}_1\) distance and relative mass error (RME) between the true measure and the predicted measure generated by SUDO. The results are compared to other UDOT based methods in Table \ref{tab:recover WFR}. Results on real data can be found in \ref{appendix:additional results}.
\begin{table}[h!]
\centering
\caption{Mean $\mathcal{W}_1$ and RME on synthetic datasets. For the methods that exhibit randomness in inference, we report the mean value and standard deviation over 5 runs.}
\label{tab:recover WFR}
\begin{tabular}{lcccccc}
\toprule
\textbf{Method} & \multicolumn{2}{c}{Simulation (2D)} & \multicolumn{2}{c}{Dyngen (5D)} & \multicolumn{2}{c}{Gaussian (1000D)} \\
\cmidrule(lr){2-3} \cmidrule(lr){4-5} \cmidrule(lr){6-7}
& $\mathcal{W}_1$ ($\downarrow$) & RME ($\downarrow$) & $\mathcal{W}_1$ ($\downarrow$) & RME ($\downarrow$) & $\mathcal{W}_1$ ($\downarrow$) & RME ($\downarrow$) \\
\midrule
TIGON  & 0.045 & 0.014 & 0.512 & \underline{0.047} & \underline{2.263}   & 0.127  \\
DeepRUOT & 0.043\tiny$\pm$0.002 & 0.017\tiny$\pm$0.001 & 0.623\tiny$\pm$0.032 & 0.065\tiny$\pm$0.011 & 3.785\tiny$\pm$0.009 & 0.303\tiny$\pm$0.070   \\
Var-RUOT  & 0.079\tiny$\pm$0.003 & 0.008\tiny$\pm$0.002 & 0.522\tiny$\pm$0.008 & 0.177\tiny$\pm$0.007 & 2.813\tiny$\pm$0.004  & \textbf{0.041}\tiny$\pm$0.006  \\
WFR-FM & \textbf{0.019} & \textbf{0.001} & \textbf{0.135} & \textbf{0.005} & \textbf{2.233}  & \underline{0.044} \\
\textbf{SUDO} & \underline{0.022} & \underline{0.002} & \underline{0.190} & \textbf{0.005} & 2.315  & 0.061 \\
\bottomrule
\end{tabular}
\end{table}
In most scenarios, SUDO outperforms simulation-based UOT methods (TIGON, DeepRUOT, VarRUOT) and achieves comparable performance to WFR-FM across various datasets. We note that since SUDO does not rely on the analytical solution of WFR, its slight loss in accuracy relative to WFR-FM is predictable and acceptable. Given the marginal measures matched, to further evaluate whether SUDO can recover the dynamic WFR flow, we compared the UDOT cost (\ref{eq:UDOT}) of the measure path \(\rho_t\) given by SUDO and WFR-FM in Table \ref{tab:action}. The cost induced by the static semi-coupling is also shown as a ground truth.
\begin{table}[h!]
\centering
\caption{The UDOT cost of SUDO and WFR-FM. Both algorithms are deterministic.}
\label{tab:action}
\resizebox{\textwidth}{!}{
\begin{tabular}{lcccc}
\toprule
\textbf{Algorithm} & \textbf{Simulation (2D)} & \textbf{Dyngen (5D)} & \textbf{EMT (10D)} & \textbf{Gaussian (1000D)} \\
\midrule
WFR-FM           & 1.0810  & 7.4620  & 0.7726 & 7.6843 \\
SUDO             & 1.0703  & 7.7064  & 0.7741 & 7.1576 \\
Static Reference & 1.0935 & 7.8403 & 0.9545 & 9.9153    \\
\bottomrule
\end{tabular}
}
\end{table}
SUDO delivers performance on par with WFR-FM across three synthetic dataset and EMT \citep{cook2020context} scRNA-seq dataset. These results validate the effectiveness of SUDO in solving UDOT problems. The costs of both algorithms are lower than the static reference. This is because the marginal measures are not precisely matched, allowing them to find a lower-cost solution.



\textbf{Scalability.} To demonstrate the scalability afforded by simulation-free training, we benchmarked SUDO's training time and memory usage on the Mouse dataset \cite{weinreb2020lineage} at varying scales (10000 to 49302 cells). We compared our method with DeepRUOT and VarRUOT, with WFR-FM serving as the optimal reference baseline. Results are shown in Table \ref{tab:scalability}. SUDO outperforms simulation-based solvers in computational cost. For details of the constant training time of VarRUOT, see \ref{appendix:varruot training time}.
\begin{table}[htbp]
\centering
\caption{Scalability evaluation on the Mouse dataset (Training time and peak memory consumption)}
\label{tab:scalability}
\resizebox{\textwidth}{!}{
\begin{tabular}{llccccc}
\toprule
\multirow{2}{*}{Algorithm} & \multirow{2}{*}{Metric} & \multicolumn{5}{c}{Cell number} \\
\cmidrule(r){3-7}
 & & 10,000 & 20,000 & 30,000 & 40,000 & 49,302 \\
\midrule
\multirow{2}{*}{DeepRUOT} & Time (s) & $493.14 \pm 58.24$ & $818.73 \pm 11.89$ & $1290.99 \pm 63.32$ & $1850.74 \pm 196.53$ & $2204.64 \pm 171.37$ \\
                          & Mem (GB) & $2.76 \pm 0.01$    & $6.16 \pm 0.00$    & $10.22 \pm 0.00$    & $14.98 \pm 0.00$     & $19.73 \pm 0.00$ \\
\midrule
\multirow{2}{*}{VarRUOT}  & Time (s) & $2707.19 \pm 319.11$ & $2689.76 \pm 277.53$ & $2701.46 \pm 284.03$ & $2754.84 \pm 251.65$ & $2756.25 \pm 281.77$ \\
                          & Mem (GB) & $2.57 \pm 0.00$ & $3.55 \pm 0.00$ & $4.58 \pm 0.01$ & $5.64 \pm 0.01$ & $6.63 \pm 0.00$ \\
\midrule
\multirow{2}{*}{\textbf{SUDO}}     & Time (s) & $221.42 \pm 2.83$ & $347.80 \pm 5.55$ & $489.48 \pm 10.65$ & $622.74 \pm 20.02$ & $744.72 \pm 12.36$ \\
                          & Mem (GB) & $4.55 \pm 0.00$ & $5.03 \pm 0.01$ & $5.87 \pm 0.01$ & $7.00 \pm 0.01$ & $8.48 \pm 0.01$ \\
\midrule
\multirow{2}{*}{WFR-FM}   & Time (s) & $61.67 \pm 2.23$ & $67.57 \pm 0.44$ & $77.61 \pm 0.34$ & $90.58 \pm 0.32$ & $104.22 \pm 0.35$ \\
                          & Mem (GB) & $0.59 \pm 0.03$ & $1.71 \pm 0.00$ & $3.25 \pm 0.00$ & $5.64 \pm 0.00$ & $8.24 \pm 0.00$ \\
\bottomrule
\end{tabular}
}
\end{table}



\textbf{The Necessity of General Growth Penalty.} While WFR is mathematically elegant and computationally efficient, a non-quadratic $\Psi$ can express more explicit biological priors \citep{BSB}. We demonstrate the necessity of general choice of \(\Psi\) on both synthetic and real scRNA-seq data.

Consider two distant cell clusters on a 2D plane (bottom-left and top-right) with initial populations of 1000 each at $t=0$, shifting to 400 and 2600 at $t=1$. Without prior knowledge, WFR—due to its symmetric penalty—favors apoptosis in the bottom-left and proliferation in the top-right, with no migration occurring if $\delta^2$ is moderate (Figure \ref{fig:1}). However, for highly proliferative and rarely apoptotic cells like cancer cells and stem cells, this solution contradicts biological priors. By employing an asymmetric growth penalty (only-growth)
\begin{equation}
\label{eq:only growth Psi}
\Psi(g)=\left\{
\begin{aligned}
&2(1-g+g\operatorname{log}g) &g>0\\
&+\infty &g\le 0
\end{aligned}
\right .
\end{equation}
to incorporate this prior into the dynamics inference, SUDO infers a more biologically plausible trajectory: rather than undergoing apoptosis, the bottom-left cluster exhibits mild proliferation and dispatches around 700 cells that concurrently proliferate and migrate toward the top-right cluster (Figure \ref{fig:1}). This result matches the prior knowledge. The results on a more complex 1000D Gaussian data can be found in \ref{appendix:gaussian data}.

\begin{figure}
  \centering
  \includegraphics[width=0.23\textwidth]{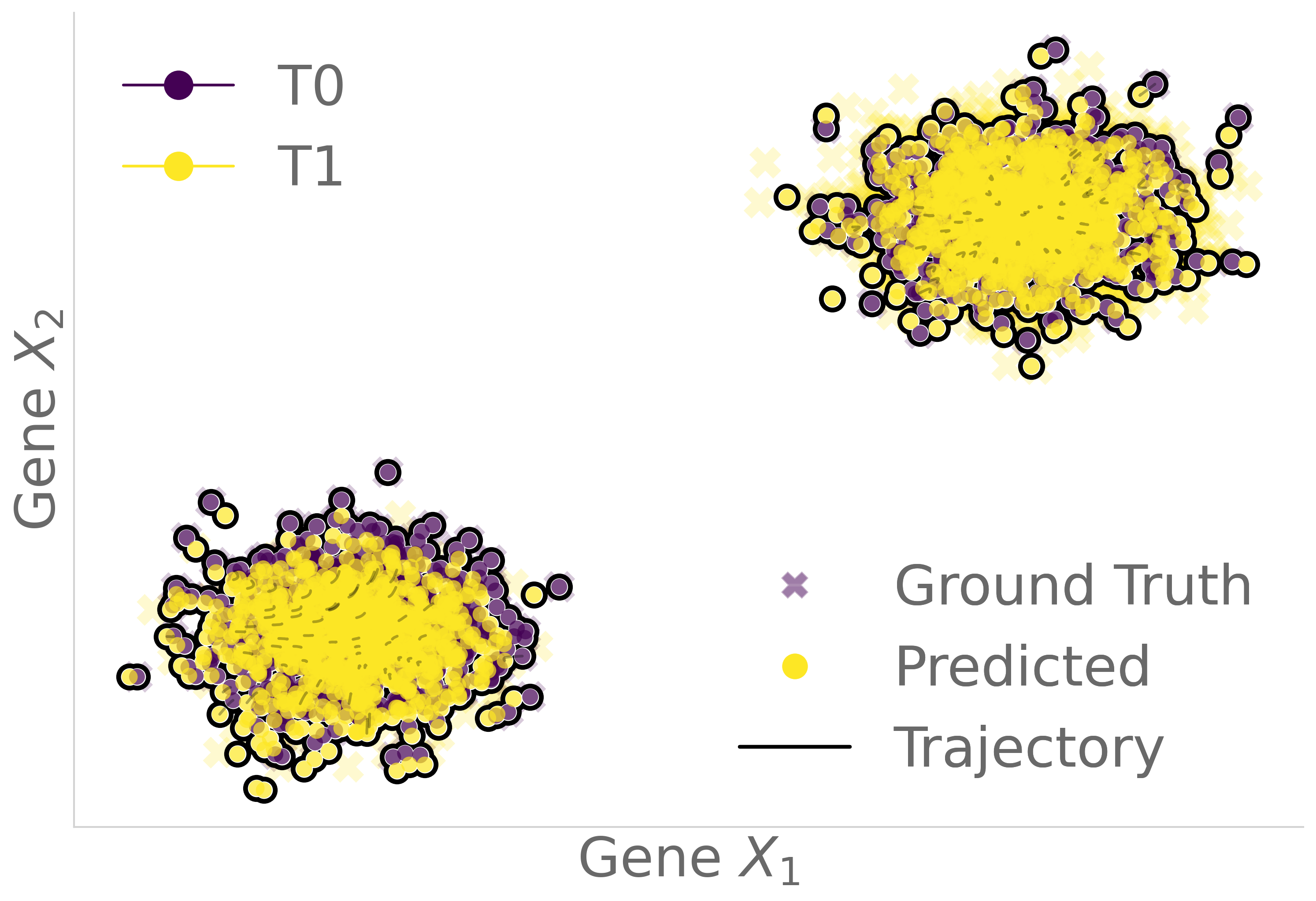} 
  \includegraphics[width=0.23\textwidth]{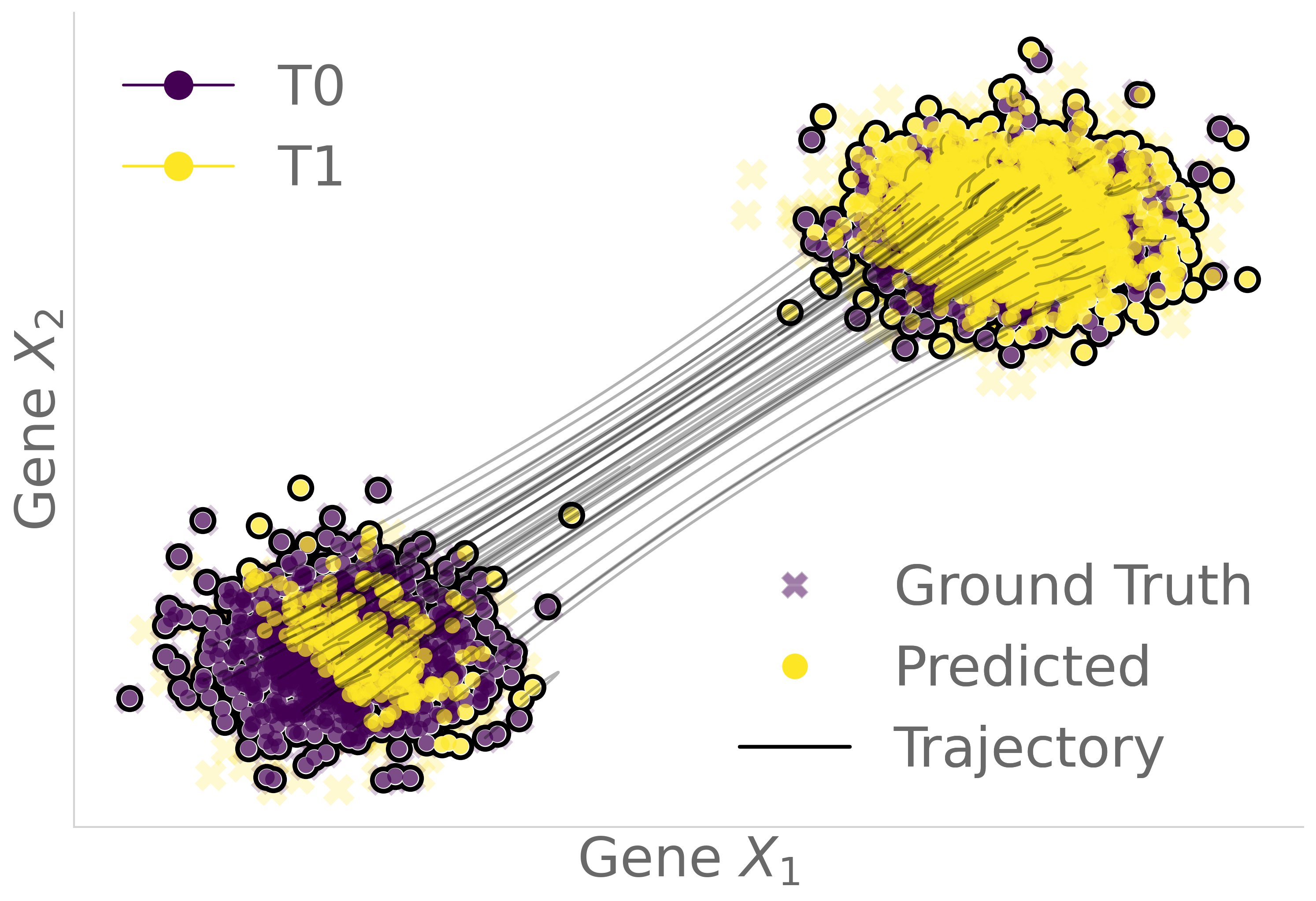} 
  \includegraphics[width=0.23\textwidth]{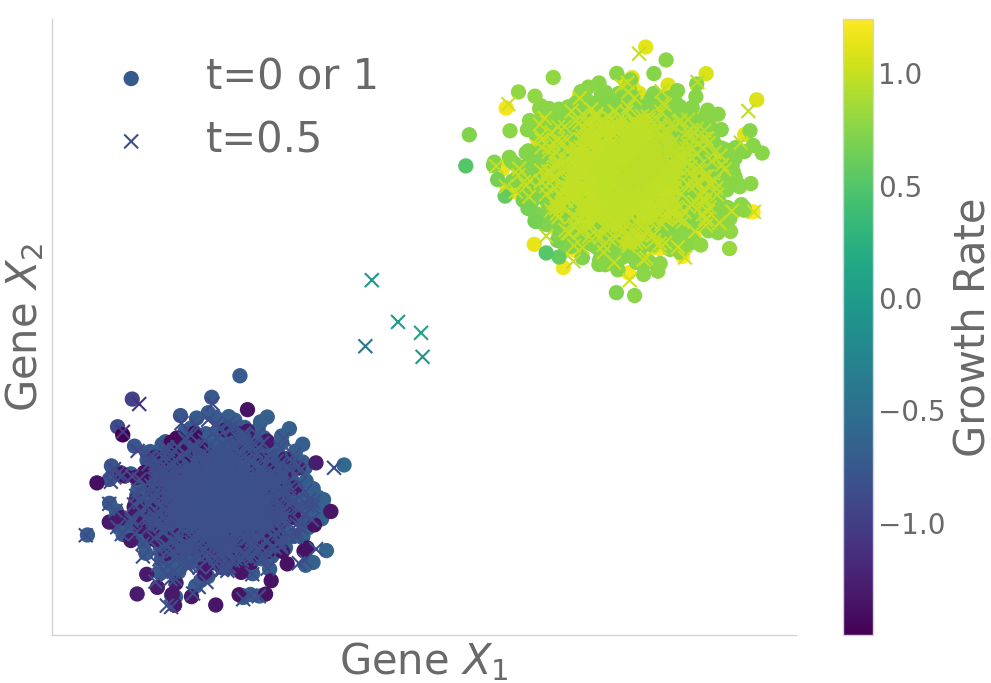} 
  \includegraphics[width=0.23\textwidth]{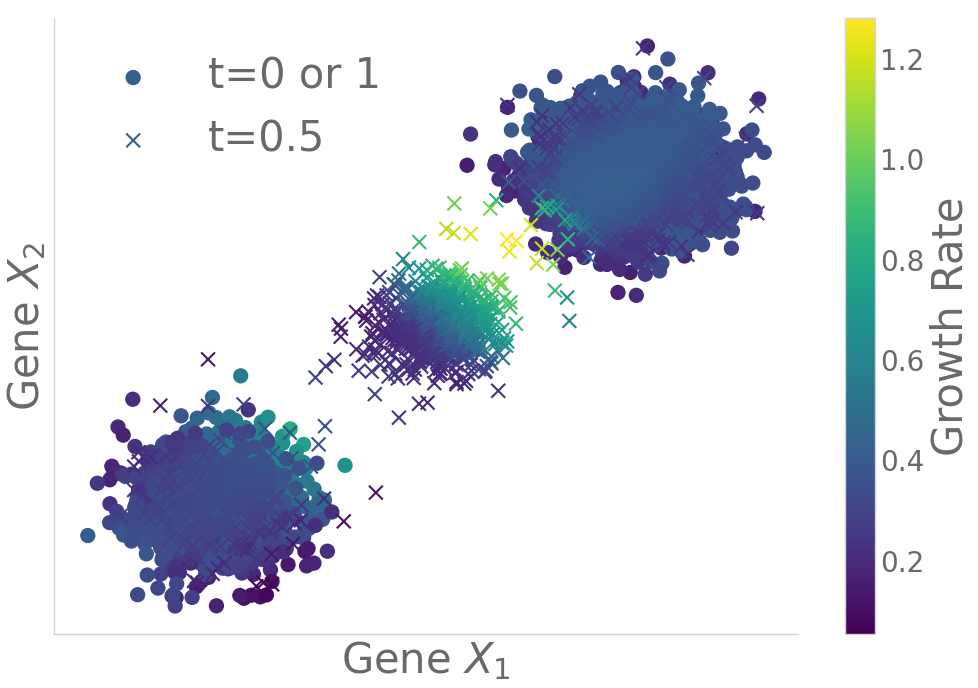} 
  
  \caption{Left: predicted trajectories by WFR-FM and SUDO; Right: predicted growth rate by WFR-FM and SUDO. We also plot the particles and their growth rate at t=0.5.}
  \label{fig:1}
\end{figure}

We further compared the growth rate induced by WFR penalty and only-growth penalty on 2D mouse hematopoiesis data \citep{TIGON,weinreb2020lineage}. As shown in Figure \ref{fig:real data growth}, WFR penalty resulted in both apoptosis and proliferation while only-growth penalty induced pure positive growth rate. In mouse hematopoiesis, proliferation is the primary driver of population dynamics, with apoptosis typically maintained at low levels. By incorporating this prior via an only-growth penalty, SUDO reconstructs more realistic growth rates than the standard WFR framework.

\begin{figure}
  \centering
  \includegraphics[width=0.32\textwidth]{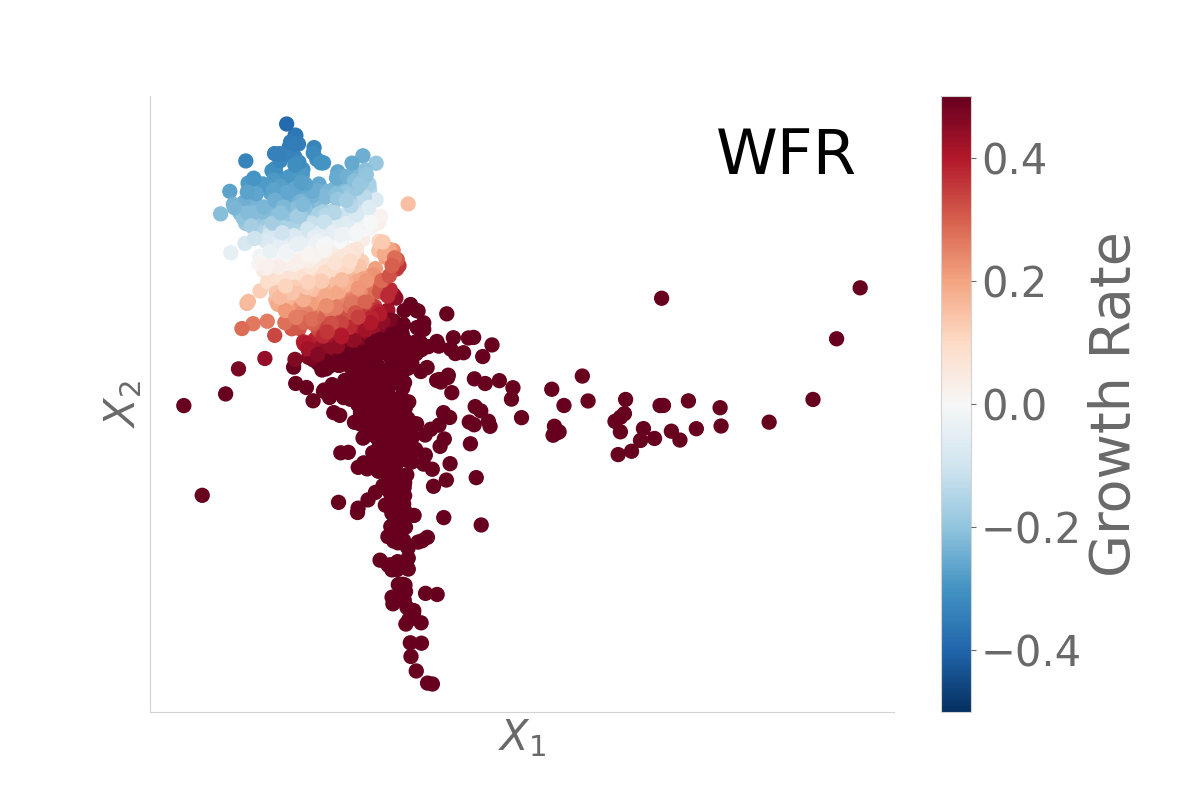} 
  \includegraphics[width=0.32\textwidth]{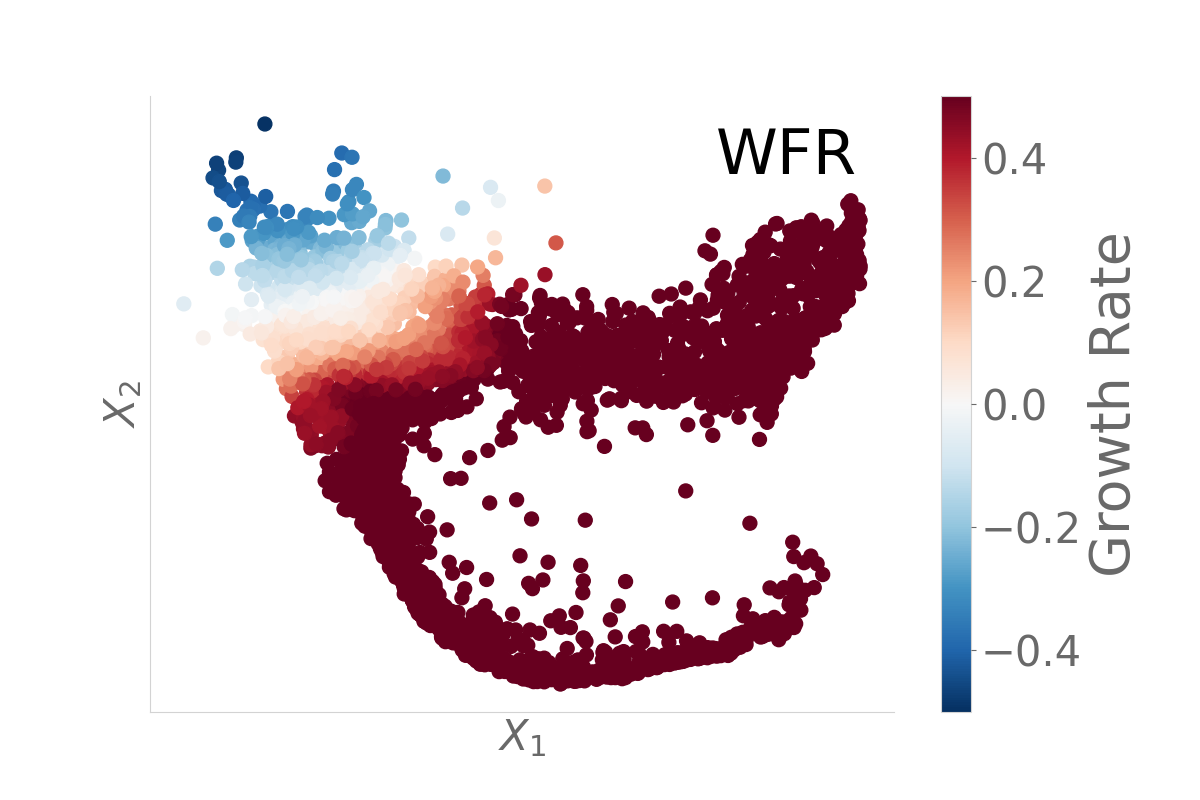} 
  \includegraphics[width=0.32\textwidth]{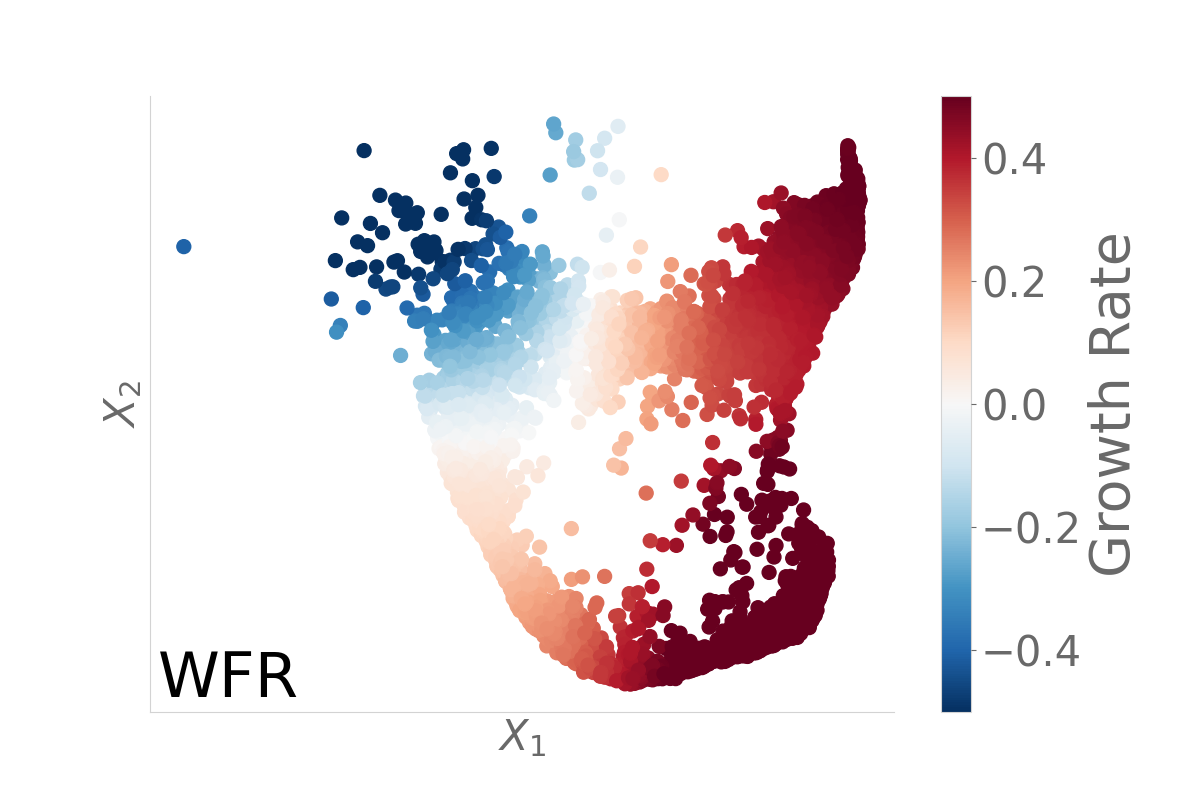} 
  \includegraphics[width=0.32\textwidth]{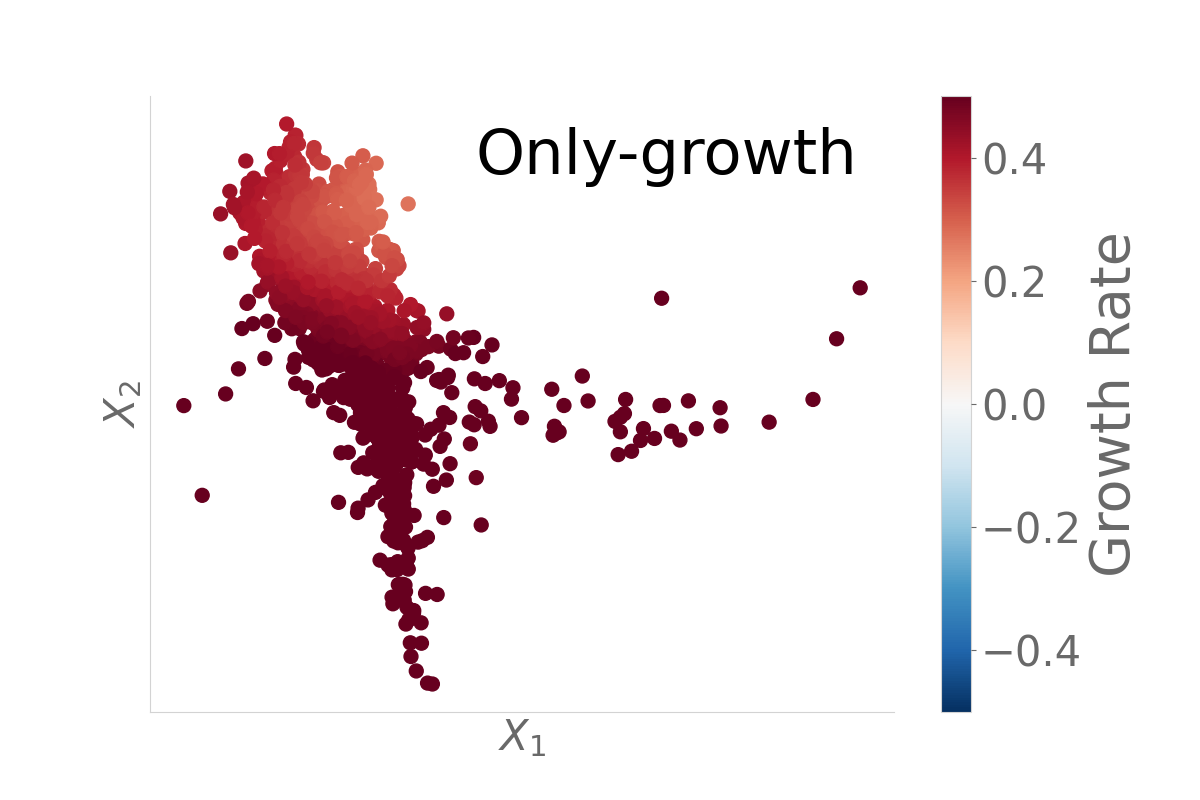} 
  \includegraphics[width=0.32\textwidth]{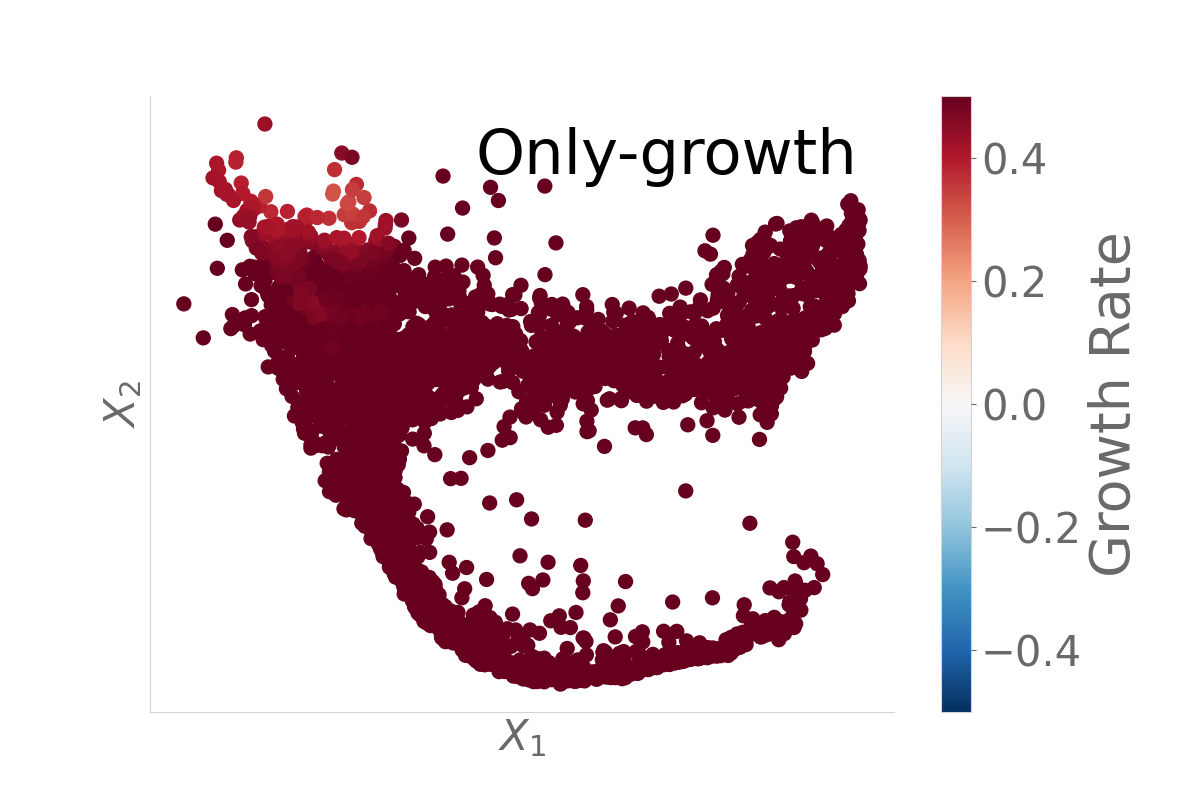} 
  \includegraphics[width=0.32\textwidth]{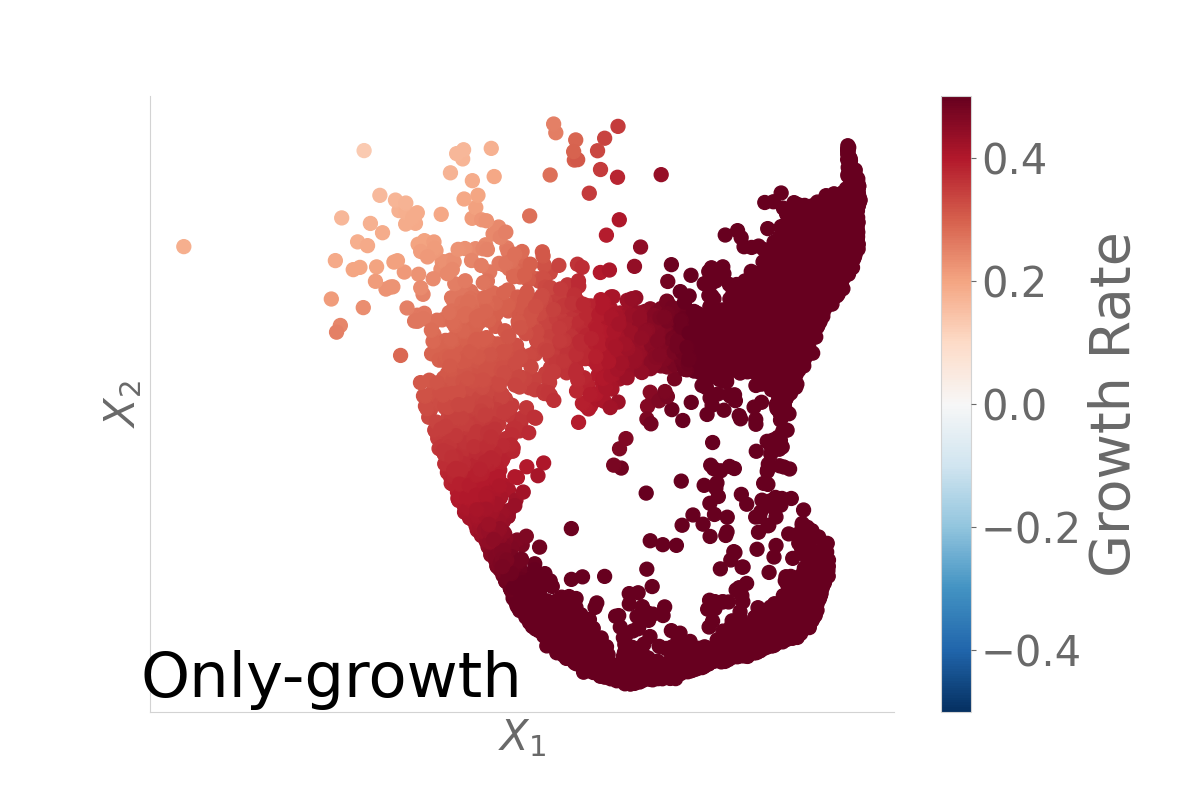} 
  
  \caption{Upper: predicted growth rate by WFR penalty; Lower: predicted growth rate by Only-growth penalty; From left to right: Time point 0, 1, 2. The colorbar is clipped to show the sign of the growth rate more clearly.}
  \label{fig:real data growth}
\end{figure}

\textbf{Quantitative advantage of only-growth penalty.} We further quantitatively demonstrate the advantage of only-growth penalty against quadratic WFR on Mouse hematopoiesis data. First, we conduct a hold one out experiment, which is training the model on the first and third time point (Day 2 and Day 6), and evaluating at the unseen second time point (Day 4). As shown in Table \ref{tab:only_growth_wfr}, only-growth penalty better approximates the ground truth dynamics of this proliferation dominant process.

\begin{table}[h!]
\centering
\caption{Hold one out experiment on Mouse. Mean and deviation over 5 runs are reported.}
\label{tab:only_growth_wfr}
\begin{tabular}{cccc}
\toprule
\textbf{WFR W1 $\downarrow$} &
\textbf{WFR RME $\downarrow$} &
\textbf{only-growth W1 $\downarrow$} &
\textbf{only-growth RME $\downarrow$} \\
\midrule
$0.4897 \pm 0.0094$ &
$0.4072 \pm 0.0182$ &
$\mathbf{0.1803 \pm 0.0069}$ &
$\mathbf{0.2030 \pm 0.0065}$ \\
\bottomrule
\end{tabular}
\end{table}
Following \citep{TIGON}, we further used barcode information to estimate the cell fate of cells from Day 2, and the growth rate of cells from Day 2 and Day 4. We then use the WFR and only-growth model shown in Figure \ref{fig:real data growth}, and compare the simulated trajectories based cell fate and the growth rates predicted by each model with the barcode-based ground truth. We estimated 95\% confidence intervals for only growth advantage against WFR, using 2,000 bootstrap resamples. As shown in Table \ref{tab:barcode_validation}, only-growth significantly outperforms WFR on all metrics, better recovering the barcode-based ground truth dynamics without using lineage information. For details of barcode-based ground truth and metrics, please see \ref{app:barcode} and \ref{app:metrics}.

\begin{table}[h!]
\centering
\caption{Comparison of WFR and the only-growth against barcode-based ground truth. The last row reports 95\% confidence intervals for the advantage of only-growth over WFR, estimated using 2,000 bootstrap resamples.}
\label{tab:barcode_validation}
\resizebox{\textwidth}{!}{
\begin{tabular}{lccccc}
\toprule
& \multicolumn{2}{c}{\textbf{Cell fate}}
& \multicolumn{3}{c}{\textbf{Growth rate}} \\
\cmidrule(lr){2-3}
\cmidrule(lr){4-6}

\textbf{Method}
& \textbf{Spearman $\uparrow$}
& \textbf{Dominant acc. $\uparrow$}
& \textbf{Day 2 RMSE $\downarrow$}
& \textbf{Day 4 RMSE $\downarrow$}
& \textbf{Pearson $\uparrow$} \\
\midrule

WFR
& $0.3870$
& $0.5932$
& $1.441$
& $1.873$
& $0.427$ \\

Only-growth
& $\mathbf{0.4175}$
& $\mathbf{0.6409}$
& $\mathbf{1.027}$
& $\mathbf{1.734}$
& $\mathbf{0.644}$ \\

Advantage (95\% CI)
& $[0.0061,\,0.0547]$
& $[0.0178,\,0.0778]$
& $[0.382,\,0.446]$
& $[0.124,\,0.154]$
& $[0.209,\,0.225]$ \\

\bottomrule
\end{tabular}
}
\end{table}

\section{Conclusion}
\label{sec:conclusion}
We present SUDO, a simulation-free framework for solving UDOT problems with general growth penalties. SUDO automatically learns the travelling Dirac, transport cost, and solves the semi-coupling via PGD. By doing so, our framework circumvents the dependency on analytical solutions of travelling Diracs that characterizes previous simulation-free approaches \citep{wfr_fm}. SUDO also delivers computational speedups over existing simulation-based algorithms \citep{DeepRUOT,sun2026variational}, thereby achieving a balance between flexibility and efficiency.

A key theoretical implication of this work is that not all growth penalties lead to meaningful coupled dynamics. We showed that concave penalties produce degenerate UDOT solutions where non-trivial growth and transport are separated, thereby identifying convex growth penalties as the relevant non-degenerate regime for coupled cellular dynamics. Empirically, we demonstrate SUDO’s effectiveness in solving UDOT problems, and highlight the need for general $\Psi$. We point out that different $\Psi$ can be viewed as different biological priors on cell proliferation and apoptosis. Integrating such prior knowledge into biological modeling enables more faithful reconstructions of cellular dynamics.

The limitations of SUDO are twofold. First, while SUDO provides an empirical approach for RUOT-like settings, its theoretical formulation is currently developed for deterministic dynamics, leaving a rigorous treatment of stochastic dynamics for future work. Second, SUDO requires a pre-specified growth penalty $\Psi$. In some applications, such as lineage tracing, partial trajectory information may be available. Incorporating such supervision to learn or calibrate the growth penalty from data, rather than specifying it a prior, is an important future direction.

\section*{Code Availability}
The code is available at https://github.com/Sirin6/SUDO.

\begin{ack}
This work was supported by the National Natural Science Foundation of China (NSFC No. 12225102 to L.Z., 8206100646 to P.Z., T2321001 to P.Z. \& L.Z., and 12288101 to P.Z. \& L.Z.) and the National Key R\&D Program of China No.2024YFA0919500 to L.Z. We thank the anonymous referees for their valuable feedback and constructive suggestions. We thank Jiantao Shen for insightful discussion about the idea of numerically approximating the travelling Dirac. We thank Jingyuan Lu for valuable discussion about validating SUDO using lineage information. 
\end{ack}

\bibliographystyle{plainnat}
\bibliography{example_paper}

@article{e2018deepritz,
  title={The Deep Ritz Method: A Deep Learning-Based Numerical Algorithm for Solving Variational Problems},
  author={E, W. and Yu, B.},
  journal={Commun. Math. Stat.},
  volume={6},
  pages={1--12},
  year={2018}
}

@inproceedings{
albergo2022building,
title={Building Normalizing Flows with Stochastic Interpolants},
author={Michael Samuel Albergo and Eric Vanden-Eijnden},
booktitle={The Eleventh International Conference on Learning Representations },
year={2023},
}

@inproceedings{
rohbeck2025modeling,
title={Modeling Complex System Dynamics with Flow Matching Across Time and Conditions},
author={Martin Rohbeck and Charlotte Bunne and Edward De Brouwer and Jan-Christian Huetter and Anne Biton and Kelvin Y. Chen and Aviv Regev and Romain Lopez},
booktitle={The Thirteenth International Conference on Learning Representations},
year={2025},
}

@inproceedings{lee2025mmsfm,
      title={Multi-Marginal Stochastic Flow Matching for High-Dimensional Snapshot Data at Irregular Time Points},
      author={Justin Lee and Behnaz Moradijamei and Heman Shakeri},
      booktitle={Forty-second International Conference on Machine Learning},
      year={2025},
      url={https://openreview.net/forum?id=ZLyb8DwXXE}
}

@inproceedings{
corso2025composing,
title={Composing Unbalanced Flows for Flexible Docking and Relaxation},
author={Gabriele Corso and Vignesh Ram Somnath and Noah Getz and Regina Barzilay and Tommi Jaakkola and Andreas Krause},
booktitle={The Thirteenth International Conference on Learning Representations},
year={2025},
}

@misc{cao2025taming,
      title={Taming Flow Matching with Unbalanced Optimal Transport into Fast Pansharpening}, 
      author={Zihan Cao and Yu Zhong and Liang-Jian Deng},
      year={2025},
      eprint={2503.14975},
      archivePrefix={arXiv},
      primaryClass={cs.CV},
      url={https://arxiv.org/abs/2503.14975}, 
}

@inproceedings{
petrovic2026curly,
title={Curly Flow Matching for Learning Non-gradient Field Dynamics},
author={Katarina Petrovi{\'c} and Lazar Atanackovic and Viggo Moro and Kacper Kapu{\'s}niak and Ismail Ilkan Ceylan and Michael M. Bronstein and Joey Bose and Alexander Tong},
booktitle={The Thirty-ninth Annual Conference on Neural Information Processing Systems},
year={2026},
url={https://openreview.net/forum?id=7cqKVDgFZQ}
}

@inproceedings{mioflow,
 author = {Huguet, Guillaume and Magruder, Daniel Sumner and Tong, Alexander and Fasina, Oluwadamilola and Kuchroo, Manik and Wolf, Guy and Krishnaswamy, Smita},
 booktitle = {Advances in Neural Information Processing Systems},
 editor = {S. Koyejo and S. Mohamed and A. Agarwal and D. Belgrave and K. Cho and A. Oh},
 pages = {29705--29718},
 publisher = {Curran Associates, Inc.},
 title = {Manifold Interpolating Optimal-Transport Flows for Trajectory Inference},
 volume = {35},
 year = {2022}
}

@inproceedings{UOT,
title={Unbalancedness in Neural Monge Maps Improves Unpaired Domain Translation},
author={Luca Eyring and Dominik Klein and Th{\'e}o Uscidda and Giovanni Palla and Niki Kilbertus and Zeynep Akata and Fabian J Theis},
booktitle={The Twelfth International Conference on Learning Representations},
year={2024},
}

@misc{neural_uot,
      title={Neural Unbalanced Optimal Transport via Cycle-Consistent Semi-Couplings}, 
      author={Frederike Lübeck and Charlotte Bunne and Gabriele Gut and Jacobo Sarabia del Castillo and Lucas Pelkmans and David Alvarez-Melis},
      year={2022},
      eprint={2209.15621},
      archivePrefix={arXiv},
      primaryClass={cs.LG},
      url={https://arxiv.org/abs/2209.15621}, 
}

@misc{mioflow2,
      title={MIOFlow 2.0: A unified framework for inferring cellular stochastic dynamics from single cell and spatial transcriptomics data}, 
      author={Xingzhi Sun and João Felipe Rocha and Brett Phelan and Dhananjay Bhaskar and Guillaume Huguet and Yanlei Zhang and Alexander Tong and Ke Xu and Oluwadamilola Fasina and Mark Gerstein and Natalia Ivanova and Christine L. Chaffer and Guy Wolf and Smita Krishnaswamy},
      year={2026},
      eprint={2603.22564},
      archivePrefix={arXiv},
      primaryClass={cs.LG},
      url={https://arxiv.org/abs/2603.22564}, 
}

@article{figalli2010optimal,
  title={The Optimal Partial Transport Problem},
  author={Figalli, Alessio},
  journal={Archive for Rational Mechanics and Analysis},
  volume={195},
  pages={533--560},
  year={2010}
}

@article{Kondratyev,
author = {Stanislav Kondratyev and L{\'e}onard Monsaingeon and Dmitry Vorotnikov},
title = {{A new optimal transport distance on the space of finite Radon measures}},
volume = {21},
journal = {Advances in Differential Equations},
number = {11/12},
publisher = {Khayyam Publishing, Inc.},
pages = {1117 -- 1164},
year = {2016},
}

@inproceedings{
wang2025joint,
title={Joint Velocity-Growth Flow Matching for Single-Cell Dynamics Modeling},
author={Dongyi Wang and Yuanwei Jiang and Zhenyi Zhang and Xiang Gu and Peijie Zhou and Jian Sun},
booktitle={The Thirty-ninth Annual Conference on Neural Information Processing Systems},
year={2025},
}

@article{peng2026stvcr,
  title={stVCR: spatiotemporal dynamics of single cells},
  author={Peng, Q. and Zhou, P. and Li, T.},
  journal={Nature Methods},
  volume={23},
  pages={542--553},
  year={2026}
}

@inproceedings{wfr_fm,
title={{WFR}-{FM}: Simulation-Free Dynamic Unbalanced Optimal Transport},
author={Qiangwei Peng and Zihan Wang and Junda Ying and Yuhao Sun and Qing Nie and Lei Zhang and Tiejun Li and Peijie Zhou},
booktitle={The Fourteenth International Conference on Learning Representations},
year={2026},
url={https://openreview.net/forum?id=1nqu7bK1mm}
}

@inproceedings{
DeepRUOT,
title={Learning stochastic dynamics from snapshots through regularized unbalanced optimal transport},
author={Zhenyi Zhang and Tiejun Li and Peijie Zhou},
booktitle={The Thirteenth International Conference on Learning Representations},
year={2025},
}

@article{TIGON,
  title={Reconstructing growth and dynamic trajectories from single-cell transcriptomics data},
  author={Sha, Yutong and Qiu, Yuchi and Zhou, Peijie and Nie, Qing},
  journal={Nature Machine Intelligence},
  volume={6},
  number={1},
  pages={25--39},
  year={2024},
  publisher={Nature Publishing Group UK London}
}

@inproceedings{NeuralODE,
 author = {Chen, Ricky T. Q. and Rubanova, Yulia and Bettencourt, Jesse and Duvenaud, David K},
 booktitle = {Advances in Neural Information Processing Systems},
 editor = {S. Bengio and H. Wallach and H. Larochelle and K. Grauman and N. Cesa-Bianchi and R. Garnett},
 pages = {},
 publisher = {Curran Associates, Inc.},
 title = {Neural Ordinary Differential Equations},
 volume = {31},
 year = {2018}
}

@inproceedings{
cfm_lipman,
title={Flow Matching for Generative Modeling},
author={Yaron Lipman and Ricky T. Q. Chen and Heli Ben-Hamu and Maximilian Nickel and Matthew Le},
booktitle={The Eleventh International Conference on Learning Representations },
year={2023},
}

@article{
cfm_tong,
title={Improving and generalizing flow-based generative models with minibatch optimal transport},
author={Alexander Tong and Kilian Fatras and Nikolay Malkin and Guillaume Huguet and Yanlei Zhang and Jarrid Rector-Brooks and Guy Wolf and Yoshua Bengio},
journal={Transactions on Machine Learning Research},
issn={2835-8856},
year={2024},
pages={1--34},
}

@article{kantorovich1942translocation,
 author={Kantorovich, Leonid V},
 ISSN = {00251909, 15265501},
 journal = {Management Science},
 number = {1},
 pages = {1--4},
 publisher = {INFORMS},
 title = {On the Translocation of Masses},
 volume = {5},
 year = {1958}
}

@article{benamou2000computational,
  title={A computational fluid mechanics solution to the Monge-Kantorovich mass transfer problem},
  author={Benamou, Jean-David and Brenier, Yann},
  journal={Numerische Mathematik},
  volume={84},
  number={3},
  pages={375--393},
  year={2000},
  publisher={Springer-Verlag Berlin/Heidelberg}
}

@article{chizat2018interpolating,
  title={An interpolating distance between optimal transport and Fisher--Rao metrics},
  author={Chizat, Lenaic and Peyr{\'e}, Gabriel and Schmitzer, Bernhard and Vialard, Fran{\c{c}}ois-Xavier},
  journal={Foundations of Computational Mathematics},
  volume={18},
  number={1},
  pages={1--44},
  year={2018},
  publisher={Springer}
}

@article{chizat2018unbalanced,
  title={Unbalanced optimal transport: Dynamic and Kantorovich formulations},
  author={Chizat, Lenaic and Peyr{\'e}, Gabriel and Schmitzer, Bernhard and Vialard, Fran{\c{c}}ois-Xavier},
  journal={Journal of Functional Analysis},
  volume={274},
  number={11},
  pages={3090--3123},
  year={2018},
  publisher={Elsevier}
}

@article{liero2018optimal,
  title={Optimal entropy-transport problems and a new Hellinger--Kantorovich distance between positive measures},
  author={Liero, Matthias and Mielke, Alexander and Savar{\'e}, Giuseppe},
  journal={Inventiones mathematicae},
  volume={211},
  pages={969--1117},
  year={2018},
  publisher={Springer}
}

@article{schrodinger1932sur,
  title={Sur la th{\'e}orie relativiste de l'{\'e}lectron et l'interpr{\'e}tation de la m{\'e}canique quantique},
  author={Schr\"odinger, Erwin},
  journal={Annales de l'Institut Henri Poincar{\'e}},
  volume={2},
  number={4},
  pages={269--310},
  year={1932},
  publisher={Gauthier-Villars}
}

@inproceedings{bortoli2021diffusion,
 author = {De Bortoli, Valentin and Thornton, James and Heng, Jeremy and Doucet, Arnaud},
 booktitle = {Advances in Neural Information Processing Systems},
 editor = {M. Ranzato and A. Beygelzimer and Y. Dauphin and P.S. Liang and J. Wortman Vaughan},
 pages = {17695--17709},
 publisher = {Curran Associates, Inc.},
 title = {Diffusion Schr\"odinger Bridge with Applications to Score-Based Generative Modeling},
 volume = {34},
 year = {2021}
}

@InProceedings{bunne2023schrodinger,
  title = 	 {The Schrödinger Bridge between Gaussian Measures has a Closed Form},
  author =       {Bunne, Charlotte and Hsieh, Ya-Ping and Cuturi, Marco and Krause, Andreas},
  booktitle = 	 {Proceedings of The 26th International Conference on Artificial Intelligence and Statistics},
  pages = 	 {5802--5833},
  year = 	 {2023},
  editor = 	 {Ruiz, Francisco and Dy, Jennifer and van de Meent, Jan-Willem},
  volume = 	 {206},
  series = 	 {Proceedings of Machine Learning Research},
  month = 	 {25--27 Apr},
  publisher =    {PMLR},
}

@InProceedings{tong2023simulationfree,
  title = 	 {Simulation-Free {S}chrödinger Bridges via Score and Flow Matching},
  author =       {Tong, Alexander Y. and Malkin, Nikolay and Fatras, Kilian and Atanackovic, Lazar and Zhang, Yanlei and Huguet, Guillaume and Wolf, Guy and Bengio, Yoshua},
  booktitle = 	 {Proceedings of The 27th International Conference on Artificial Intelligence and Statistics},
  pages = 	 {1279--1287},
  year = 	 {2024},
  editor = 	 {Dasgupta, Sanjoy and Mandt, Stephan and Li, Yingzhen},
  volume = 	 {238},
  series = 	 {Proceedings of Machine Learning Research},
  month = 	 {02--04 May},
  publisher =    {PMLR},
}

@misc{BSB,
      title={Regularized unbalanced optimal transport as entropy minimization with respect to branching Brownian motion}, 
      author={Aymeric Baradat and Hugo Lavenant},
      year={2021},
      eprint={2111.01666},
      archivePrefix={arXiv},
      primaryClass={math.PR},
      url={https://arxiv.org/abs/2111.01666},
}

@article{moon2019visualizing,
  title={Visualizing structure and transitions in high-dimensional biological data},
  author={Moon, Kevin R and Van Dijk, David and Wang, Zheng and Gigante, Scott and Burkhardt, Daniel B and Chen, William S and Yim, Kristina and Elzen, Antonia van den and Hirn, Matthew J and Coifman, Ronald R and others},
  journal={Nature biotechnology},
  volume={37},
  pages={1482--1492},
  year={2019},
  publisher={Nature Publishing Group US New York}
}

@article{cook2020context,
  title={Context specificity of the EMT transcriptional response},
  author={Cook, David P and Vanderhyden, Barbara C},
  journal={Nature communications},
  volume={11},
  pages={2142},
  year={2020},
  publisher={Nature Publishing Group UK London}
}

@article{cannoodt2021spearheading,
  title={Spearheading future omics analyses using dyngen, a multi-modal simulator of single cells},
  author={Cannoodt, Robrecht and Saelens, Wouter and Deconinck, Louise and Saeys, Yvan},
  journal={Nature communications},
  volume={12},
  pages={3942},
  year={2021},
  publisher={Nature Publishing Group UK London}
}

@article{weinreb2020lineage,
  title={Lineage tracing on transcriptional landscapes links state to fate during differentiation},
  author={Weinreb, Caleb and Rodriguez-Fraticelli, Alejo and Camargo, Fernando D and Klein, Allon M},
  journal={Science},
  volume={367},
  number={6479},
  pages={eaaw3381},
  year={2020},
  publisher={American Association for the Advancement of Science}
}

@inproceedings{adam,
  author={Diederik P. Kingma and Jimmy Ba},
  title={Adam: A Method for Stochastic Optimization},
  year={2015},
  cdate={1420070400000},
  booktitle={ICLR (Poster)},
}

@article{pot,
  title={Pot: Python optimal transport},
  author={Flamary, R{\'e}mi and Courty, Nicolas and Gramfort, Alexandre and Alaya, Mokhtar Z and Boisbunon, Aur{\'e}lie and Chambon, Stanislas and Chapel, Laetitia and Corenflos, Adrien and Fatras, Kilian and Fournier, Nemo and others},
  journal={Journal of Machine Learning Research},
  volume={22},
  number={78},
  pages={1--8},
  year={2021}
}

@misc{fatras2021minibatchOT,
      title={Minibatch optimal transport distances; analysis and applications}, 
      author={Kilian Fatras and Younes Zine and Szymon Majewski and Rémi Flamary and Rémi Gribonval and Nicolas Courty},
      year={2021},
      eprint={2101.01792},
      archivePrefix={arXiv},
      primaryClass={stat.ML},
}

@article{waddingot,
  title={Optimal-transport analysis of single-cell gene expression identifies developmental trajectories in reprogramming},
  author={Schiebinger, Geoffrey and Shu, Jian and Tabaka, Marcin and Cleary, Brian and Subramanian, Vidya and Solomon, Aryeh and Gould, Joshua and Liu, Siyan and Lin, Stacie and Berube, Peter and others},
  journal={Cell},
  volume={176},
  number={4},
  pages={928--943},
  year={2019}
}

@article{moscot,
  title={Mapping cells through time and space with moscot},
  author={Klein, Dominik and Palla, Giovanni and Lange, Marius and others},
  journal={Nature},
  volume={638},
  pages={1065-1075},
  year={2025},
  publisher={Nature Publishing Group UK London}
}

@inproceedings{kapusniak2024metric,
 author = {Kapu\'{s}niak, Kacper and Potaptchik, Peter and Reu, Teodora and Zhang, Leo and Tong, Alexander and Bronstein, Michael and Bose, Avishek Joey and Di Giovanni, Francesco},
 booktitle = {Advances in Neural Information Processing Systems},
 doi = {10.52202/079017-4291},
 editor = {A. Globerson and L. Mackey and D. Belgrave and A. Fan and U. Paquet and J. Tomczak and C. Zhang},
 pages = {135011--135042},
 publisher = {Curran Associates, Inc.},
 title = {Metric Flow Matching for Smooth Interpolations on the Data Manifold},
 volume = {37},
 year = {2024}
}

@inproceedings{klein2024genot,
 author = {Klein, Dominik and Uscidda, Th\'{e}o and Theis, Fabian and Cuturi, Marco},
 booktitle = {Advances in Neural Information Processing Systems},
 doi = {10.52202/079017-3301},
 editor = {A. Globerson and L. Mackey and D. Belgrave and A. Fan and U. Paquet and J. Tomczak and C. Zhang},
 pages = {103897--103944},
 publisher = {Curran Associates, Inc.},
 title = {GENOT: Entropic (Gromov) Wasserstein Flow Matching with Applications to Single-Cell Genomics},
 volume = {37},
 year = {2024}
}

@InProceedings{trajectorynet,
  title = 	 {{T}rajectory{N}et: A Dynamic Optimal Transport Network for Modeling Cellular Dynamics},
  author =       {Tong, Alexander and Huang, Jessie and Wolf, Guy and Van Dijk, David and Krishnaswamy, Smita},
  booktitle = 	 {Proceedings of the 37th International Conference on Machine Learning},
  pages = 	 {9526--9536},
  year = 	 {2020},
  editor = 	 {III, Hal Daumé and Singh, Aarti},
  volume = 	 {119},
  series = 	 {Proceedings of Machine Learning Research},
  month = 	 {13--18 Jul},
  publisher =    {PMLR},
}

@inproceedings{shi2024diffusion,
 author = {Shi, Yuyang and De Bortoli, Valentin and Campbell, Andrew and Doucet, Arnaud},
 booktitle = {Advances in Neural Information Processing Systems},
 editor = {A. Oh and T. Naumann and A. Globerson and K. Saenko and M. Hardt and S. Levine},
 pages = {62183--62223},
 publisher = {Curran Associates, Inc.},
 title = {Diffusion Schr\"odinger Bridge Matching},
 volume = {36},
 year = {2023}
}

@article{bunne2023learning,
  title={Learning single-cell perturbation responses using neural optimal transport},
  author={Bunne, Charlotte and Stark, Stefan G and Gut, Gabriele and others},
  journal={Nature methods},
  volume={20},
  pages={1759--1768},
  year={2023},
  publisher={Nature Publishing Group US New York}
}

@inproceedings{
sun2026variational,
title={Variational Regularized Unbalanced Optimal Transport: Single Network, Least Action},
author={Yuhao Sun and Zhenyi Zhang and Zihan Wang and Tiejun Li and Peijie Zhou},
booktitle={The Thirty-ninth Annual Conference on Neural Information Processing Systems},
year={2026},
url={https://openreview.net/forum?id=Iguyg0LULD}
}

@inproceedings{duchi2008efficient,
  title={Efficient projections onto the l1-ball for learning in high dimensions},
  author={Duchi, John and Shalev-Shwartz, Shai and Singer, Yoram and Chandra, Tushar},
  booktitle={Proceedings of the 25th international conference on Machine learning},
  year={2008}
}

@misc{USB,
      title={Beyond Continuity: Simulation-free Reconstruction of Discrete Branching Dynamics from Single-cell Snapshots}, 
      author={Junda Ying and Yuxuan Wang and Bowen Yang and Peijie Zhou and Lei Zhang},
      year={2026},
      eprint={2605.00545},
      archivePrefix={arXiv},
      primaryClass={cs.LG},
      url={https://arxiv.org/abs/2605.00545}, 
}

@book{heinonen2012nonlinear,
  title={Nonlinear Potential Theory of Degenerate Elliptic Equations},
  author={Heinonen, J. and Kilpel{\"a}inen, T. and Martio, O.},
  isbn={9780486149257},
  series={Dover Books on Mathematics},
  year={2012},
  publisher={Dover Publications}
}

@article{SiLU,
title = {Sigmoid-weighted linear units for neural network function approximation in reinforcement learning},
journal = {Neural Networks},
volume = {107},
pages = {3-11},
year = {2018},
note = {Special issue on deep reinforcement learning},
issn = {0893-6080},
author = {Stefan Elfwing and Eiji Uchibe and Kenji Doya},
}

@book{dacorogna2008direct,
  title={Direct Methods in the Calculus of Variations},
  author={Dacorogna, Bernard},
  series={Applied Mathematical Sciences},
  edition={2},
  year={2008},
  publisher={Springer New York},
  isbn={978-0-387-35779-9}
}

@article{spring,
    author = {Weinreb, Caleb and Wolock, Samuel and Klein, Allon M},
    title = {SPRING: a kinetic interface for visualizing high dimensional single-cell expression data},
    journal = {Bioinformatics},
    volume = {34},
    number = {7},
    pages = {1246-1248},
    year = {2018},
    month = {04},
    issn = {1367-4803}
}
\newpage








\appendix

\section{Proofs}
\label{appendix:proof}
Proofs of theorems and propositions

\subsection{Proof of Theorem \ref{thm:concave}}
\label{pf:concave}
\textbf{Theorem \ref{thm:concave}.}
\textit{For \(\Psi\) which is concave on both sides of \(g_0\), the solution of UDOT problem (\ref{eq:UDOT}) separates the non-trivial growth and transport.}

\textit{Proof.} By concavity, it is easy to show that \(\Psi(g) = O(|g|)\) as \(|g|\to\infty\). Recall the UDOT problem
\begin{equation}
\label{eq:problem}
\begin{aligned}
&\text{UDOT}(\mu_0,\mu_1) =
\inf_{\rho,g,\bm{u}} \int_0^1\int_{\mathcal{X}}  \, \frac{1}{2}\Big(\Vert \bm{u}(\bm{x},t)\Vert_2^2+\Psi(g(\bm{x},t))\Big)\rho_t(\bm{x})\mathrm{d} \bm{x}\mathrm{d}t \\
&\text{s.t.}\ \ \ \ \ \ \ \ \ \ \ \ \ \ \partial_t\rho+\nabla_{\bm{x}}\cdot(\rho \bm{u})=\rho g,\ \rho_0=\mu_0,\ \rho_1=\mu_1
\end{aligned}
\end{equation}

\textbf{Case 1.} \(\Psi(g)=o(|g|)\quad(|g|\to\infty)\)

Consider a family of \((g,\bm{u})\) defined as following.
\begin{equation}
\label{eq:concave g}
\left \{
\begin{aligned}
&g(\bm{x},t)=-\frac{N\mu_0(\bm{x})}{(1-Nt)\mu_0(\bm{x})}=-\frac{N}{1-Nt}\quad&\bm{u}(\bm{x},t)=\bm{0}\quad &t\in[0,\frac{1}{N}]\\
&g(\bm{x},t)=0\quad&\text{arbitrary }\bm{u}(\bm{x},t)\quad &t\in[\frac{1}{N},\frac{N-1}{N}]\\
&g(\bm{x},t)=\frac{N\mu_1(\bm{x})}{(Nt-N+1)\mu_1(\bm{x})}=\frac{N}{Nt-N+1}\quad&\bm{u}(\bm{x},t)=\bm{0}\quad &t\in[\frac{N-1}{N},1]\\
\end{aligned}
\right .
\end{equation}
Intuitively, the total mass vanishes linearly in \([0,\frac{1}{N}]\), and grows linearly in \([\frac{N-1}{N},1]\). One can easily check that \((g,\bm{u})\) defined above transports \(\mu_0\) to \(\mu_1\). The transport cost in \([\frac{1}{N},\frac{N-1}{N}]\) is \(0\) due to the zero total mass. Thus, the total transport cost is

\begin{equation}
\begin{aligned}
&\int_0^{\frac{1}{N}}\int_{\mathcal{X}}\Psi(g(\bm{x},t))\rho(\bm{x},t)\mathrm{d}\bm{x}\mathrm{d}t+\int^1_{\frac{N-1}{N}}\int_{\mathcal{X}}\Psi(g(\bm{x},t))\rho(\bm{x},t)\mathrm{d}\bm{x}\mathrm{d}t\\
=&\int_0^{\frac{1}{N}}\Psi(-\frac{N}{1-Nt})(1-Nt)\int_{\mathcal{X}}\mu_0(\bm{x})\mathrm{d}\bm{x}\mathrm{d}t+\int^1_{\frac{N-1}{N}}\Psi(\frac{N}{Nt-N+1})(Nt-N+1)\int_{\mathcal{X}}\mu_1(\bm{x})\mathrm{d}\bm{x}\mathrm{d}t\\
=&M_0\int_0^{\frac{1}{N}}\Psi(-\frac{N}{1-Nt})(1-Nt)\mathrm{d}t+M_1\int^1_{\frac{N-1}{N}}\Psi(\frac{N}{Nt-N+1})(Nt-N+1)\mathrm{d}t\\
\end{aligned}
\end{equation}

where \(M_0,M_1\) are the total mass of \(\mu_0,\mu_1\). Since \(\Psi(g)=o(|g|)\quad(|g|\to\infty)\),  \(\forall \epsilon>0, \exists N>0 \ s.t.\ \forall |g|>N, \Psi(g)\le\epsilon |g|\). 

\begin{equation}
\begin{aligned}
&M_0\int_0^{\frac{1}{N}}\Psi(-\frac{N}{1-Nt})(1-Nt)\mathrm{d}t+M_1\int^1_{\frac{N-1}{N}}\Psi(\frac{N}{Nt-N+1})(Nt-N+1)\mathrm{d}t\\
\le&M_0\epsilon\int_0^{\frac{1}{N}}\frac{N}{1-Nt}(1-Nt)\mathrm{d}t+M_1\epsilon\int^1_{\frac{N-1}{N}}\frac{N}{Nt-N+1}(Nt-N+1)\mathrm{d}t\\
=&(M_0+M_1)\epsilon \to0\quad(\epsilon\to0)
\end{aligned}
\end{equation}

Hence, the optimal UDOT cost of this case is \(0\). Intuitively, the corresponding dynamics is that at time $t=0$, the mass of $\mu_0$ decays to zero at an infinite rate, while at time $t=1$, the mass grows at an infinite rate to form $\mu_1$. 

\textbf{Case 2.} \(\Psi(g)\sim|g|\quad(|g|\to\infty)\)

Assume $$ \lim_{g \to +\infty} \frac{\Psi(g)}{g} = \alpha, \quad -\lim_{g \to -\infty} \frac{\Psi(g)}{g} = \beta $$

Using the relaxation theorem for non-convex variational problems \citep{dacorogna2008direct}, for linear growth functions, the infimum of (\ref{eq:problem}) is equal to the infimum of its convex envelope relaxation.

\begin{equation}
\label{eq:convex relaxation}
\begin{aligned}
&
\inf_{\rho,g,\bm{u}} \int_0^1\int_{\mathcal{X}}  \, \frac{1}{2}\Big(\Vert \bm{u}(\bm{x},t)\Vert_2^2+\Psi^{**}(g(\bm{x},t))\Big)\rho_t(\bm{x})\mathrm{d} \bm{x}\mathrm{d}t \\
&\text{s.t.}\ \ \ \ \ \ \ \partial_t\rho+\nabla_{\bm{x}}\cdot(\rho \bm{u})=\rho g,\ \rho_0=\mu_0,\ \rho_1=\mu_1
\end{aligned}
\end{equation}

where \(\Psi^{**}\) is the convex envelope if \(\Psi\). And there exists a minimizing sequence of (\ref{eq:problem}) which converges to the solution of (\ref{eq:convex relaxation}) in weak sense. In this case, 
\begin{equation}
\Psi^{**}(g) = \alpha \max(g-g_0,0) + \beta \max(g_0-g,0) = \alpha (g-g_0)^+ + \beta (g-g_0)^- 
\end{equation} 

Let \(\tilde{s}=(g-g_0)\rho\), the corresponding relaxed problem is
\begin{equation}
\label{eq:relax2}
\begin{aligned}
&\inf_{\rho, \bm{u}, \tilde{s}}  \frac{1}{2}\int_0^1\int_{\mathcal{X}} \Big(\|\bm{u}(\bm{x},t)\|^2\rho(\bm{x},t) +\alpha \tilde{s}^+ + \beta\tilde{s}^-\Big)\mathrm{d}x\mathrm{d}t \\
&s.t. \ \partial_t \rho + \nabla_{\bm{x}} \cdot (\rho \bm{u}) = g_0 \rho+\tilde{s}^+ -\tilde{s}^-,\ \rho_0=\mu_0,\ \rho_1=\mu_1
\end{aligned}
\end{equation}

Change the variables
\begin{equation}
\left \{
\begin{aligned}
&\hat{\rho}(\bm{x},t) = \rho(\bm{x},t) e^{-g_0 t}\\ 
&\hat{\bm{u}}(\bm{x},t) = \bm{u}(\bm{x},t)\\
&\hat{s}(\bm{x},t) = \tilde{s}(\bm{x},t) e^{-g_0 t}
\end{aligned}
\right .
\end{equation}

(\ref{eq:relax2}) becomes

\begin{equation}
\label{eq:relax3}
\begin{aligned}
&\inf_{\hat{\rho}, \hat{\bm{u}}, \hat{s}}  \frac{1}{2}\int_0^1\int_{\mathcal{X}} e^{g_0t}\Big(\|\hat{\bm{u}}(\bm{x},t)\|^2\hat{\rho}(\bm{x},t) +\alpha \hat{s}^+ + \beta\hat{s}^-\Big)\mathrm{d}x\mathrm{d}t \\
&s.t. \ \ \partial_t \hat{\rho} + \nabla_{\bm{x}} \cdot (\hat{\rho} \hat{\bm{u}}) = \hat{s}^+ - \hat{s}^-,\ \hat{\rho}_0=\mu_0,\ \hat{\rho}_1=\mu_1e^{-g_0}
\end{aligned}
\end{equation}

Construct the Lagrangian functional of the constraint optimization problem (\ref{eq:relax3})
\begin{equation}
\begin{aligned}
&\mathcal{L}(\hat{\rho}, \hat{\bm{u}}, \hat{s}^+, \hat{s}^-, \phi) \\&= \int_0^1 \int_{\mathcal{X}} \left[ \frac{1}{2} e^{g_0 t} \left( \|\hat{\bm{u}}\|^2 \hat{\rho} + \alpha \hat{s}^+ + \beta \hat{s}^- \right) + \phi \left( \partial_t \hat{\rho} + \nabla_x \cdot (\hat{\rho}\hat{\bm{u}}) - \hat{s}^+ + \hat{s}^- \right) \right] \mathrm{d}\bm{x} \mathrm{d}t   
\end{aligned}
\end{equation}
where \(\phi\) is the Lagrangian multiplier. Integrate by part, we have
\begin{equation}
\begin{aligned}
\mathcal{L} = &\int_0^1 \int_{\mathcal{X}} \left[ \hat{\rho} \left( \frac{1}{2} e^{g_0 t} \|\hat{\bm{u}}\|^2 - \partial_t \phi - \hat{\bm{u}} \cdot \nabla_{\bm{x}} \phi \right) + \hat{s}^+ \left( \frac{1}{2} \alpha e^{g_0 t} - \phi \right) + \hat{s}^- \left( \frac{1}{2} \beta e^{g_0 t} + \phi \right) \right] \mathrm{d}\bm{x} \mathrm{d}t \\
&+ \int_{\mathcal{X}} \phi(\bm{x},1) \hat{\rho}_1(\bm{x}) \mathrm{d}\bm{x} - \int_{\mathcal{X}} \phi(\bm{x},0) \hat{\rho}_0(\bm{x}) \mathrm{d}\bm{x}  
\end{aligned}
\end{equation}

By KKT conditions, we have the following properties of the optimal solution.
\begin{equation}
\label{eq:KKT}
\left \{
\begin{aligned}
&\hat{\rho} \left( e^{g_0 t} \hat{\bm{u}} - \nabla_{\bm{x}} \phi \right) = 0 &(\text{stationary}\ \ \hat{\bm{u}})\\
&\frac{1}{2} \alpha e^{g_0 t} - \phi(\bm{x},t) \ge 0 &(\text{stationary}\ \ \hat{s}^+)\\
&\frac{1}{2} \beta e^{g_0 t} + \phi(\bm{x},t) \ge 0 &(\text{stationary}\ \ \hat{s}^-)\\ 
&\hat{s}^+(\bm{x},t) \left( \frac{1}{2} \alpha e^{g_0 t} - \phi(\bm{x},t) \right) = 0 &(\text{complementary slackness}\ \ \hat{s}^+)\\
&\hat{s}^-(\bm{x},t) \left( \frac{1}{2} \beta e^{g_0 t} + \phi(\bm{x},t) \right) = 0 &(\text{complementary slackness}\ \ \hat{s}^-)\\
\end{aligned}
\right .
\end{equation}

If \(\hat{s}\neq0\) on some domain \(A\), we have either \(\hat{s}^+>0\) or \(\hat{s}^->0\). By complementary slackness, it implies that \(\phi(\bm{x},t)=\frac{1}{2} \alpha e^{g_0 t}\) or \(\phi(\bm{x},t)=-\frac{1}{2} \beta e^{g_0 t}\) on \(A\). Both implies that \(\phi(\bm{x},t)\) is independent to the spatial variable \(\bm{x}\), hence \(\hat{\bm{u}}(\bm{x},t)=e^{-g_0t}\nabla_{\bm{x}} \phi=\bm{0}\) on \(A\). Changing back to the original variables, we have that \(\tilde{s}(\bm{x},t)\bm{u}(\bm{x},t)=\bm{0}\), i.e. the non-trivial part of the growth (other than just \(g_0\)) and transport are separated in the optimal solution. As a result of the relaxation theorem, we have a minimizing sequence of the original problem (\ref{eq:problem}) which converges to the solution above in weak sense. 

In summary, both cases induced an unbalanced transport where non-trivial growth and transport are separated. The proof is completed. \hfill $\square$

\subsection{Proof of Theorem \ref{thm:EL eqn}}
\label{pf:EL eqn}
\textbf{Theorem \ref{thm:EL eqn}.}
\textit{There exists two 1-dimensional functions $k(t),l(t)$ which only depend on $d=\Vert \bm{x}_1-\bm{x}_0\Vert,r=\frac{m_1}{m_0}$ s.t. $\bm{x}(t)=\bm{x}_0+k(t)\frac{\bm{x}_1-\bm{x}_0}{d}$, $m(t)=m_0l(t)$, i.e. the travelling Dirac is straight.}

\textit{Proof.} Recall the minimization problem (\ref{eq:Dirac UDOT})
\begin{equation}
\label{eq:Dirac UDOT appendix}
\begin{aligned}
&\text{UDOT-DD}(m_0\delta_{\bm{x}_0},m_1\delta_{\bm{x}_1}) = \inf_{m,\bm{x}} \int_0^1\frac{1}{2}\Big(\Vert \dot{\bm{x}}(t)\Vert_2^2+ \Psi(\frac{\dot{m}(t)}{m(t)})\Big)m(t)\mathrm{d}t \\
&\text{s.t.} \ \ \ \ \ \ \ \ \ \ \ \ \ \ \ \ \ \ \ \ \ m(0)=m_0,m(1)=m_1,\bm{x}(0)=\bm{x}_0,\bm{x}(1)=\bm{x}_1&
\end{aligned}
\end{equation}
We derive the optimality conditions via Euler-Lagrange equation.

\begin{equation}
\left \{
\begin{aligned}
&\frac{\mathrm{d}}{\mathrm{d}t}(m\dot{\bm{x}})=\bm{0}\quad\quad &(\frac{\partial}{\partial\bm{x}}=\frac{\mathrm{d}}{\mathrm{d}t}\frac{\partial}{\partial\dot{\bm{x}}})\\
&|\dot{\bm{x}}|^2+\Psi-\frac{\dot{m}}{m}\Psi'=\frac{\ddot{m}m-\dot{m}^2}{m^2}\Psi''\quad\quad &(\frac{\partial}{\partial m}=\frac{\mathrm{d}}{\mathrm{d}t}\frac{\partial}{\partial\dot{m}})
\end{aligned}
\right .
\end{equation}

The first equation implies momentum conservation in the travelling Dirac, yielding a constant velocity direction. This restricts the spatial trajectory to a straight line; thus, it is sufficient to solely parameterize the relative position \(k(t)\) of the particle along this line at time $t$. Hence there exists a 1-dimensional function \(k(t)\  s.t.\  \bm{x}(t)=\bm{x}_0+k(t)\frac{\bm{x}_1-\bm{x}_0}{d}\).

Replacing \(\bm{x}(t)\) and \(m(t)\) with \(\bm{x}(t)+\bm{\alpha}\) and \(\beta m(t)\) where \(\bm{\alpha},\beta\) are constant vector and scalar, it is easy to see that the following problem shares the same minimizer with (\ref{eq:Dirac UDOT appendix}).
\begin{equation}
\label{eq:Dirac UDOT shift}
\begin{aligned}
&\text{UDOT-DD}(\beta m_0\delta_{\bm{x}_0+\bm{\alpha}},\beta m_1\delta_{\bm{x}_1+\bm{\alpha}}) = \inf_{m,\bm{x}} \int_0^1\frac{1}{2}\Big(\Vert \dot{\bm{x}}(t)\Vert_2^2+ \Psi(\frac{\beta\dot{m}(t)}{\beta m(t)})\Big)\beta m(t)\mathrm{d}t \\
&\text{s.t.} \ \ \ \ \ \beta m(0)=\beta m_0,\beta m(1)=\beta m_1,\bm{x}(0)+\bm{\alpha}=\bm{x}_0+\bm{\alpha},\bm{x}(1)+\bm{\alpha}=\bm{x}_1+\bm{\alpha}&
\end{aligned}
\end{equation}
Thus, there exists a 1-dimensional function \(l(t)\ s.t.\ m(t)=m_0l(t)\), and \(k(t), l(t)\) only depend on \(\frac{m_1}{m_0}\) and \(\bm{x}_1-\bm{x}_0\). Also note that only the functional of \(\Vert\dot{\bm{x}}(t)\Vert\) is minimized, hence the direction is not important. Hence, \(k(t), l(t)\) only depend on \(r\) and \(d\). The proof is completed. \hfill $\square$

\subsection{Proof of Proposition \ref{prop:cost form}}
\label{pf:cost form}
\textbf{Proposition \ref{prop:cost form}.}
\textit{The cost $C_d$ is a homogeneous function of degree 1 w.r.t $(m_0, m_1)$. Thus, it takes form of $C_d(m_0,m_1)=m_0f_d(\frac{m_1}{m_0})$ where $f_d(r)=C_d(1,r)$. For even $\Psi$, $C_d$ is symmetric.}

\textit{Proof.} Recall the formula of transport cost where \(r=\frac{m_1}{m_0}\).
\begin{equation}
\label{eq:Path appendix}
\begin{aligned}
&C_d(m_0,m_1) = m_0\inf_{l,k} \int_0^1\frac{1}{2}\Big(\dot{k}(t)^2+ \Psi(\frac{\dot{l}(t)}{l(t)})\Big)l(t)\mathrm{d}t \\
&\text{s.t.} \ \ \ \ \ \ \ \ \  l(0)=1,l(1)=r,k(0)=0,k(1)=d&
\end{aligned}
\end{equation}

Replacing \(m_0,m_1\) with \(\beta m_0,\beta m_1\), we have
\begin{equation}
\label{eq:Path appendix with ratio}
\begin{aligned}
&C_d(\beta m_0,\beta m_1) = \beta m_0\inf_{l,k} \int_0^1\frac{1}{2}\Big(\dot{k}(t)^2+ \Psi(\frac{\dot{l}(t)}{l(t)})\Big)l(t)\mathrm{d}t \\
&\text{s.t.} \ \ \ \ \ \ \ \ \  l(0)=1,l(1)=r,k(0)=0,k(1)=d&
\end{aligned}
\end{equation}
Hence, \(C_d(\beta m_0,\beta m_1)=\beta C_d(m_0,m_1)\), i.e. \(C_d\) is a homogeneous function of degree 1 w.r.t $(m_0, m_1)$. Thus, \(C_d(m_0,m_1)=m_0C_d(1,r)\).

For even \(\Psi\), i.e. \(\Psi(g)=\Psi(-g)\),
\begin{equation}
\begin{aligned}
&\int_0^1\frac{1}{2}\Big(\big(\frac{\mathrm{d}}{\mathrm{d}t}k(1-t)\big)^2+ \Psi(\frac{\frac{\mathrm{d}}{\mathrm{d}t}l(1-t)}{l(1-t)})\Big)l(1-t)\mathrm{d}t\\ =&\int_0^1\frac{1}{2}\Big(\dot{k}(1-t)^2+ \Psi(\frac{-\dot{l}(1-t)}{l(1-t)})\Big)l(1-t)\mathrm{d}t\\
=&\int_0^1\frac{1}{2}\Big(\dot{k}(1-t)^2+ \Psi(\frac{\dot{l}(1-t)}{l(1-t)})\Big)l(1-t)\mathrm{d}t\\
=&\int_0^1\frac{1}{2}\Big(\dot{k}(t)^2+ \Psi(\frac{\dot{l}(t)}{l(t)})\Big)l(t)\mathrm{d}t
\end{aligned}
\end{equation}

holds for any \(k(t),l(t)\). Since
\begin{equation}
\begin{aligned}
C_d(m_0,m_1)&=m_0\inf_{l,k} \int_0^1\frac{1}{2}\Big(\dot{k}(t)^2+ \Psi(\frac{\dot{l}(t)}{l(t)})\Big)l(t)\mathrm{d}t \\
&\text{s.t.} \ \  l(0)=1,l(1)=r,k(0)=0,k(1)=d\\
&=m_0\inf_{l,k}\int_0^1\frac{1}{2}\Big(\big(\frac{\mathrm{d}}{\mathrm{d}t}k(1-t)\big)^2+ \Psi(\frac{\frac{\mathrm{d}}{\mathrm{d}t}l(1-t)}{l(1-t)})\Big)l(1-t)\mathrm{d}t\\
&\text{s.t.} \ \  l(0)=1,l(1)=r,k(0)=0,k(1)=d\\
&=m_1\inf_{l,k}\int_0^1\frac{1}{2}\Big(\big(\frac{\mathrm{d}}{\mathrm{d}t}k(1-t)\big)^2+ \Psi(\frac{\frac{1}{r}\frac{\mathrm{d}}{\mathrm{d}t}l(1-t)}{\frac{1}{r}l(1-t)})\Big)\frac{1}{r}l(1-t)\mathrm{d}t\\
&\text{s.t.} \ \  \frac{1}{r}l(0)=\frac{1}{r},\frac{1}{r}l(1)=1,k(0)=0,k(1)=d\\
\end{aligned}
\end{equation}
Rename \(\hat{l}(t)=\frac{1}{r}l(1-t),\hat{k}(t)=k(1-t)\), we have
\begin{equation}
\begin{aligned}
C_d(m_0,m_1)&=m_1\inf_{\hat{l},\hat{k}}\int_0^1\frac{1}{2}\Big(\dot{\hat{k}}(t)^2+ \Psi(\frac{\dot{\hat{l}}(t)}{\hat{l}(t)})\Big)\hat{l}(t)\mathrm{d}t\\
&\text{s.t.} \ \  \hat{l}(0)=1,\hat{l}(1)=\frac{1}{r},\hat{k}(0)=d,\hat{k}(1)=0\\
&=C_d(m_1,m_0)
\end{aligned}
\end{equation}

The proof is completed.\hfill $\square$

\subsection{Proof of Theorem \ref{thm:uniformly convex}}
\label{pf:uniformly convex}
\textbf{Theorem \ref{thm:uniformly convex}.}
\textit{For uniformly convex $\Psi$, i.e. \(\exists \kappa>0 \ s.t. 
\Psi''\ge\kappa\), $C_d$ is jointly convex and sublinear w.r.t $(m_0, m_1)$ when \(d\le\pi\sqrt{\frac{\kappa}{2}}\), i.e. $C_d(a,b)+C_d(c,d)\ge C_d(a+c,b+d)$.}

\textit{Proof.} Recall the formula of transport cost where \(r=\frac{m_1}{m_0}\).
\begin{equation}
\label{eq:Path appendix 2}
\begin{aligned}
&C_d(m_0,m_1) = \inf_{l,k} \int_0^1\frac{1}{2}\Big(\dot{k}(t)^2+ \Psi(\frac{\dot{l}(t)}{l(t)})\Big)l(t)\mathrm{d}t \\
&\text{s.t.} \ \ \ \  l(0)=m_0,l(1)=m_1,k(0)=0,k(1)=d&
\end{aligned}
\end{equation}

By Cauchy's inequality,

\begin{equation}
d^2=\Big(\int_0^1\dot{k}(t)\mathrm{d}t\Big)^2\le\int_0^1l(t)\dot{k}(t)^2\mathrm{d}t\int_0^1\frac{1}{l(t)}\mathrm{d}t
\end{equation}
Hence, when \(\dot{k}\propto\frac{1}{l}\), the equality holds.

\begin{equation}
\inf_{k}\int_0^1l(t)\dot{k}(t)^2\mathrm{d}t=d^2\Big(\int_0^1\frac{1}{l(t)}\mathrm{d}t\Big)^{-1}
\end{equation}

We then reformulate (\ref{eq:Path appendix 2}) as

\begin{equation}
\label{eq:J}
\begin{aligned}
&C_d(m_0,m_1) = \inf_{l} J[l]\ \ \ \ \text{s.t.}\ \ \  \ l(0)=m_0,l(1)=m_1
\end{aligned}
\end{equation}

where
\begin{equation}
J[l] = J_{mass}[l] + J_{dist}[l] = \underbrace{ \int_0^1 \frac{1}{2} l\Psi\left(\frac{\dot{l}}{l}\right) \mathrm{d}t }_{\text{Convex part}} + \underbrace{ \frac{d^2}{2} \left( \int_0^1 \frac{1}{l} \mathrm{d}t \right)^{-1} }_{\text{Concave part}}
\end{equation}

Let \(w(t)=\frac{\eta(t)}{l(t)}\), we calculate the second order variation of \(J_{dist}[l]\) along \(\eta(t)\). Note that \(\eta(t)\) has no boundary condition.

\begin{equation}
\begin{aligned}
\delta^2 J_{dist}[\eta] &=  \frac{d^2}{(\int_0^1 l^{-1} \mathrm{d}t)^3} \left[ \left(\int_0^1 \frac{\eta}{l^2} \mathrm{d}t \right)^2 - \left(\int_0^1 \frac{1}{l} \mathrm{d}t \right) \left(\int_0^1 \frac{\eta^2}{l^3} \mathrm{d}t \right) \right]\\
&=\frac{d^2}{(\int_0^1 l^{-1} \mathrm{d}t)^3} \left[ \left(\int_0^1 \frac{w}{l} \mathrm{d}t \right)^2 - \left(\int_0^1 \frac{1}{l} \mathrm{d}t \right) \left(\int_0^1 \frac{w^2}{l} \mathrm{d}t \right) \right]
\end{aligned}
\label{eq:secnd order variation of dist 1}
\end{equation}

Define a new measure \(\mathrm{d}\nu(t)=\frac{1}{Z}l^{-1}(t)\mathrm{d}t\) where \(Z=\int_0^1 l^{-1}(t) \mathrm{d}t\).

\begin{equation}
\begin{aligned}
\delta^2 J_{dist}[\eta] 
&=\frac{d^2}{Z^3} \left[ Z^2\left(\int_0^1 w(t) \mathrm{d}
\nu(t) \right)^2 - Z^2 \left(\int_0^1 w^2(t) \mathrm{d}\nu(t) \right) \right]\\
&=-\frac{d^2}{Z}\left[\left(\int_0^1 w^2(t) \mathrm{d}\nu(t) \right)-\left(\int_0^1 w(t) \mathrm{d}
\nu(t) \right)^2\right]
\end{aligned}
\label{eq:secnd order variation of dist 2}
\end{equation}

Next, we calculate the second order variation of \(J_{mass}[l]\) along \(\eta(t)\).

\begin{equation}
\begin{aligned}
\delta^2 J_{mass}[\eta] &= \int_0^1 \frac{l}{2}\Psi''\left(\frac{\dot{l}}{l}\right) \left[ \frac{\dot{l}^2}{l^4}\eta^2 - 2\frac{\dot{l}}{l^3}\eta\dot{\eta} + \frac{1}{l^2}\dot{\eta}^2 \right] \mathrm{d}t\\ 
&= \int_0^1 \frac{l}{2}\Psi''\left(\frac{\dot{l}}{l}\right) \left( \frac{\dot{\eta}l - \eta\dot{l}}{l^2} \right)^2 \mathrm{d}t\\
&=\int_0^1 \frac{l}{2}\Psi''\left(\frac{\dot{l}}{l}\right)\dot{w}^2\mathrm{d}t
\end{aligned}
\label{eq:secnd order variation of mass 1}
\end{equation}

Remark that \(\dot{w}(t)=\frac{\mathrm{d}w}{\mathrm{d}t}=\frac{\mathrm{d}w}{\mathrm{d}\nu}\frac{\mathrm{d}\nu}{\mathrm{d}t}=\frac{1}{Zl(t)}\frac{\mathrm{d}w}{\mathrm{d}\nu}\).

\begin{equation}
\begin{aligned}
\delta^2 J_{mass}[\eta]
&=\int_0^1 \frac{l}{2}\Psi''\left(\frac{\dot{l}}{l}\right)\frac{1}{Z^2l^2}\left(\frac{\mathrm{d}w}{\mathrm{d}\nu}\right)^2\mathrm{d}t\\
&=\frac{1}{2Z}\int_0^1\Psi''\left(\frac{\dot{l}}{l}\right)\left(\frac{\mathrm{d}w}{\mathrm{d}\nu}\right)^2\mathrm{d}\nu(t)
\end{aligned}
\label{eq:secnd order variation of mass 2}
\end{equation}

Combine (\ref{eq:secnd order variation of dist 2}) and (\ref{eq:secnd order variation of mass 2}), we have

\begin{equation}
\label{eq:second order variation}
\begin{aligned}
\delta^2J[\eta]&=\frac{1}{2Z}\int_0^1\Psi''\left(\frac{\dot{l}}{l}\right)\left(\frac{\mathrm{d}w}{\mathrm{d}\nu}\right)^2\mathrm{d}\nu(t)-\frac{d^2}{Z}\left[\left(\int_0^1 w^2(t) \mathrm{d}\nu(t) \right)-\left(\int_0^1 w(t) \mathrm{d}
\nu(t) \right)^2\right]\\
&\ge\frac{1}{Z}\left\{\frac{\kappa}{2}\int_0^1\left(\frac{\mathrm{d}w}{\mathrm{d}\nu}\right)^2\mathrm{d}\nu(t)-d^2\left[\left(\int_0^1 w^2(t) \mathrm{d}\nu(t) \right)-\left(\int_0^1 w(t) \mathrm{d}
\nu(t) \right)^2\right]\right\}
\end{aligned}
\end{equation}

By Poincaré-Wirtinger inequality on \([0,1]\) \citep{heinonen2012nonlinear}, we have

\begin{equation}
\int_0^1\left(\frac{\mathrm{d}w}{\mathrm{d}\nu}\right)^2\mathrm{d}\nu(t)\ge \pi^2 \left[ \left(\int_0^1 w^2(t) \mathrm{d}\nu(t) \right)-\left(\int_0^1 w(t) \mathrm{d}
\nu(t) \right)^2\right]
\end{equation}

Hence, when  \(d\le\pi\sqrt{\frac{\kappa}{2}}\), we have \(\delta^2J[\eta]\ge0\) for any \(\eta\), i.e. the functional \(J[l]\) is convex w.r.t \(l(t)\). Now, we can prove the jointly convexity of \(C_d(m_0,m_1)\). Let \(l_1(t),l_2(t)\) satisfy the boundary conditions

\begin{equation}
\left\{
\begin{aligned}
&l_1(0)=m_0, l_1(1)=m_1\\
&l_2(0)=n_0,l_2(1)=n_1
\end{aligned}
\right .
\end{equation}

such that \(l_1(t),l_2(t)\) achieve the minimal cost \(C_d(m_0,m_1), C_d(n_0,n_1)\). Consider the convex combination \(l_{\lambda}(t)=\lambda l_1(t)+(1-\lambda)l_2(t)\) with boundary condition
\begin{equation}
l_{\lambda}(0)=\lambda m_0+(1-\lambda)n_0,l_{\lambda}(1)=\lambda m_1+(1-\lambda)n_1
\end{equation}

By the convexity of \(J[l]\) and the definition of \(C_d\), we have

\begin{equation}
\begin{aligned}
C_d(\lambda m_0+(1-\lambda)n_0,\lambda m_1+(1-\lambda)n_1)
&\le J[l_{\lambda}]\le\lambda J[l_1]+(1-\lambda)J[l_2]\\&=\lambda C_d(m_0,m_1)+(1-\lambda)C_d(n_0,n_1)  
\end{aligned}
\end{equation}

This is the jointly convexity of \(C_d\). Recall that \(C_d\) is also a homogeneous function of degree 1 w.r.t $(m_0, m_1)$. The sublinearity follows natrually.

\begin{equation}
\begin{aligned}
C_d(m_0+n_0,m_1+n_1)\le\frac{1}{2}C_d(2m_0,2m_1)+\frac{1}{2}C_d(2n_0,2n_1)=C_d(m_0,m_1)+C_d(n_0,n_1)  
\end{aligned}
\end{equation}

The proof is completed.\hfill $\square$

\subsection{Proof of Theorem \ref{thm:convexity}}
\label{pf:convexity}
\textbf{Theorem \ref{thm:convexity}.}
\textit{For uniformly convex $\Psi$ with $\Psi''\ge \kappa>0$, and \(\text{diam}(\mathcal{X})\le\pi\sqrt{\frac{\kappa}{2}}\), (\ref{eq:static UDOT}) is a convex optimization w.r.t $(\gamma_0,\gamma_1)$.}

\textit{Proof.} As proved in theorem \ref{thm:uniformly convex}, \(\forall\bm{x},\bm{y}\in\mathcal{X}\), since \(\Vert\bm{x}-\bm{y}\Vert\le\text{diam}(\mathcal{X})\le\pi\sqrt{\frac{\kappa}{2}}\), \(C_{\Vert\bm{x}-\bm{y}\Vert}(m_0,m_1)\) is jointly convex w.r.t \((m_0,m_1)\). Hence,
\begin{equation}
\mathcal{J}(\gamma_0, \gamma_1) = \int_{\mathcal{X}^2} C_{\|\bm{x}-\bm{y}\|}(\gamma_0(\bm{x},\bm{y}), \gamma_1(\bm{x}, \bm{y})) \mathrm{d}\bm{x}\mathrm{d}\bm{y}   
\end{equation}
is a jointly convex functional w.r.t \((\gamma_0,\gamma_1)\). Since the constraint of semi-coupling (\ref{eq:semi-coupling constraint}) is linear and hence convex,  (\ref{eq:static UDOT}) is a convex optimization w.r.t \((\gamma_0,\gamma_1)\). The proof is completed.\hfill $\square$

\section{Implementation details}
\label{appendix:implementation}
\subsection{Computational resources}
The experiments were performed on a shared high-performance computing cluster with NVIDIA A800 GPU and 12 CPU cores. All neural networks were built using PyTorch and trained by Adam optimizer \citep{adam}.

\subsection{Network architectures}
\textbf{Path model.} The path model \((\phi_{\eta},\psi_\eta)\) is realized by a single neural network with input dim 3 \((t,d,r)\), output dim 2 \((\phi,\psi)\), and 5 hidden layers with dim 256. For smoothness, we use SiLU \citep{SiLU} activation. 

\textbf{Cost model.} The cost model \(\mathcal{E}_\xi\) is realized by a neural network with input dim 2 \((d,r)\), output dim 1 \((\mathcal{E})\), and 5 hidden layers with dim 256. We also use SiLU activation. 

\textbf{Flow model.} We parameterized the velocity net \(\bm{u}_{\bm{\theta}}\), the growth net \(g_{\bm{\theta}}\), and if stochastic version is used, the score net \(\bm{s}_{\bm{\theta}}\) by three neural networks. These networks have 5 hidden layers with dim 256, and LeakyReLU activations, which is consistent to the previous simulation-free methods \citep{wfr_fm}.

\subsection{Grid}
In practice, we compute the path loss (\ref{eq:path loss}) and the cost loss (\ref{eq:cost loss}) on a grid \(D\times R\). In all experiments, we set \(R\) as the uniform logarithmic scale grid on \([0.01, 10]\), and set \(D\) as the uniform grid on \([0.1,d_{max}]\), where \(d_{max}\) is the largest distance between the data points in consecutive time points. The grid size of \(D\times R\) is \(64\times64\).

\subsection{Mini-batch OT coupling}
Due to the time and memory constraints of computing large-scale OT, we utilize the mini-batch OT strategy \citep{fatras2021minibatchOT} for large dataset (EB and Mouse), consistent with prior works \citep{cfm_tong,tong2023simulationfree,wfr_fm}. We follow their implementation and replace the Sinkhorn solver by projected gradient descent solver.

\subsection{Metrics}
\label{app:metrics}
\textbf{Measure matching.} Given two measures \(p,q\) with total mass \(m_0,m_1\), let \(\tilde{p}=p/m_0\), \(\tilde{q}/m_1\). Their difference can be measured in two aspects: the difference between their shape \(\tilde{p},\tilde{q}\), and the difference between their total mass \(m_0,m_1\).

We use Wasserstein-1 distance \(\mathcal{W}_1\) to measure the difference between the shapes.
\begin{equation}
\label{eq:W1}
\mathcal{W}_1(\tilde{p}, \tilde{q}) = \min_{\pi \in \Pi(\tilde{p}, \tilde{q})} \int \|\bm{x} - \bm{y}\|_2 d\pi(\bm{x}, \bm{y})
\end{equation}
where 
\[\Pi(\tilde{p}, \tilde{q})=\{\pi(\bm{x}, \bm{y})\in\mathcal{M}_+(\mathcal{X}^2)|\int_{\mathcal{X}}\pi(\bm{x}, \bm{y})\mathrm{d}\bm{y}=\tilde{p}(\bm{x}),\int_{\mathcal{X}}\pi(\bm{x}, \bm{y})\mathrm{d}\bm{x}=\tilde{q}(\bm{y})\}\]
In practice, \(\tilde{p},\tilde{q}\) are represented by weighted particles. The \(\mathcal{W}_1\) can be efficiently obtained using POT package \citep{pot}.

We use relative mass error (RME) to measure the difference between the total masses.
\begin{equation}
\label{eq:RME}
\text{RME}(p,q)=\frac{|m_p-m_q|}{m_q}
\end{equation}
In practice, \(p\) is the predicted measure, \(q\) is the true measure.

\textbf{Cell fate dominant accuracy.} The Mouse hematopoiesis data \citep{weinreb2020lineage} used in the main text has two terminal type Neu and Mo at Day 6. Dominant accuracy is defined as follows. If the descendants of a barcode contain more Neu cells than Mo at Day 6, we define the dominant fate of this barcode as Neu. Similarly, if the trajectory inferred by the model leads to a higher proportion of Neu descendants for a given barcode, we define the model-predicted dominant fate as Neu. The dominant accuracy measures the accuracy of the model's predicted dominant fates. 

\subsection{UDOT cost}
Given \(\bm{u}_{\bm{\theta}}\) and \(g_{\bm{\theta}}\), we calculate the induced UDOT cost (\ref{eq:UDOT}) by numerical integral. In detail, we initialize the system by assigning a mass of $1/N$ to each of the $N$ cells at $t=0$, thereby constructing an initial measure with a unit total mass. Starting from these initial states, we generate $N$ trajectories using the forward Euler method with a time step of 0.01. By performing numerical integration of the UDOT cost along these trajectories, we derive an estimate of the overall UDOT cost for the measure path.

\subsection{Static WFR cost}
On a dataset with \(N+1\) time points \(t_0,t_1,\cdots,t_N\), and \(\delta\), we first compute the semi-coupling pairs \(\{(\gamma_0^{(k)},\gamma_1^{(k)})\}_{k=0:N-1}\). \((\gamma_0^{(k)},\gamma_1^{(k)})\) stands for the semi-coupling between the time point \(t_k\) and \(t_{k+1}\). The static WFR cost is calculated using the closed form \citep{chizat2018interpolating}.

\begin{equation}
\label{eq:static WFR cost}
\begin{aligned}
&\text{WFR}^2_{\delta}(\mu_{t_0},\mu_{t_1},\cdots,\mu_{t_N})\\
&=2\delta^2\sum_{k=0}^{N-1}\sum_{\bm{x},\bm{y}}\left(\gamma_0^{(k)}(\bm{x},\bm{y})+\gamma_1^{(k)}(\bm{x},\bm{y})-2\sqrt{\gamma_0^{(k)}(\bm{x},\bm{y})\gamma_1^{(k)}(\bm{x},\bm{y})}\cos(\min\{\frac{\Vert\bm{x}-\bm{y}\Vert}{2\delta},\frac{\pi}{2}\})\right)   
\end{aligned}
\end{equation}

For consistency, we also scale the total mass of all measures to make \(\mu_0\) a unit probability measure.

\subsection{A Stochastic Version for RUOT via Unbalanced Score Matching}
\label{appendix:USM}
Although SUDO can only be applied to UDOT problem (\ref{eq:UDOT}) whose constraint is a continuity equation with source term, one can replace the conditional unbalanced flow matching loss with a conditional unbalanced score matching loss with noise level $\sigma$ \citep{USB}.

\begin{equation}
\begin{aligned}
\label{eq:SUDO CUSM}
\mathcal{L}_{\text{CUSM}}(\bm{\theta})=
&\mathbb{E}_{t \sim \mathcal{U}[0,1], (\bm{x}_0,\bm{x}_1)\sim\gamma_0^\star, }m_{\eta,t}\big(\left\| \bm{\bm{u}_{\theta}}(\bm{x},t) - \big(\frac{1-2t}{t(1-t)}(\bm{x}-\bm{x}_{\eta,t})+\dot{\bm{x}}_{\eta,t}\big) \right\|_2^2\\
&+\left\| g_{\bm{\theta}}(\bm{x},t)- \frac{\dot{m_{\eta,t}}}{m_{\eta,t}} \right\|_2^2+\lambda^2(t)\left\| \bm{s}_{\bm{\theta}}(\bm{x},t)+\frac{\bm{x} - \bm{x}_{\eta,t}}{\sigma^2t(1-t)} \right\|_2^2\big)
\end{aligned}
\end{equation}

where \(\bm{s}_{\bm{\theta}}(\bm{x},t)\) is a neural network approximating the score function \(\bm{s}(\bm{x},t)=\nabla_{\bm{x}}\operatorname{ln}\rho_t(\bm{x})\). \(\lambda(t)=\sigma\sqrt{t(1-t)}\) is a weight scheduler for numerical stability adapted from \citep{tong2023simulationfree}, and \(\bm{x}\sim\mathcal{N}(\bm{x}_{\eta,t},\sigma^2t(1-t)\mathrm{I})\). The triplet \((\bm{u}_{\bm{\theta}},g_{\bm{\theta}},\bm{s}_{\bm{\theta}})\) can be viewd as an approximation to RUOT (\ref{eq:dynamic RUOT}).

\subsection{Inference}
\label{appendix:inference}
The output of SUDO is \((\bm{u}_{\bm{\theta}},g_{\bm{\theta}})\), or \((\bm{u}_{\bm{\theta}},g_{\bm{\theta}},\bm{s}_{\bm{\theta}})\) for stochastic version. To recover the continuous measure flow of UDOT (\ref{eq:UDOT}), an ODE simulator is applied for inference.
\begin{equation}
\label{eq:ODE inference}
\left \{
\begin{aligned}
&\mathrm{d}\bm{x}_t = \bm{u}_{\theta}(\bm{x}_t,t)\mathrm{d}t\\
&\mathrm{d}\operatorname{ln}m_t=g_{\theta}(\bm{x}_t,t)\mathrm{d}t
\end{aligned}
\right.
\end{equation}
For stochastic version, a SDE simulator is applied.
\begin{equation}
\label{eq:SDE inference}
\left \{
\begin{aligned}
&\mathrm{d}\bm{x}_t = (\bm{u}_{\theta}(\bm{x}_t,t)+
\frac{\sigma^2}{2}\bm{s}_{\bm{\theta}}(\bm{x}_t,t))\mathrm{d}t+\sigma\mathrm{d}\bm{W}_t\\
&\mathrm{d}\operatorname{ln}m_t=g_{\theta}(\bm{x}_t,t)\mathrm{d}t
\end{aligned}
\right .
\end{equation}

\subsection{VarRUOT training time}
\label{appendix:varruot training time}
In the publicly available version of VarRUOT \citep{sun2026variational}, the algorithm samples a fixed number of cells (2048) for NeuralSDE training regardless of the dataset size, which explains why its training time does not increase as the dataset grows (Table \ref{tab:scalability}). However, this fixed-sampling strategy leads to progressively worse performance on larger datasets. Conversely, if trained to all cells, VarRUOT will need more time on large data (e.g., training time > 11 hours for 49302 cells).

\subsection{Barcode-based ground truth}
\label{app:barcode}
The Mouse dataset \citep{weinreb2020lineage} contains three time points: Day 2, Day 4, and Day 6. At Day 6, the cells are divided into two terminal cell types, Neu and Mo. Following \citep{TIGON}, we used barcode information provided in the data to estimate the cell fate of cells at the initial time point, Day 2. Specifically, the cell fate of a Day 2 cell was defined as the proportion of Day 6 cells carrying the same barcode that belonged to the Neu cell type. We also followed \citep{TIGON} to compute barcode-based growth rates for cells at Day 2 and Day 4. Specifically, suppose that a given barcode is associated with $n_0$ cells at time $t_0$ and $n_1$ cells at time $t_1$. The growth rate assigned to cells carrying this barcode at time $t_0$ is defined as $g=\frac{\log(n_1/n_0)}{t_1-t_0}.$ 

\newpage
\section{Additional results}
\label{appendix:additional results}
\subsection{The Path Model Approximates the Travelling Dirac}
\label{appendix:path}
To evaluate whether the path model can accurately approximate the true travelling Dirac path (see \ref{appendix:WFR-FM} for detailed closed form), we train a path model \((\phi_{\eta},\psi_{\eta})\) on a WFR problem. We choose \(\delta=4\). The model was trained on the range \(d\in[0.1,12],\ r\in[0.01,10]\). We plot \(l(t)\) in Figure \ref{fig:mass path} for \(d \in \{0.01,0.1,1,2,6,12,15\},\ r\in\{0.001,0.01,0.1,0.5,1,2,5,10,20\}\). The true path is blue, and the learned path is red. Note that \(d=0.01,15\) (The first and last column) and \(r=0.001,20\) (the first and last row) are unseen conditions. In the seen scenarios, the learned paths closely align with the ground-truth paths, demonstrating that our path model has indeed captured the underlying dynamics of mass variation. Although the model doesn't generalize perfectly to larger unseen values of $d$, it exhibits a certain degree of generalization capability with respect to $r$ and small $d$. However, the practical application of our path model does not require it to generalize to larger values of $d$.

\begin{figure}[h!]
  \centering
  \includegraphics[width=\textwidth]{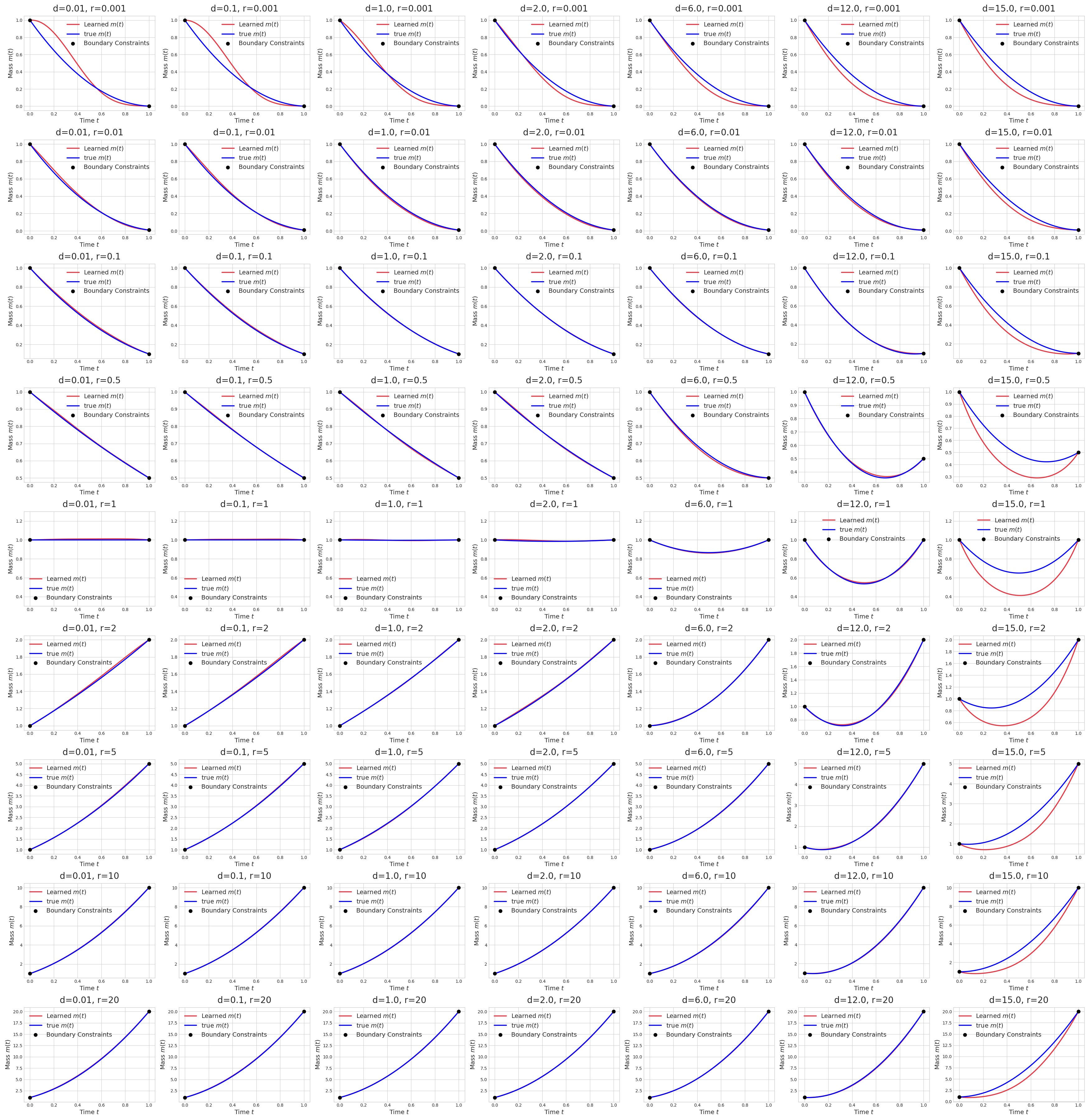}
  
  \caption{Red curve: mass variation learned by path model; Blue curve: True mass variation. Columns: \(d \in \{0.01,0.1,1,2,6,12,15\}\); Rows: \(r\in\{0.001,0.01,0.1,0.5,1,2,5,10,20\}\).}
  \label{fig:mass path}
\end{figure}

We plot \(k(t)\) in Figure \ref{fig:x path} for \(d \in \{0.01,0.1,1,2,6,12\},\ r\in\{0.001,0.01,0.1,0.5,1,2,5,10,20\}\). \(d=15\) is not plotted since \(15>4\pi\), the particles never travel, instead, they undergo pure birth-death dynamics \citep{chizat2018interpolating}. 

\begin{figure}[h!]
  \centering
  \includegraphics[width=\textwidth]{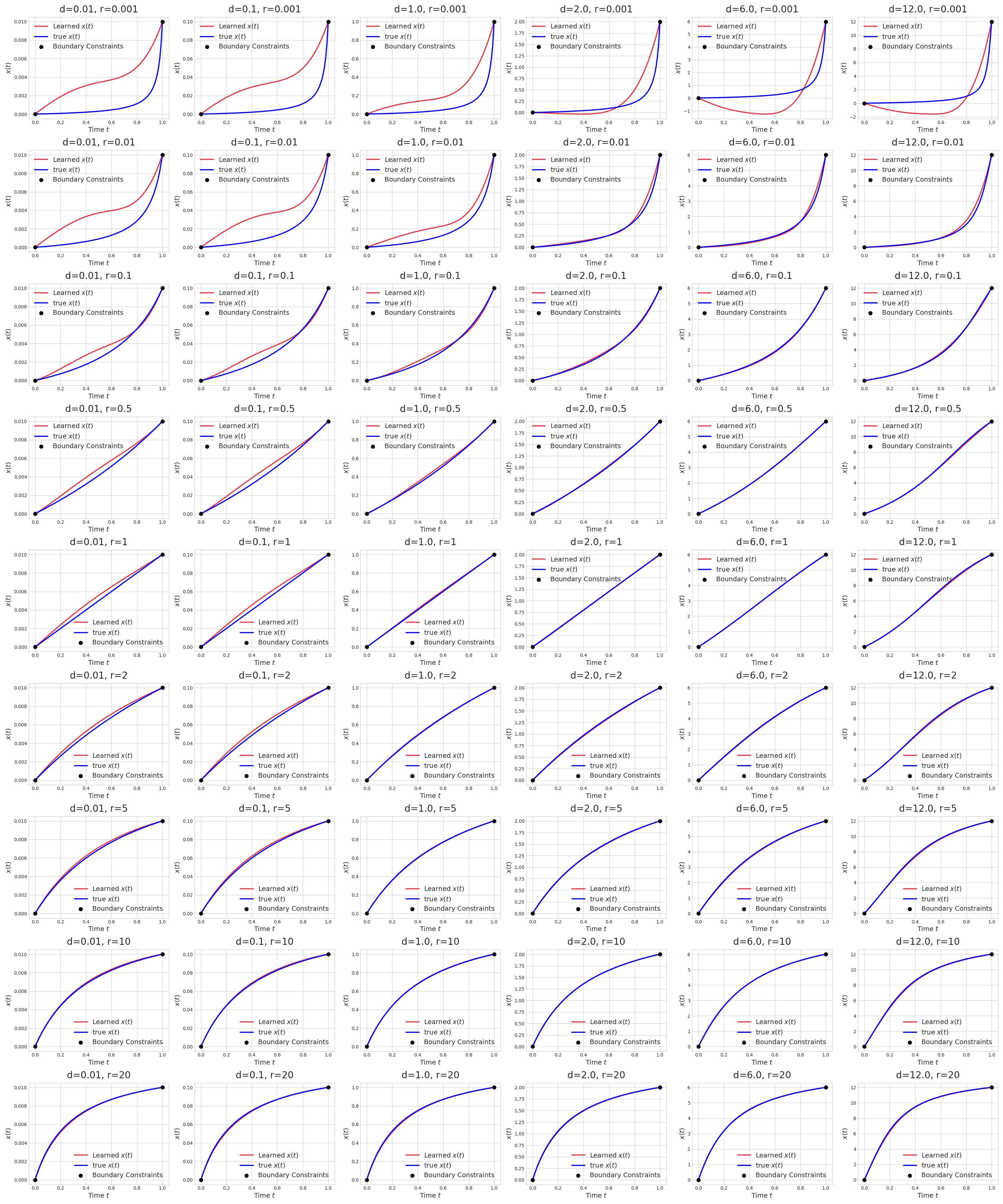}
  
  \caption{Red curve: position variation learned by path model; Blue curve: True mass variation. Columns: \(d \in \{0.01,0.1,1,2,6,12\}\); Rows: \(r\in\{0.001,0.01,0.1,0.5,1,2,5,10,20\}\).}
  \label{fig:x path}
\end{figure}

In most cases, the learned paths also closely align with the ground-truth. The paths that are not learned well are those where $r$ is very small. They correspond to relatively small costs and $m(t)$. Therefore, their weights in the subsequent cost fitting and unbalanced flow matching are small, ensuring that they won't overly influence the subsequent procedures.

We trained the model for 3000 epochs. The training takes about 14 seconds.

\subsection{The Cost Model Approximates the Transport Cost}
\label{appendix:cost}
We trained the path model and the cost model \(\mathcal{E}_\xi\) for different \(\delta\) and grid size, and calculated the relative \(L^2\) norm error (\(E_2\)) between \(\mathcal{E}_\xi(d,r)\) and the true transport cost \(C_d(1,r)\) (The closed form of \(C_d\) is introduced in \ref{appendix:WFR-FM}) on \([0,12]\times[0.01,10]\), a sufficiently large domain (Table \ref{tab:cost}). Saying the grid size is \(N\) means that we discretize the interval $[0, 12]$ into $N$ equidistant points to form $D$, and the interval $[0.01, 10]$ into $N$ logarithmically equidistant points to form $R$, and trained the cost model on grid \(D\times R\).

\begin{table}[ht]
\centering
\caption{Relative \(L^2\) norm error (\(E_2\)) between the \(C_d(1,r)\) and \(\mathcal{E}_\xi(d,r)\) under different training settings.}
\label{tab:cost}
\begin{tabular}{lccc}
\toprule
\multirow{2}{*}{\textbf{$\delta$}} & \multicolumn{3}{c}{\textbf{$E_2$ (\%) $\downarrow$}} \\
\cmidrule(lr){2-4}
 & $N=64$ & $N=32$ & $N=16$ \\
\midrule
0.5 & 8.36 & 8.29 & 8.64 \\
1   & 7.18 & 7.16 & 7.25 \\
2   & 7.63 & 7.57 & 7.62 \\
3   & 6.71 & 6.65 & 6.69 \\
5   & 4.20 & 3.91 & 4.01 \\
10  & 3.05 & 3.26 & 3.12 \\
15  & 2.94 & 2.61 & 2.69 \\
30  & 2.95 & 2.93 & 2.95 \\
50  & 3.38 & 3.45 & 3.93 \\
\bottomrule
\end{tabular}
\end{table}

We trained the cost model for 3000 epochs. The average training time for each run is about 20 seconds (Without training the path model). For experiments in the main text, we use the grid size \(N=64\).

\subsection{Simulation Gene Data}
\label{appendix:simulation gene data}
We adopted the Simulation Gene data from \citep{DeepRUOT}. It simulates a gene regulatory network consists of 3 genes (\(X_1,X_2,X_3\)).

$$\frac{dX_1}{dt} = \frac{\alpha_1 X_1^2 + \beta}{1 + \alpha_1 X_1^2 + \gamma_2 X_2^2 + \gamma_3 X_3^2 + \beta} - \delta_1 X_1 + \eta_1 \xi_t$$$$\frac{dX_2}{dt} = \frac{\alpha_2 X_2^2 + \beta}{1 + \gamma_1 X_1^2 + \alpha_2 X_2^2 + \gamma_3 X_3^2 + \beta} - \delta_2 X_2 + \eta_2 \xi_t$$$$\frac{dX_3}{dt} = \frac{\alpha_3 X_3^2}{1 + \alpha_3 X_3^2} - \delta_3 X_3 + \eta_3 \xi_t$$

$X_1$ and $X_2$ undergo mutual inhibition and auto-activation, with both being repressed by $X_3$ and stimulated by an external signal $\beta$. Here, $\xi_i$ represents the stochastic noise within the system. The growth is introduced by probabilistic cell division with probability $g=\frac{X^2_2}{2(1+X^2_2)}\%$. Simulations were initialized at two different locations. The bottom-left represents a stable attractor, whereas cells in the bottom-right migrate toward the top-left while proliferating. The resulting dataset contains five temporal snapshots of the system, specifically at $t \in \{0, 8, 16, 24, 32\}$.

\textbf{Ablation on \(\delta\).} When solving the WFR problem, we conducted an ablation study on the penalty parameter $\delta$. As shown in Table \ref{tab:delta}, within a reasonable range of $\delta$, SUDO can robustly match the measures across different time points. We use \(\delta=1.2\) in the main text.

\begin{table}[h!]
\centering
\caption{Sensitivity analysis for parameter $\delta$ on simulation gene dataset.}
\label{tab:delta}
\resizebox{\textwidth}{!}{
\begin{tabular}{lcccccccc}
\toprule
\textbf{Parameter} & \multicolumn{2}{c}{\textbf{t=1}} & \multicolumn{2}{c}{\textbf{t=2}} & \multicolumn{2}{c}{\textbf{t=3}} & \multicolumn{2}{c}{\textbf{t=4}} \\
\cmidrule(lr){2-3} \cmidrule(lr){4-5} \cmidrule(lr){6-7} \cmidrule(lr){8-9}
& $\mathcal{W}_1$ & RME & $\mathcal{W}_1$ & RME & $\mathcal{W}_1$ & RME & $\mathcal{W}_1$ & RME \\
\midrule
$\delta=0.5$ & 0.036 & 0.001 & 0.037 & 0.010 & 0.044 & 0.005 & 0.045 & 0.015 \\
$\delta=1.0$ & 0.023 & 0.002 & 0.027 & 0.006 & 0.022 & 0.007 & 0.022 & 0.011 \\
$\delta=1.2$ & 0.024 & 0.008 & 0.025 & 0.001 & 0.020 & 0.002 & 0.019 & 0.005 \\
$\delta=1.5$ & 0.024 & 0.002 & 0.031 & 0.009 & 0.031 & 0.011 & 0.029 & 0.002 \\
$\delta=2.0$ & 0.032 & 0.004 & 0.040 & 0.008 & 0.031 & 0.014 & 0.044 & 0.024 \\
$\delta=2.5$ & 0.030 & 0.005 & 0.064 & 0.004 & 0.059 & 0.004 & 0.079 & 0.021 \\
\bottomrule
\end{tabular}
}
\end{table}

\textbf{Ablation on \(\Psi\).} We also tried different growth penalties \(\Psi\) on simulation gene data. We trained SUDO to solve to UDOT problem with 4 different \(\Psi\). 

\begin{equation}
\Psi(g) = \left \{
\begin{aligned}
&\delta^2g^2,\ &\text{WFR}\\
&1-g+g\operatorname{log}g\ (g>0),\ &\text{Only-growth}\\
&1+g-g\operatorname{log}(-g)\ (g<0),\ &\text{Only-death}\\
&1-\sqrt{1+g^2}+g\operatorname{log}(g+\sqrt{1+g^2}),\ &\text{No preference}\\
\end{aligned}
\right .
\end{equation}

The \(\mathcal{W}_1\) and RME results are shown in Table \ref{tab:Psi}. SUDO matches the marginal measures under all \(\Psi\) settings. This suggests that the choice of $\Psi$ does not affect SUDO's ability to learn a measure path connecting the time points; rather, it merely influences how SUDO connects them.

\begin{table}[h!]
\centering
\caption{Sensitivity analysis for penalty $\Psi$ on simulation gene dataset.}
\label{tab:Psi}
\resizebox{\textwidth}{!}{
\begin{tabular}{lcccccccc}
\toprule
\textbf{Parameter} & \multicolumn{2}{c}{\textbf{t=1}} & \multicolumn{2}{c}{\textbf{t=2}} & \multicolumn{2}{c}{\textbf{t=3}} & \multicolumn{2}{c}{\textbf{t=4}} \\
\cmidrule(lr){2-3} \cmidrule(lr){4-5} \cmidrule(lr){6-7} \cmidrule(lr){8-9}
& $\mathcal{W}_1$ & RME & $\mathcal{W}_1$ & RME & $\mathcal{W}_1$ & RME & $\mathcal{W}_1$ & RME \\
\midrule
WFR & 0.024 & 0.008 & 0.025 & 0.001 & 0.020 & 0.002 & 0.019 & 0.005 \\
Only-growth & 0.027 & 0.004 & 0.031 & 0.009 & 0.022 & 0.011 & 0.019 & 0.014 \\
Only-death & 0.046 & 0.001 & 0.057 & 0.001 & 0.052 & 0.006 & 0.029 & 0.009 \\
No preference & 0.028 & 0.009 & 0.032 & 0.020 & 0.027 & 0.021 & 0.027 & 0.029 \\
\bottomrule
\end{tabular}
}
\end{table}

\subsection{Dyngen Data}
\label{appendix:dyngen data}
Following the experimental setup in \citep{mioflow,wang2025joint}, we utilize a synthetic dataset generated by the Dyngen framework \citep{cannoodt2021spearheading}. The dataset consists of 728 cells, with their dimensionality reduced to five via PHATE \citep{moon2019visualizing}. This data captures intricate biological dynamics, characterized by temporal mass fluctuations and a distinct bifurcation. The bifurcation is inherently unbalanced, as the cell density in the lower branch significantly exceeds that of the upper branch. In the main text, we use \(\delta=1.7\) for WFR problem. The detailed \(\mathcal{W}_1\) and RME results are shown in Table \ref{tab:Dyngen}. SUDO ranks second only to WFR-FM, which utilizes the closed-form traveling Dirac solution. It outperforms existing simulation-based algorithms in terms of both measure matching accuracy and computational efficiency.

\begin{table}[h!]
\centering
\caption{Comparison of method performance over time on the Dyngen dataset. Best results are in bold, and the second best are underlined.}
\label{tab:Dyngen}
\resizebox{\textwidth}{!}{
\begin{tabular}{lcccccccc}
\toprule
\textbf{Method} & \multicolumn{2}{c}{\textbf{t=1}} & \multicolumn{2}{c}{\textbf{t=2}} & \multicolumn{2}{c}{\textbf{t=3}} & \multicolumn{2}{c}{\textbf{t=4}} \\
\cmidrule(lr){2-3} \cmidrule(lr){4-5} \cmidrule(lr){6-7} \cmidrule(lr){8-9}
& $\mathcal{W}_1$ & RME & $\mathcal{W}_1$ & RME & $\mathcal{W}_1$ & RME & $\mathcal{W}_1$ & RME \\
\midrule
TIGON & 0.446 & 0.033 & 0.584 & \underline{0.060} & 0.415 & 0.023 & 0.603 & 0.071 \\
DeepRUOT & 0.454 & \underline{0.011} & 0.481 & 0.070 & 0.870 & 0.104 & 0.688 & 0.074 \\
Var-RUOT & 0.315 & 0.128 & 0.548 & 0.336 & 0.630 & 0.222 & 0.593 & 0.023 \\
WFR-FM & \textbf{0.110} & \textbf{0.003} & \textbf{0.098} & \textbf{0.007} & \textbf{0.211} & \underline{0.008} & \textbf{0.121} & \textbf{0.002} \\
\textbf{SUDO} & \underline{0.140} & \textbf{0.003} & \underline{0.168} & \textbf{0.007} & \underline{0.280} & \textbf{0.005} & \underline{0.170} & \underline{0.006} \\
\bottomrule
\end{tabular}
}
\end{table}

\subsection{Gaussian Data}
\label{appendix:gaussian data}
We utilize the 1000D Gaussian Mixture Model (GMM) dataset as described in \citep{wfr_fm}. In accordance with their experimental setup, the initial population consists of 500 cells distributed across two Gaussian components: 100 in the upper component and 400 in the lower one. The upper component remains stationary while undergoing significant proliferation, whereas the lower component undergoes no growth, but symmetrically bifurcates into two clusters. By the terminal time point, the population reaches 1,400 cells, with the upper component expanding to 1,000 cells and the lower component splitting into two 200-cell Gaussians. The high dimensionality of this dataset provides a benchmark for evaluating SUDO's performance in high dimensional space. In the main text, we use $\delta=1.4$.

\textbf{SUDO learns a more plausible growth rate.} A main limitation of WFR-FM \citep{wfr_fm} is that its growth rate must follow the WFR geometry, where mass varies according to a quadratic function of time. When connecting two measures of equal total mass using a WFR flow, the resulting path tends to decrease and then increase the mass to minimize kinetic energy consumption, thereby achieving a lower total cost. However, in our Gaussian dataset construction, the lower cluster undergoes pure transport without any growth or death. To prevent the model from yielding a mass-reducing solution typical of WFR and to ensure the upper component exhibits pure growth while the lower one remains pure transport, we employ an only-growth $\Psi$ (\ref{eq:only growth Psi}). By assigning a infinite cost to apoptosis, we explicitly penalize and thus avoid any mass decreasing. In practice, we relax the infinite cost on \((-\infty,0]\) to a quadratic function \(10000g^2+1\).

Table \ref{tab:gaussian Psi} demonstrates that SUDO matches the target measure under both growth penalty setting.

\begin{table}[h!]
\centering
\caption{The measure matching results on gaussian 1000d data with different growth penalty}
\label{tab:gaussian Psi}

\begin{tabular}{cccc}
\toprule 
\multicolumn{2}{c}{WFR} & \multicolumn{2}{c}{Only-growth} \\
\cmidrule(lr){1-2} \cmidrule(lr){3-4}
$\mathcal{W}_1$ & RME & $\mathcal{W}_1$ & RME\\
\midrule
2.315 & 0.061 & 2.333 & 0.016\\
\bottomrule
\end{tabular}
\end{table}

As shown in Figure \ref{fig:gaussian Psi}, SUDO learns a more plausible growth rate than WFR-FM where the lower cluster undergoes pure transport with zero growth rate.

\begin{figure}[h!]
  \centering
  \includegraphics[width=0.45\textwidth]{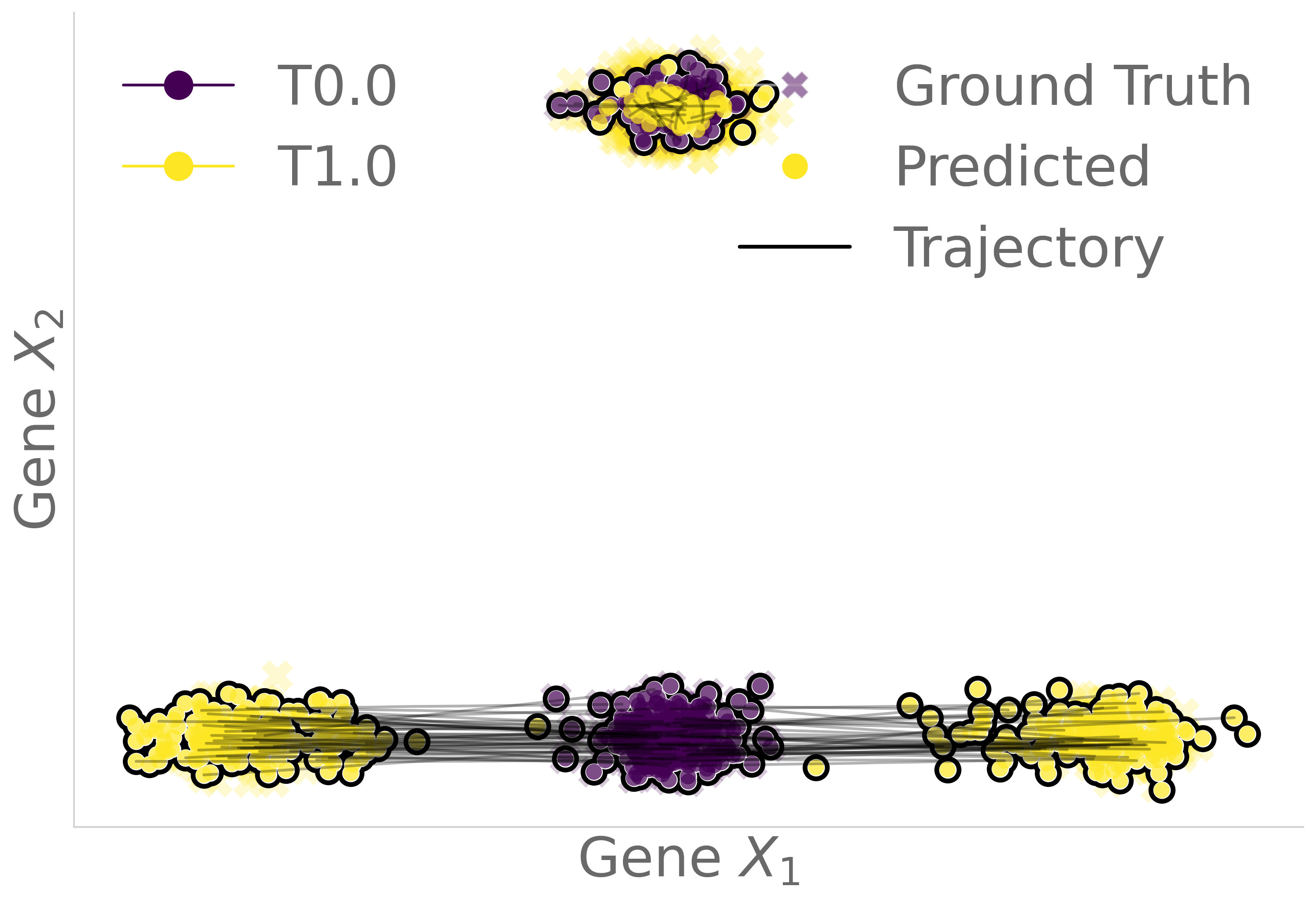} 
  \includegraphics[width=0.45\textwidth]{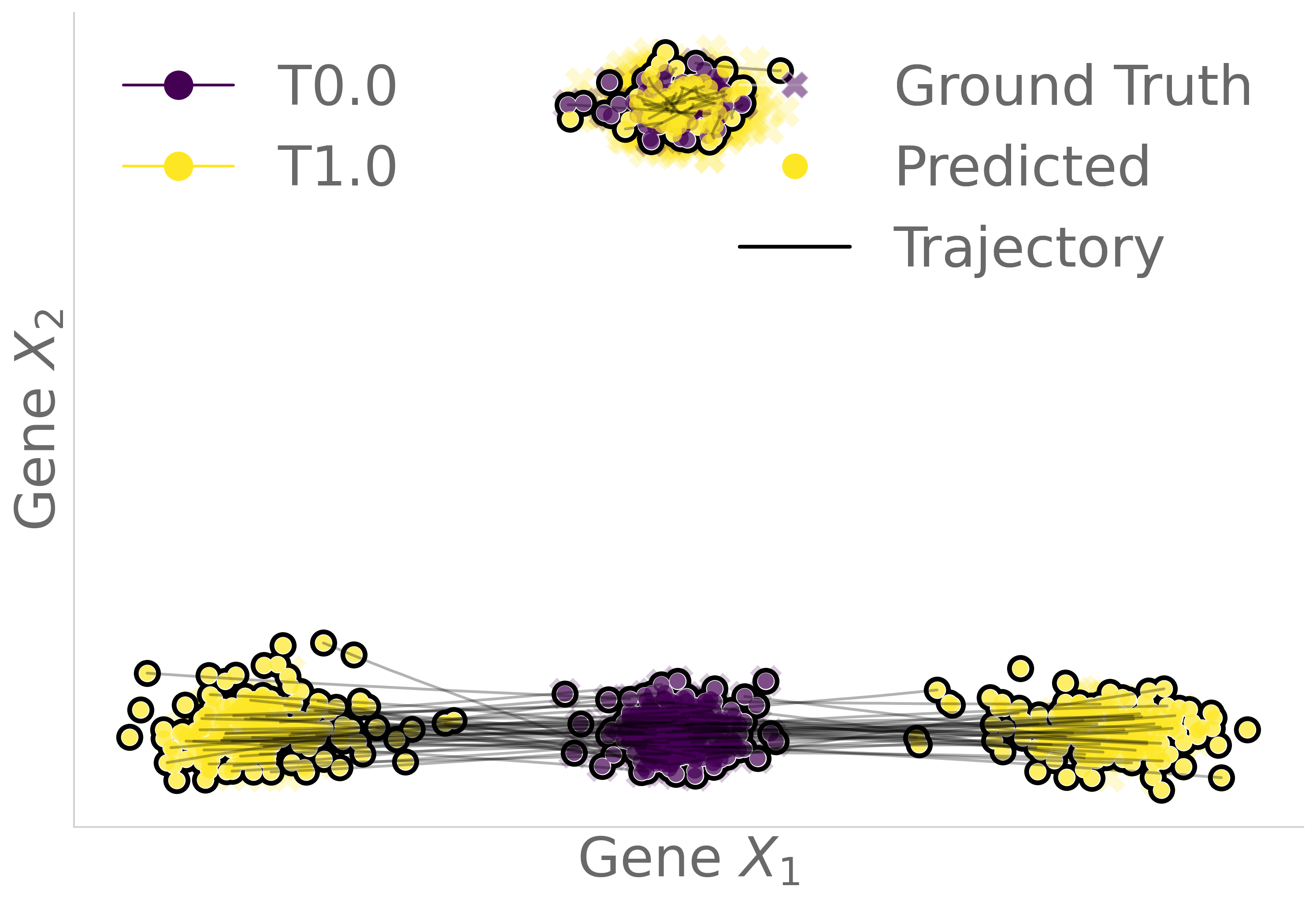} 
  \includegraphics[width=0.45\textwidth]{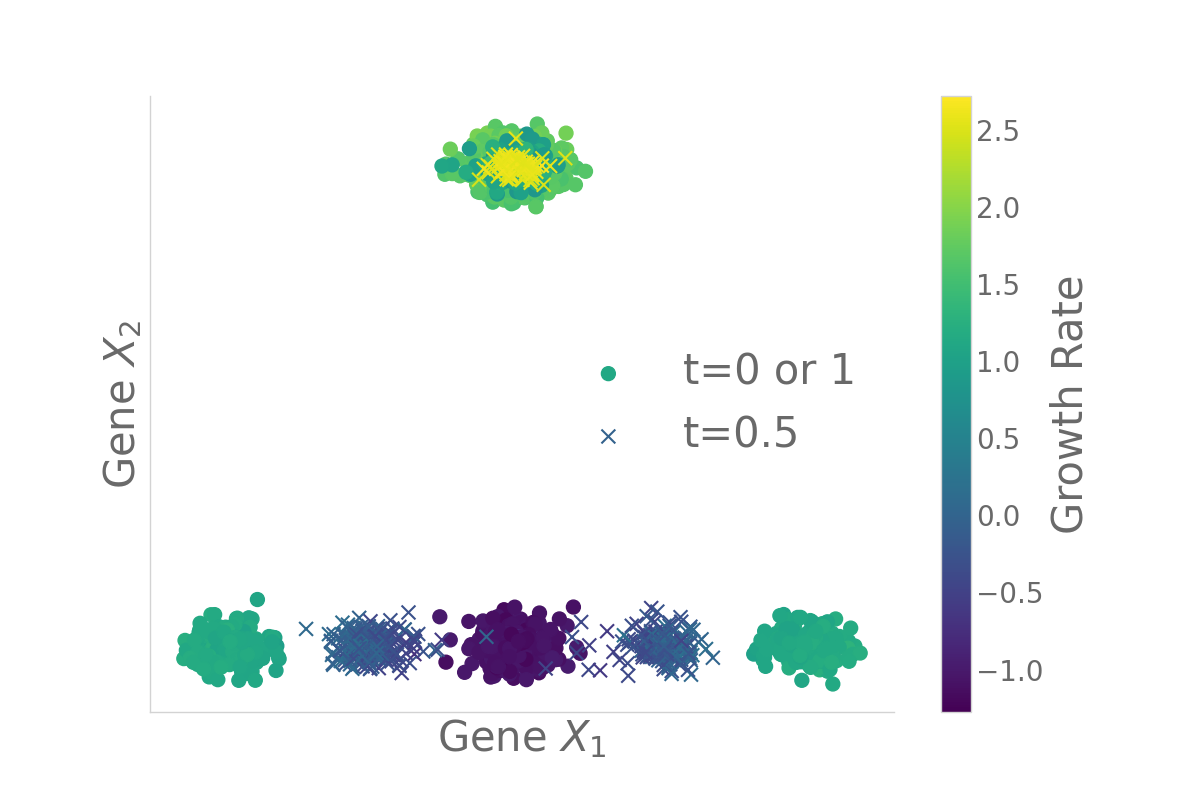} 
  \includegraphics[width=0.45\textwidth]{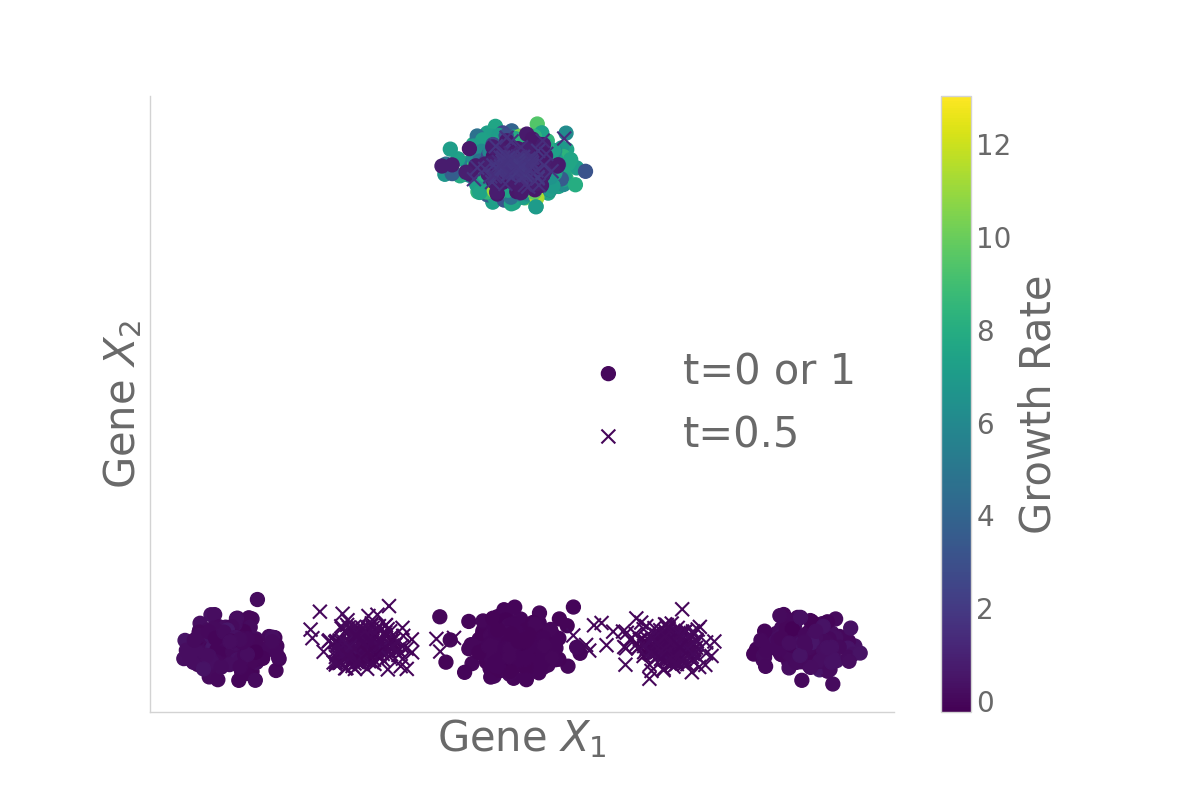} 
  
  \caption{Left: predicted trajectories and growth rate by WFR-FM; Right: predicted trajectories and growth rate by SUDO.}
  \label{fig:gaussian Psi}
\end{figure}

\subsection{Epithelial Mesenchymal Transition Data}
\label{appendix:emt data}
We utilize the single-cell dataset from \citep{cook2020context}, which characterizes the epithelial-mesenchymal transition (EMT) in A549 lung cancer cells across four time points. Following the preprocessing pipeline established by \citep{TIGON}, the high-dimensional gene expression profiles were projected into a 10-dimensional latent space using an autoencoder. Setting \(\delta=2\), the performance of SUDO is shown in Table \ref{tab:emt}.

\begin{table}[h!]
\centering
\caption{Comparison of method performance over time on the 10D EMT dataset.}
\label{tab:emt}
\begin{tabular}{lcccccc}
\toprule
\textbf{Method} & \multicolumn{2}{c}{\textbf{t=1}} & \multicolumn{2}{c}{\textbf{t=2}} & \multicolumn{2}{c}{\textbf{t=3}} \\
\cmidrule(lr){2-3} \cmidrule(lr){4-5} \cmidrule(lr){6-7}
& $\mathcal{W}_1$ & RME & $\mathcal{W}_1$ & RME & $\mathcal{W}_1$ & RME \\
\midrule
TIGON & 0.2433 & \underline{0.002} & 0.2661 & \underline{0.003} & 0.2847 & \textbf{0.001} \\
DeepRUOT & 0.2902 & \textbf{0.001} & 0.3193 & 0.011 & 0.3291 & \underline{0.002} \\
Var-RUOT & 0.2540 & 0.075 & 0.2670 & 0.014 & 0.2683 & 0.041 \\
WFR-FM & \textbf{0.2099} & \textbf{0.001} & \textbf{0.2272} & \textbf{0.002} & \textbf{0.2346} & \textbf{0.001} \\
\textbf{SUDO} & \underline{0.2305} & \underline{0.002} & \underline{0.2485} & \underline{0.003} & \underline{0.2556} & \underline{0.002} \\

\bottomrule
\end{tabular}
\end{table}

SUDO is comparable to simulation-free WFR-FM, and is consistently better than simulation-based baselines. 

\subsection{Embryoid Bodies Data}
\label{appendix:eb data}
We utilize the human embryoid bodies (EB) dataset presented by \citep{moon2019visualizing}. This dataset captures the 27-day differentiation process of human embryoid bodies and consists of 16,819 single-cell profiles across five time points. Following the dimensionality reduction protocol of \citep{wang2025joint}, the raw data was preprocessed using Principal Component Analysis (PCA). We used the top 100 PCs for all subsequent analyses. We set \(\delta=30\), and the results are shown in Table \ref{tab:eb_100d}. We use chunk size 2000 for mini batch OT.

\begin{table}[h!]
\centering
\caption{Comparison of method performance over time on the 100D EB dataset. Best results are in bold, and the second best are underlined.}
\label{tab:eb_100d}
\resizebox{\textwidth}{!}{
\begin{tabular}{lcccccccc}
\toprule
\textbf{Method} & \multicolumn{2}{c}{\textbf{t=1}} & \multicolumn{2}{c}{\textbf{t=2}} & \multicolumn{2}{c}{\textbf{t=3}} & \multicolumn{2}{c}{\textbf{t=4}} \\
\cmidrule(lr){2-3} \cmidrule(lr){4-5} \cmidrule(lr){6-7} \cmidrule(lr){8-9}
& $\mathcal{W}_1$ & RME & $\mathcal{W}_1$ & RME & $\mathcal{W}_1$ & RME & $\mathcal{W}_1$ & RME \\
\midrule
TIGON & 10.547 & 0.014 & 12.926 & 0.052 & 13.897 & 0.107 & 14.945 & 0.096 \\
DeepRUOT & 10.256 & \textbf{0.002} & \underline{11.103} & 0.074 & \underline{11.529} & 0.136 & \textbf{12.406} & 0.047 \\
Var-RUOT & 11.746 & 0.091 & 12.237 & 0.024 & 12.957 & 0.150 & 13.335 & 0.074 \\
WFR-FM & \textbf{9.941} & \underline{0.009} & \textbf{11.040} & \textbf{0.006} & \textbf{11.516} & \underline{0.008} & \underline{12.664} & \underline{0.005} \\
\textbf{SUDO} & \underline{10.231} & 0.015 & 11.389 & \underline{0.009} & 11.914 & \textbf{0.006} & 13.118 & \textbf{0.002} \\
\bottomrule
\end{tabular}
}
\end{table}

\textbf{Ablation on mini batch OT.} We also conducted an ablation study on the mini batch OT strategy \citep{fatras2021minibatchOT}. On 100D EB dataset, we trained SUDO using different mini batch OT chunk size, as well as full batch OT. As shown in Table \ref{tab:batch_size}, the performance of SUDO is not sensitive to the mini batch OT strategy. 

\begin{table}[h!]
\centering
\caption{Sensitivity analysis for batch size of mini-batch OT chunk size on the 100D EB dataset.}
\label{tab:batch_size}

\begin{tabular}{lcccccccc}
\toprule
\textbf{Batch Size} & \multicolumn{2}{c}{\textbf{t=1}} & \multicolumn{2}{c}{\textbf{t=2}} & \multicolumn{2}{c}{\textbf{t=3}} & \multicolumn{2}{c}{\textbf{t=4}}\\
\cmidrule(lr){2-3} \cmidrule(lr){4-5} \cmidrule(lr){6-7} \cmidrule(lr){8-9}
& $\mathcal{W}_1$ & RME & $\mathcal{W}_1$ & RME & $\mathcal{W}_1$ & RME & $\mathcal{W}_1$ & RME \\
\midrule
500  & 10.258 & 0.010 & 11.408 & 0.016 & 11.993 & 0.002 & 13.360 & 0.008 \\
1000 & 10.250 & 0.013 & 11.364 & 0.014 & 11.902 & 0.005 & 13.151 & 0.014 \\
2000 & 10.231 & 0.015 & 11.389 & 0.009 & 11.914 & 0.006 & 13.118 & 0.002 \\
3000 & 10.213 & 0.017 & 11.436 & 0.014 & 12.010 & 0.004 & 13.348 & 0.010 \\
4000 & 10.174 & 0.015 & 11.379 & 0.010 & 11.898 & 0.001 & 13.177 & 0.004 \\
\midrule 
full-batch & 10.195 & 0.015 & 11.391 & 0.011 & 11.911 & 0.003 & 13.143 & 0.001 \\
\bottomrule
\end{tabular}

\end{table}

\subsection{Mouse Hematopoiesis Data}
\label{appendix:mouse data}
We utilize the raw mouse hematopoietic dataset from \cite{weinreb2020lineage}. It consists of 49,302 cells from three time points. The dataset is preprocessed via PCA and the first 50 PCs were kept. Given its significant population expansion and substantial sample size, this dataset serves as an ideal benchmark for evaluating scalability of SUDO. We have done this in the main text (\ref{sec:experiment}). We set \(\delta=15\), and mini batch chunk size 2000. As shown in Table \ref{tab:mouse}, SUDO can be applied to large dataset.

\begin{table}[h!]
\centering
\caption{Comparison of method performance over time on the 50D Mouse dataset. Best results are in bold, and the second best are underlined.}
\label{tab:mouse}
\begin{tabular}{lcccc}
\toprule
\textbf{Method} & \multicolumn{2}{c}{\textbf{t=1}} & \multicolumn{2}{c}{\textbf{t=2}} \\
\cmidrule(lr){2-3} \cmidrule(lr){4-5}
& $\mathcal{W}_1$ & RME & $\mathcal{W}_1$ & RME \\
\midrule
TIGON & 6.140 & 0.382 & 6.973 & 0.326 \\
DeepRUOT & 6.052 & 0.062 & \underline{6.757} & 0.041 \\
Var-RUOT & 7.951 & 0.131 & 10.862 & 0.154 \\
WFR-FM & \textbf{5.486} & \textbf{0.012} & \textbf{6.211} & \textbf{0.011} \\
\textbf{SUDO} & \underline{5.894} & \underline{0.027} & 6.824 & 
\underline{0.019} \\
\bottomrule
\end{tabular}
\end{table}

In \ref{sec:experiment}, we follow the original preprocessing procedure of \citep{weinreb2020lineage,TIGON,DeepRUOT} to obtain a 2D mouse hematopoiesis data which consists of 10998 cells. These cells are in clones committing to neutrophils and monocytes fates at day 2, 4, 6. The data was projected to the reduced two force-directed layouts (SPRING plots) \citep{spring} after batch correction.

\subsection{Numerical Travelling Diracs for only-growth and only-death}
\begin{figure}[h!]
  \centering
  \includegraphics[width=0.9\textwidth]{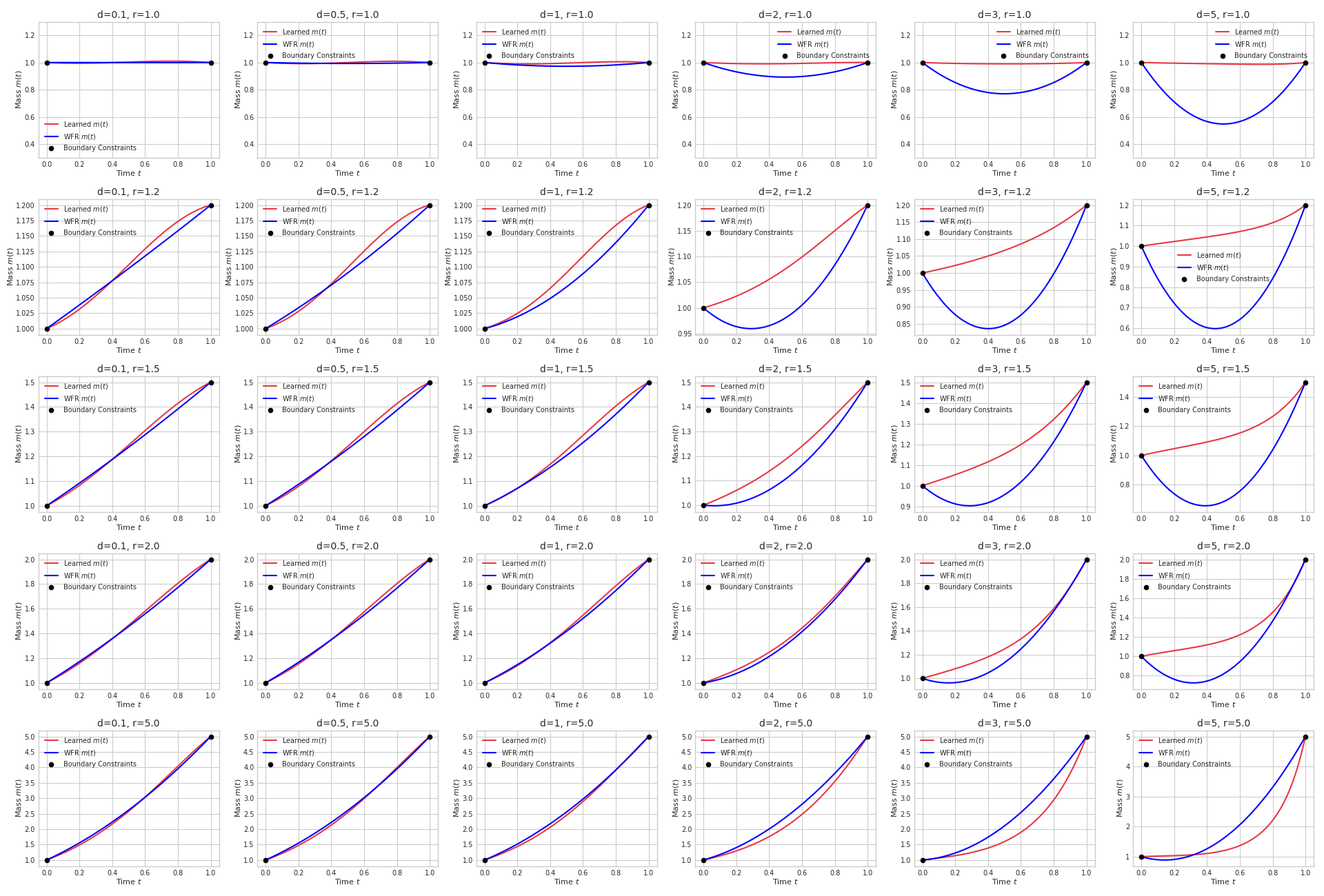}
  
  \caption{Red curve: mass variation learned by only-growth penalty; Blue curve: mass variation induced by WFR penalty. Columns: \(d \in \{0.1,0.5,1,2,3,5\}\); Rows: \(r\in\{1,1.2,1.5,2,5\}\).}
  \label{fig:only growth mass path}
\end{figure}

\begin{figure}[h!]
  \centering
  \includegraphics[width=0.9\textwidth]{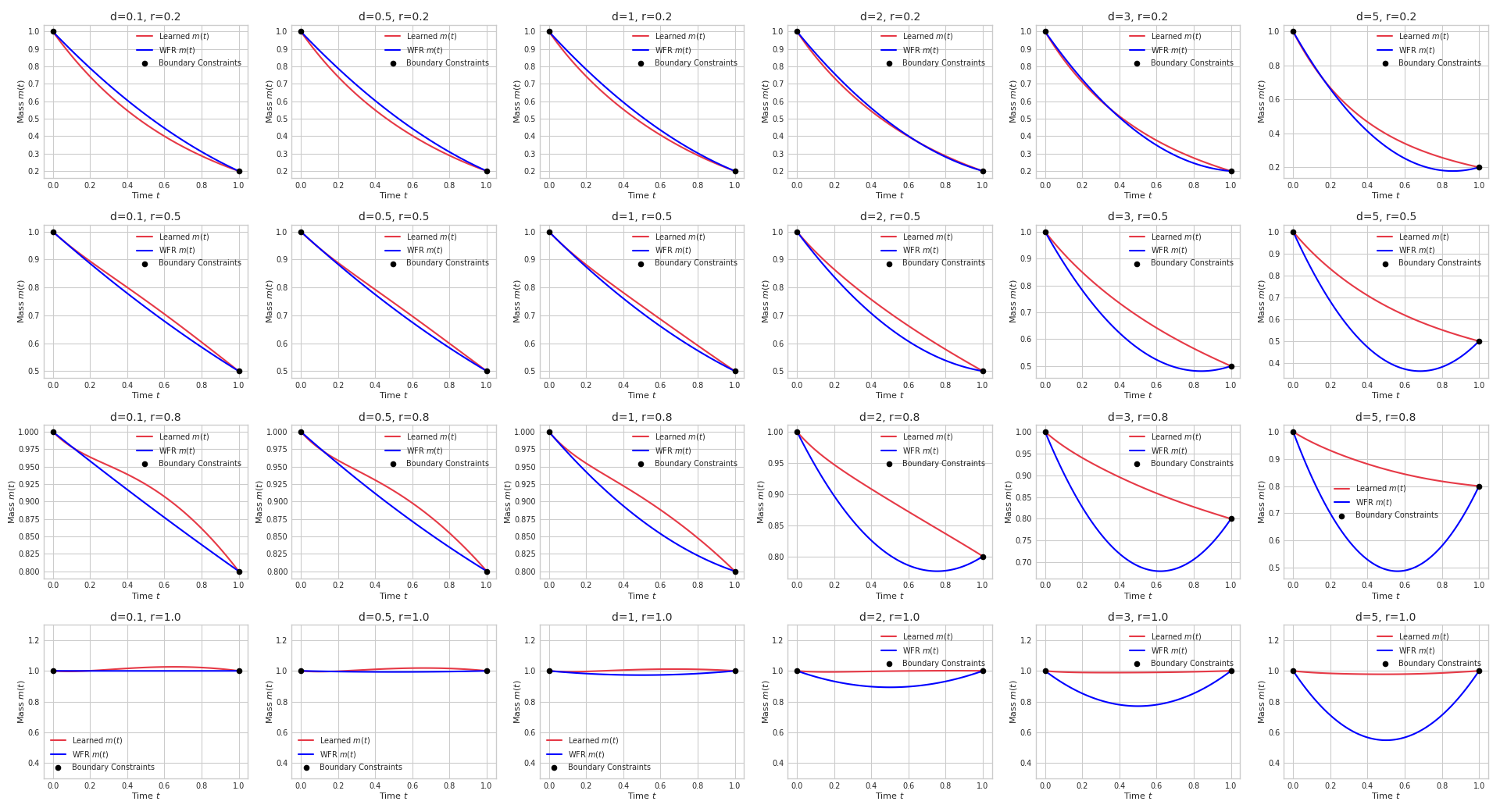}
  
  \caption{Red curve: mass variation learned by only-death penalty; Blue curve: mass variation induced by WFR penalty. Columns: \(d \in \{0.1,0.5,1,2,3,5\}\); Rows: \(r\in\{0.2,0.5,0.8,1\}\).}
  \label{fig:only death mass path}
\end{figure}

We plot the mass variation of travelling Diracs induced by the only-growth and only-death penalty (\ref{eq:BSB Psi}). These travelling Diracs differ from the WFR travelling Diracs. For only-growth, mass never decrease, while for only-death, mass never increase.

\section{Relations to other neural UOT solvers}
\subsection{UOT-FM}
UOT-FM \citep{UOT} was a straightforward extension of OT under unbalanced settings. They added a KL penalty to deal with the unbalanced marginals.
\begin{equation}
\label{UOT}
\begin{aligned}
\text{UOT}(\mu_0,\mu_1) &:= \inf_{\gamma\in\mathcal{M}_+(\mathcal{X}^2)} \{\int_{\mathcal{X}^2}  \, c(\bm{x},\bm{y})\gamma(\bm{x},\bm{y})\mathrm{d} \bm{x}\mathrm{d} \bm{y}\\&+\lambda_1\operatorname{KL}(\int_\mathcal{X}\gamma(\bm{x},\bm{y})\mathrm{d} \bm{y}\Vert \mu_0(\bm{x}))+\lambda_2\operatorname{KL}(\int_\mathcal{X}\gamma(\bm{x},\bm{y})\mathrm{d} \bm{x}\Vert \mu_1(\bm{y}))\}
\end{aligned}
\end{equation}
Having the corresponding unbalanced OT coupling \(\gamma\), they applied the idea of OT-CFM \citep{cfm_tong} to learn a flow. This is a natural way to introduce the unbalanced effect and generalize OT-CFM to unbalanced measures. 

This formulation primarily accounts for mass variation through the static transport plan. As a result, the learned dynamics are guided by an unbalanced endpoint coupling, while the continuous-time evolution of mass is not explicitly parameterized along the path. In particular, the flow-matching step focuses on learning the velocity field, so the growth or decay component is not directly modeled as part of the dynamical system.

SUDO follows a different direction by building on the unbalanced flow matching framework, where both the velocity field and the growth rate are learned simultaneously. Moreover, by solving the UDOT problem, SUDO provides a dynamic formulation in which transport and mass variation are coupled along the continuous-time trajectory. 

\subsection{WFR-FM}
\label{appendix:WFR-FM}
When choosing \(\Psi(g)=\delta^2g^2\) in UDOT formulation (\ref{eq:UDOT}), it is called the WFR problem \citep{chizat2018interpolating}.
\begin{equation}
\label{Dynamical WFR}
\begin{aligned}
&\text{WFR}_{\delta}^2(\mu_0,\mu_1) = \inf_{\rho,g,\bm{u}} \int_0^1\int_{\mathcal{X}}  \, \frac{1}{2}(\Vert \bm{u}(\bm{x},t)\Vert_2^2+\delta^2  g(\bm{x},t)^2)\rho_t(\bm{x})\mathrm{d} \bm{x}\mathrm{d}t \\
&\text{s.t.} \ \ \ \ \ \ \ \ \ \ \ \ \ \ \ \ \ \ \ \ \partial_t\rho+\nabla_{\bm{x}}\cdot(\rho \bm{u})=\rho g,\ \rho_0=\mu_0,\ \rho_1=\mu_1,
\end{aligned}
\end{equation}
The square is introduced to make \(\text{WFR}_{\delta}\) itself a metric on \(\mathcal{M}_+(\mathcal{X})\). When \(\Vert\bm{x}_1-\bm{x}_0\Vert\le\pi\delta\), the travelling Dirac has a closed form.
\begin{equation}
\label{WFR travelling Dirac}
\left \{
\begin{aligned}
&m(t)=At^2-2Bt+m_0\\
&\bm{x}(t)=\bm{x}_0+\frac{\bm{\omega}_0}{\sqrt{m_0 A - B^2}} ( \arctan ( \frac{At - B}{\sqrt{m_0 A - B^2}} ) - \arctan ( \frac{-B}{\sqrt{m_0 A - B^2}}) )
\end{aligned}
\right .
\end{equation}
where \(A,B,\bm{\omega}_0\) can be determined by \(m_0,m_1,\bm{x}_0,\bm{x}_1,\delta\) in closed form.
\begin{equation}
\label{ABw}
\left\{
\begin{aligned}
&A=m_0+m_1-2\sqrt{\frac{m_0m_1}{1+\tau^2}}\ \ , \ \ B=m_0-\sqrt{\frac{m_0m_1}{1+\tau^2}}\\
&\bm{\omega}_0=2\delta\tau \sqrt{\frac{m_0m_1}{1+\tau^2}}\bm{l}\ \ ,  \ \tau=\operatorname{tan}(\frac{ \Vert\bm{x}_1-\bm{x}_0\Vert_2}{2\delta})
\end{aligned}
\right .
\end{equation}
Where \(\bm{l}\) is the unit vector of the direction \(\bm{x}_1-\bm{x}_0\). The transport cost between two Dirac measures also has a simple closed form.
\begin{equation}
\label{dirac WFR}
\text{WFR-DD}_{\delta}^2(m_0\delta_{\bm{x}_0},m_1\delta_{\bm{x}_1}) = 2\delta^2(m_0+m_1-2\sqrt{m_0m_1}\cos(\min\{\frac{\Vert\bm{x}_0-\bm{x}_1\Vert_2}{2\delta}
,\frac{\pi}{2}\})) 
\end{equation}

The corresponding static form of WFR is proved to be equivalent to an entropic OT problem \citep{chizat2018unbalanced}.
\begin{equation}
\label{WFR OET}
\begin{aligned}
\text{WFR}_{\delta}^2(\mu_0,\mu_1) &= 2\delta^2\inf_{\gamma\in\mathcal{M}_+(\mathcal{X}^2)} \{\int_{\mathcal{X}^2}  \, -2\operatorname{ln}\operatorname{\overline{\cos}}(\frac{\Vert \bm{x}-\bm{y}\Vert_2}{2\delta})\gamma(\bm{x},\bm{y})\mathrm{d} \bm{x}\mathrm{d} \bm{y}\\&+\operatorname{KL}(\int_\mathcal{X}\gamma(\bm{x},\bm{y})\mathrm{d} \bm{y}\Vert \mu_0(\bm{x}))+\operatorname{KL}(\int_\mathcal{X}\gamma(\bm{x},\bm{y})\mathrm{d} \bm{x}\Vert \mu_1(\bm{y}))\}
\end{aligned}
\end{equation}

Given the optimal coupling \(\gamma\), the semi-coupling of WFR can be obtained by renormalization \citep{liero2018optimal,wfr_fm}.
\begin{equation}
\gamma_0(\bm{x},\bm{y})=\frac{\gamma(\bm{x},\bm{y})}{\int_\mathcal{X}\gamma(\bm{x},\bm{z})\mathrm{d} \bm{z}}\mu_0(\bm{x}), \gamma_1(\bm{x},\bm{y})=\frac{\gamma(\bm{x},\bm{y})}{\int_\mathcal{X}\gamma(\bm{z},\bm{y})\mathrm{d} \bm{z}}\mu_1(\bm{y})
\end{equation}

Based on these results, \citep{wfr_fm} proposed WFR-FM, a simulation-free framework for learning dynamic WFR flow between observed measures.

The main limitation of WFR-FM is that the framework is based on the closed form of travelling Dirac which can not be obtained for general \(\Psi\). Actually, as discussed in \citep{USB}, when \(\Psi\) is not a quadratic function, the Euler-Lagrange equation of travelling Dirac
\begin{equation}
\left \{
\begin{aligned}
&\frac{\mathrm{d}}{\mathrm{d}t}(m\dot{\bm{x}})=\bm{0}\quad\quad &(\frac{\partial}{\partial\bm{x}}=\frac{\mathrm{d}}{\mathrm{d}t}\frac{\partial}{\partial\dot{\bm{x}}})\\
&|\dot{\bm{x}}|^2+\Psi-\frac{\dot{m}}{m}\Psi'=\frac{\ddot{m}m-\dot{m}^2}{m^2}\Psi''\quad\quad &(\frac{\partial}{\partial m}=\frac{\mathrm{d}}{\mathrm{d}t}\frac{\partial}{\partial\dot{m}})
\end{aligned}
\right .
\end{equation}
is hard to solve. 

In contrast, SUDO learns the travelling Dirac and the transport cost, and solves the semi-coupling by projected gradient descent. Given the semi-coupling and the travelling Dirac, SUDO utilizes the unbalanced flow matching framework proposed in WFR-FM to learn the UDOT flow. Thus, SUDO can be viewed as a generalization of WFR-FM on non-quadratic \(\Psi\).

\subsection{TIGON and DeepRUOT}
TIGON \citep{TIGON} used NeuralODE to solve the UDOT problem. DeepRUOT \citep{DeepRUOT} introduced stochasticity into it, and solved the RUOT problem. These methods can be viewed as simulation-based UDOT solvers which suffer from the expensive computational cost. In contrast, SUDO is a simulation-free solver, hence, it is more efficient than simulation-based methods.

\subsection{VarRUOT}
\label{appendix:varrout}
VarRUOT \citep{sun2026variational} parameterized the dual variable of the RUOT problem with a single network to achieve faster convergence and lower action, and represents the velocity and growth rate by the dual variable. In their original work, they also proposed to solve a special class of RUOT problem with concave growth penalty \(\Psi(g)=|g|^p,\ 0<p<1\) and obtained a non-zero solution, given that the monotonicity of growth rate along trajectory is determined by the convexity or concavity of \(\Psi\). 

The relation between these observations and Theorem \ref{thm:concave} in this work can be understood from the distinction between empirical stationary dynamics and global variational behavior. Theorem \ref{thm:concave} characterizes a degeneracy of the UDOT objective under concave growth penalties: for strictly sublinear concave \(\Psi\), one can construct minimizing sequences whose total cost approaches zero by separating mass change from transport. This result describes the global infimum structure of the variational problem, rather than excluding the existence of non-trivial stationary or finite-parameter solutions obtained under a particular numerical scheme.

From this perspective, the non-zero solutions reported by VarRUOT are compatible with our theorem. Since VarRUOT solves a SDE simulation-based, parameterized, and regularized optimization problem, its empirical solutions may correspond to stable stationary points within the chosen neural parameterization and training procedure. Moreover, the monotonicity results in VarRUOT are derived from first-order optimality conditions, which are necessary conditions and therefore can also hold for stationary solutions that are not global minimizers of the original variational problem. Thus, Theorem \ref{thm:concave} does not contradict the empirical findings of VarRUOT. Instead, it clarifies why concave growth penalties may lead to degenerate global variational behavior, and thereby helps explain the gap between empirical non-trivial dynamics and the underlying global optimization structure. 

\section{Growth penalties from branching Sch\"odinger problem}
\label{appendix:BSB}
In this section, we follow \citep{BSB} to introduce a family of convex \(\Psi\) with nice microscopic biological interpretation. This family of \(\Psi\) is rooted in the branching Schr\"odinger problem which is related to the RUOT (\ref{eq:dynamic RUOT}) problem. To introduce their results, we first introduce the branching Brownian motion.

\subsection{Branching Brownian motion}
Branching Brownian motion (BBM) is a branching processes where particles undergo both diffusion and branching. Mathematically, it is characterized by a doublet $(\nu,\bm{q})$, where $\nu\in\mathbb{R}$ dictates the scale of diffusion and $\bm{q}\in\mathcal{M}_{+}(\mathbb{N})$ represents the branching mechanism—a finite, unnormalized measure over the natural numbers. From this mechanism, we can derive both the overall branching rate $\lambda=\sum_{k\in\mathbb{N}}\bm{q}_k$ and the normalized offspring distribution $\bm{p}=\bm{q}/\lambda$.

Dynamically, a particle in a BBM system diffuses via standard Brownian motion scaled by $\nu$ (\(\nu\mathbb{W}\),\(\mathbb{W}\) is the standard Bronian motion). This continuous movement is interrupted at a random time $\tau\sim Exp(\lambda)$ governed by an internal exponential clock. At \(\tau\), the original particle undergoes a branching event to produce $k\sim\bm{p}$ offspring, where $k=0$ simply indicates particle death, $k=1$ means nothing happens, $k=2$ indicates that the particle divides into two independent new particles, and so on. Each newly generated particle then independently restarts this exact same stochastic process. 

By constraining possible outcomes to strict duplication or death \(\bm{p}_k=0,\ k\neq0,2\), BBM perfectly mirrors the underlying microscopic biological mechanisms of cell division and apoptosis. Furthermore, its inherent capacity to model discrete growth in population mass makes it an suitable framework for single-cell birth-death dynamics.

\subsection{Schr\"odinger bridge}
Originally introduced by \citep{schrodinger1932sur}, the Schrödinger bridge (SB) problem provides a mathematical framework for establishing stochastic trajectories between two given probability distributions. Specifically, it seeks the optimal stochastic process $\mathbb{P}^{\star}$ that precisely matches the normalized boundary distributions $\mu_0$ and $\mu_1$, while remaining as close as possible to a predefined reference process $\mathbb{Q}$. This is formulated as a Kullback-Leibler (KL) divergence minimization problem:
\begin{equation}
\label{eq:SB}
\mathbb{P}^{\star} = \underset{\mathbb{P}:p_0=\mu_0, p_1=\mu_1}{\arg \min} \text{KL}(\mathbb{P} \Vert \mathbb{Q})
\end{equation}
where $p_t$ represents the marginal distributions of the process $\mathbb{P}$ at time $t$. A standard choice for the reference prior $\mathbb{Q}$ is $\nu\mathbb{W}$, representing a standard Brownian motion scaled by a diffusion coefficient $\nu$. When parameterized this way, the framework is widely referred to as the diffusion Schrödinger bridge (DSB) \cite{bortoli2021diffusion,bunne2023schrodinger,shi2024diffusion}.

\subsection{Branching Schr\"odinger problem}
To extend the SB framework into unbalanced settings where \(\mu_0,\mu_1\) are unnormalized measures with different total mass, \citep{BSB} developed a generalized framework known as the branching Schrödinger problem. By employing branching Brownian motion (BBM) as the underlying reference process, this formulation inherently supports mass creation and destruction. The objective is defined by:
\begin{equation}
\label{eq:BSB}
\mathbb{P}^{\star} = \underset{\mathbb{P}:p_0=\mu_0, p_1=\mu_1}{\arg \min} \text{KL}(\mathbb{P} \Vert \mathbb{Q})
\end{equation}
where $\mathbb{Q}$ is a BBM, and the target marginals $\mu_0$ and $\mu_1$ are explicitly permitted to be unnormalized.

The branching Schrödinger bridge problem seeks to find a branching stochastic process that connects the target measures while remaining as close as possible to the reference BBM. Intuitively, the parameters $\nu$, $\lambda$, and $\bm{p}$ of the reference BBM represent the prior knowledge about the randomness and growth of the system. 

More specifically, we restrict ourselves to cases where \(\bm{p}_k=0,\ k\neq0,2\). The choice of \((\bm{p}_0,\bm{p}_2)\) represents the prior about the relative rate of cell proliferation and apoptosis. For example, setting $\bm{p}_2=1$ reflects a prior assumption of pure cellular proliferation, whereas $\bm{p}_0=1$ corresponds to pure apoptosis.

\subsection{Penalties}
Unfortunately, the branching Schr\"odinger problem is ill-posed and hard to solve. \citep{BSB} found that its convex relaxation is exactly the RUOT problem (\ref{eq:dynamic RUOT}).
\begin{equation}
\label{appendix:dynamic RUOT}
\begin{aligned}
&\text{RUOT}(\mu_0,\mu_1) =
\inf_{\rho,g,\bm{u}} \int_0^1\int_{\mathcal{X}}  \, \frac{1}{2}\Big(\Vert \bm{u}(\bm{x},t)\Vert_2^2+\Psi_{\nu,\lambda}(g(\bm{x},t))\Big)\rho_t(\bm{x})\mathrm{d} \bm{x}\mathrm{d}t \\
&\text{s.t.}\ \ \ \ \ \ \ \ \ \ \ \ \ \ \partial_t\rho+\nabla_{\bm{x}}\cdot(\rho \bm{u})=\rho g+\frac{\nu^2}{2}\Delta_{\bm{x}}\rho,\ \rho_0=\mu_0,\ \rho_1=\mu_1
\end{aligned}
\end{equation}

where \(\Psi_{\nu,\lambda}\) is determined by \((\nu,\lambda,\bm{p})\). Three important instances are
\begin{equation}
\label{eq:BSB Psi appendix}
\left \{
\begin{aligned}
&\Psi_{1,1}(g)=2(1-g+g\operatorname{log}g)\ (g>0),\ &\bm{p}_2=1\\
&\Psi_{1,1}(g)=2(1+g-g\operatorname{log}(-g))\ (g<0),\ &\bm{p}_0=1\\
&\Psi_{1,1}(g)=2\big(1-\sqrt{1+g^2}+g\operatorname{log}(g+\sqrt{1+g^2})\big)\ &\bm{p}_0=\bm{p}_2=1/2\\
\end{aligned}
\right .
\end{equation}

These are penalties we introduced in (\ref{eq:BSB Psi}), representing the prior about the growth of the system: Only-growth, Only-death, and No preference.

For general \((\nu,\lambda)\), \(\Psi_{\nu,\lambda}(g)=\nu^2\lambda\Psi(\frac{g}{\lambda})\). Note that our notation is slightly different to their original paper. In their paper, \(\nu\) is the variance of diffusion, while our \(\nu\) is the standard deviation of diffusion. For general \(\bm{p}\), \(\Psi\) can be represented as a Legendre transform of some function determined by \((\nu,\lambda,\bm{p})\), hence, it is convex. For more details, please refer to \citep{BSB}.

Though these convex \(\Psi\) are introduced in RUOT, we point out that UDOT (\ref{eq:UDOT}) is just the limit case of RUOT with zero noise \(\nu\). With small noise level, RUOT is similar to UDOT. Thus, we adopt these penalties as empirical growth preference prior.

\newpage
\section{Algorithm workflow}
\label{appendix:algorithm}

\begin{algorithm}[h!]
\caption{Path model training}
\label{alg:path model training}
\begin{algorithmic}[1]
\Require Growth penalty \(\Psi\), grid for distance \(D\), grid for mass ratio \(R\), path model \((\phi_\eta,\psi_\eta)\).
\While{Training}
    \State $t\sim\mathcal{U}[0,1]$
    \State $\mathcal{L}_{path}(\eta)\gets 0$
    \For{$(d,r)\in D\times R$}
        \State $k_{\eta} \gets t+t(1-t)\phi_{\eta}(t,d,r)$
        \State $l_{\eta}\gets r^t\operatorname{exp}(t(1-t)\psi_{\eta}(t,d,r))$
        \State $\mathcal{L}_{path}(\eta)\gets\mathcal{L}_{path}(\eta)+\Big(\dot{k}_{\eta}^2+ \Psi(\frac{\dot{l}_{\eta}}{l_{\eta}})\Big)l_{\eta}$
    \EndFor
    \State $\eta \gets \mathrm{Update}(\eta, \nabla_{\eta} \mathcal{L}_{path}(\eta))$
\EndWhile
\State \Return \((\phi_\eta,\psi_\eta)\)
\end{algorithmic}
\end{algorithm}

\begin{algorithm}[h!]
\caption{Cost model training}
\label{alg:cost model training}
\begin{algorithmic}[1]
\Require Growth penalty \(\Psi\), grid for distance \(D\), grid for mass ratio \(R\), trained path model \((\phi_\eta,\psi_\eta)\), cost model \(\mathcal{E}_\xi\), numerical integral node set $T=\{0,\frac{1}{N},\frac{2}{N},\cdots,\frac{N-1}{N},1\}$.
\While{Training}
    \State $\mathcal{L}_{cost}(\xi)\gets 0$
    \For{$(d,r)\in D\times R$}
        \State $I=0$
        \For{$t\in T$}
            \State $k_{\eta} \gets t+t(1-t)\phi_{\eta}(t,d,r)$
            \State $l_{\eta}\gets r^t\operatorname{exp}(t(1-t)\psi_{\eta}(t,d,r))$
            \If{$t\in\{0,1\}$}
                \State $I \gets I+\frac{1}{4N}\big(\dot{k}_{\eta}^2+ \Psi(\frac{\dot{l}_{\eta}}{l_{\eta}})\big)l_{\eta}$
            \Else
                \State $I \gets I+\frac{1}{2N}\big(\dot{k}_{\eta}^2+ \Psi(\frac{\dot{l}_{\eta}}{l_{\eta}})\big)l_{\eta}$
            \EndIf
        \EndFor
        \State $\mathcal{L}_{cost}(\xi)\gets \mathcal{L}_{cost}(\xi)+\big(\mathcal{E}_\xi(d,r)-I\big)^2$
    \EndFor
    \State $\xi \gets \mathrm{Update}(\xi, \nabla_{\xi} \mathcal{L}_{cost}(\xi))$
\EndWhile
\State \Return \(\mathcal{E}_\xi\)
\end{algorithmic}
\end{algorithm}

\begin{algorithm}[h!]
\caption{Flow model training}
\label{alg:flow model training}
\begin{algorithmic}[1]
\Require Sample-able distributions $\mu_{t_0}, \mu_{t_1}, \ldots, \mu_{t_K}$, growth penalty \(\Psi\), trained path model \((\phi_\eta,\psi_\eta)\), velocity net $\bm{u_{\theta}}(\bm{x},t)$, growth rate net $g_{\bm{\theta}}(\bm{x},t)$, semi-couplings $(\gamma_0^{(k)},\gamma_1^{(k)})$.
\While{Training}
    \State $\mathcal{L}_{\text{CUFM}}(\bm{\theta})\gets 0$
    \For{$k = 0 \to K-1$}
        \State $(\bm{x}_{t_k}, \bm{x}_{t_{k+1}}) \sim \gamma_0^{(k)},d\gets \Vert\bm{x}_{t_{k+1}}-\bm{x}_{t_{k}}\Vert,r\gets\gamma_1^{(k)}(\bm{x}_{t_k}, \bm{x}_{t_{k+1}})$
        \State $t \sim \mathcal{U}[0,1]$, $t \gets t_k + (t_{k+1}-t_k)t$
        
        \State $\bm{x}_{\eta,t} \gets \bm{x}_{t_k}+(\bm{x}_{t_{k+1}}-\bm{x}_{t_k})(t+t(1-t)\phi_{\eta}(t,d,r))$
        \State $m_{\eta,t}\gets r^t\operatorname{exp}(t(1-t)\psi_{\eta}(t,d,r))$

        \State $\mathcal{L}_{\text{CUFM}}(\bm{\theta})\gets \mathcal{L}_{\text{CUFM}}(\bm{\theta})+m_{\eta,t}\big(\left\| \bm{\bm{u}_{\theta}}(\bm{x}_{\eta,t},t) - \dot{\bm{x}}_{\eta,t} \right\|_2^2+\left\| g_{\bm{\theta}}(\bm{x}_{\eta,t},t)- \frac{\dot{m}_{\eta,t}}{m_{\eta,t}} \right\|_2^2\big)$
    \EndFor
    \State $\bm{\theta} \gets \mathrm{Update}(\bm{\theta}, \nabla_{\bm{\theta}} \mathcal{L}_{\text{CUFM}}(\bm{\theta}))$
\EndWhile
\State \Return $\bm{u_{\theta}}$, $g_{\bm{\theta}}$, and $\bm{s}_{\bm{\theta}}$
\end{algorithmic}
\end{algorithm}

\newpage
\begin{algorithm}[h!]
\caption{Inference}
\label{alg:inference}
\begin{algorithmic}[1]
\Require Data at initial time point $\mathcal{D}_0$, timestep \(\Delta t\), learned vector net $\bm{u_{\theta}}(\bm{x},t)$, growth rate net $g_{\bm{\theta}}(\bm{x},t)$.
\For {\(\bm{x}\) in \(\mathcal{D}_0\)}
    \State \(\omega=1\)
    \While {simulation}
        \State $\bm{x} \gets \bm{x}+\bm{u_{\theta}}(\bm{x},t)\Delta t$
        \State $\omega \gets \omega \cdot e^{g_{\bm{\theta}}(\bm{x},t)\Delta t}$
    \EndWhile
    \State \(\mathcal{D}_{weight}\) append \((\bm{x},\omega)\)
\EndFor
\State \Return \(\mathcal{D}_{weight}\)
\end{algorithmic}
\end{algorithm}



\end{document}